\documentclass{article}

\usepackage{microtype}
\usepackage{graphicx}
\usepackage{subfigure}
\usepackage{float}
\usepackage{booktabs}
\usepackage{multirow}
\usepackage[table,xcdraw]{xcolor}

\usepackage{hyperref}
\usepackage[shortlabels]{enumitem}

\usepackage[accepted]{icml2024}

\usepackage{amsmath}
\usepackage{amssymb}
\usepackage{mathtools}
\usepackage{amsthm}
\usepackage{caption}
\usepackage{enumitem}
\usepackage{float}
\usepackage{csquotes}

\DeclareMathOperator*{\argmin}{arg\,min}

\usepackage[capitalize,noabbrev]{cleveref}

\theoremstyle{plain}
\newtheorem{theorem}{Theorem}[section]

\newtheorem{lemma}[theorem]{Lemma}
\newtheorem{corollary}[theorem]{Corollary}
\theoremstyle{definition}
\newtheorem{definition}[theorem]{Definition}

\theoremstyle{remark}
\newtheorem{remark}[theorem]{Remark}

\usepackage[textsize=tiny]{todonotes}

\icmltitlerunning{\texttt{AAggFF}: Adaptive Aggregation for Fair Federated Learning}

\begin{document}
\twocolumn[
\icmltitle{Pursuing Overall Welfare in Federated Learning \\ through Sequential Decision Making}

\icmlsetsymbol{equal}{*}
\begin{icmlauthorlist}
\icmlauthor{Seok-Ju Hahn}{unistie}
\icmlauthor{Gi-Soo Kim}{unistie,unistgsai}
\icmlauthor{Junghye Lee}{snutemep,snuprac,snuier}
\end{icmlauthorlist}

\icmlaffiliation{unistie}{Department of Industrial Enginnering, Ulsan National Institute of Science and Technology (UNIST), Ulsan, South Korea}
\icmlaffiliation{unistgsai}{Artificial Intelligence Graduate School, Ulsan National Institute of Science and Technology (UNIST), Ulsan, South Korea}
\icmlaffiliation{snutemep}{Technology Management, Economics and Policy Program, Seoul National University (SNU), Seoul, South Korea}
\icmlaffiliation{snuprac}{Graduate School of Engineering Practice, Seoul National University (SNU), Seoul, South Korea}
\icmlaffiliation{snuier}{Institute of Engineering Research, Seoul National University (SNU), Seoul, South Korea}

\icmlcorrespondingauthor{Junghye Lee}{junghye@snu.ac.kr}
\icmlcorrespondingauthor{Gi-Soo Kim}{gisookim@unist.ac.kr}

\icmlkeywords{Machine Learning, ICML}
\vskip 0.3in
]



\printAffiliationsAndNotice{}  

\begin{abstract}
    In traditional federated learning, a single global model cannot perform equally well for all clients. 
    Therefore, the need to achieve the \textit{client-level fairness} in federated system has been emphasized, which can be realized by modifying the static aggregation scheme for updating the global model to an adaptive one, in response to the local signals of the participating clients.
    Our work reveals that existing fairness-aware aggregation strategies can be unified into an online convex optimization framework, in other words, a central server's \textit{sequential decision making} process.
    To enhance the decision making capability, we propose simple and intuitive improvements for suboptimal designs within existing methods, presenting \texttt{AAggFF}. 
    Considering practical requirements, we further subdivide our method tailored for the \textit{cross-device} and the \textit{cross-silo} settings, respectively. 
    Theoretical analyses guarantee sublinear regret upper bounds for both settings: $\mathcal{O}(\sqrt{T \log{K}})$ for the cross-device setting, and $\mathcal{O}(K \log{T})$ for the cross-silo setting, with $K$ clients and $T$ federation rounds. 
    Extensive experiments demonstrate that the federated system equipped with \texttt{AAggFF} achieves better degree of client-level fairness than existing methods in both practical settings. 
    Code is available at \href{https://github.com/vaseline555/AAggFF}{https://github.com/vaseline555/AAggFF}.
\end{abstract}

\section{Introduction}
\label{sec:intro}
    Federated Learning (FL) has been posed as an effective strategy to acquire a global model without centralizing data, therefore with no compromise in privacy \cite{fedavg}.
    It is commonly assumed that the central server coordinates the whole FL procedure by repeatedly \textit{aggregating} local updates from participating $K$ clients during $T$ rounds.
    
    Since each client updates the copy of a global model with its own data, variability across clients' data distributions causes many problems \cite{prob_fl, challenge_fl}.
    The \textit{client-level fairness} \cite{clientlevelfairness} is one of the main problems affected by such a statistical heterogeneity \cite{prob_fl, challenge_fl, afl, qffl}. 
    Although the performance of a global model is high in average, some clients may be more benefited than others, resulting in violation of the client-level fairness. 
    In this situation, there inevitably exists a group of clients who cannot utilize the trained global model due to its poor performance.
    This is a critical problem in practice since the underperformed groups may lose motivation to participate in the federated system.
    To remedy this problem, previous works \cite{afl, qffl, term, fedmgda, propfair} proposed to modify the static aggregation scheme into an \textit{adaptive aggregation} strategy, according to given local signals (e.g., losses or gradients). 
    In detail, the server \textit{re-weights local updates} by assigning larger \textit{mixing coefficients} to higher local losses.
    
    When updating the mixing coefficients, however, \textit{only a few bits are provided to the server}, compared to the update of a model parameter.
    For example, suppose that there exist $K$ clients in the federated system, each of which has $N$ local samples. 
    When all clients participate in each round, $KN$ samples are used effectively for updating a new global model $\boldsymbol{\theta}$.
    On the contrary, only $K$ bits (e.g., local losses: $F_1(\boldsymbol{\theta}), ..., F_K(\boldsymbol{\theta})$) are provided to the server for an update of mixing coefficients.
    This is aggravated in cases where $K$ is too large, thus client sampling is inevitably required. 
    In this case, the server is provided with far less than $K$ signals, which hinders faithful update of the mixing coefficients.
    
    For sequentially updating a status in this \textit{sample-deficient} situation, the Online Convex Optimization (OCO) framework is undoubtedly the best solution.
    Interestingly, we discovered that most existing adaptive aggregation strategies can be readily unified into the OCO framework. 
    Starting from this unification result, we propose an improved design for a fair FL algorithm in the view of sequential decision making.
    Since there exist OCO algorithms specialized for the setting where the decision space is a simplex (i.e., same as the domain of the mixing coefficient), these may be adopted to FL setting with some modifications for practical constraints. 
    
    In practice, FL is subdivided into two settings: \textit{cross-silo} setting and \textit{cross-device} setting \cite{prob_fl}.
    For $K$ clients and $T$ training rounds, each setting requires a different dependency on $K$ and $T$.
    In the cross-silo setting, the number of clients (e.g., institutions) is small and usually less than the number of rounds (i.e., $K < T$).
    e.g., $K=20$ institutions with $T=200$ rounds~\cite{real_silo}.
    On the other hand, in the cross-device setting, the number of clients (e.g., mobile devices) is larger than the number of rounds ($K > T$). 
    e.g., $K=1.5\times10^6$ with $T=3,000$ rounds~\cite{real_device}.
    In designing an FL algorithm, these conditions should be reflected for the sake of practicality.
    
    \paragraph{Contributions} 
    We propose \texttt{AAggFF}, a sequential decision making framework for the central server in FL 
    tailored for inducing client-level fairness in the system.
    The contributions of our work are summarized as follows.
    \begin{enumerate}
      \item[$\bullet$] We unify existing fairness-aware adaptive aggregation methods into an OCO framework 
      and propose better online decision making designs for pursuing client-level fairness by the central server. 
      (Section \ref{sec:backgrounds})  
      \item[$\bullet$] We propose \texttt{AAggFF}, which is designed to enhance the client-level fairness, 
      and further specialize our method into two practical settings:
      \texttt{AAggFF-S} for cross-silo FL and \texttt{AAggFF-D} for cross-device FL. 
      (Section \ref{sec:proposed})
      \item[$\bullet$] We provide regret analyses on the behavior of two algorithms, 
      \texttt{AAggFF-S} and \texttt{AAggFF-D}, presenting sublinear regrets. 
      (Section \ref{sec:analyses}) 
      \item[$\bullet$] We evaluate \texttt{AAggFF} on extensive benchmark datasets for realistic FL scenarios with other baselines.
      \texttt{AAggFF} not only boosts the worst-performing clients but also maintains overall performance. 
      (Section \ref{sec:experiments})
    \end{enumerate}

\section{Related Works}
\label{sec:rel_works}
\paragraph{Client-Level Fairness in Federated Learning} 
    The statistical heterogeneity across clients often causes non-uniform performances of a single global model on participating clients, which is also known as the violation of client-level fairness.
    Fairness-aware FL algorithms aim to eliminate such inequality to achieve uniform performance across all clients.
    There are mainly two approaches to address the problem \cite{clientlevelfairness}: 
    a single model approach, and a personalization approach\nocite{superfed}.
    This paper mainly focuses on the former, which is usually realized by modifying the FL objective,
such as a minimax fairness objective~\cite{afl} (which is also solved by multi-objective optimization \cite{fedmgda} and is also modified to save a communication cost \cite{drfa}), 
    alpha-fairness \cite{alphafair} objective \cite{qffl}, 
    suppressing outliers (i.e., clients having high losses) by tilted objective \cite{term},
    and adopting the concept of proportional fairness to reach Nash equilibrium in performance distributions \cite{propfair}.
    While the objective can be directly aligned with existing notions of fairness, 
    it is not always a standard for the design of a fair FL algorithm.
    Notably, most of works share a common underlying principle: \textit{assigning more weights to a local update having larger losses}. 

    \begin{definition} (Client-Level Fairness; Definition 1 of~\cite{qffl}, Section 4.2 of~\cite{clientlevelfairness})
    We informally define the notion of client-level fairness in FL as the status where a trained global model yields uniformly good performance across all participating clients.
    Note that uniformity can be measured by the spread of performances.
    \end{definition}

\paragraph{Online Decision Making}
    The OCO framework is designed for making \textit{sequential decisions} with the best utilities, having solid theoretical backgrounds.
    It aims to minimize the cumulative mistakes of a decision maker (e.g., central sever), 
    given a response of the environment (e.g., losses from clients) for finite rounds $t\in[T]$.
    The cumulative mistakes of the learner are usually denoted as the cumulative regret (see (\ref{eq:regret})), 
    and the learner can achieve sublinear regret in finite rounds using well-designed OCO algorithms \cite{oco1, oco2, oco3}.
    In designing an OCO algorithm, two main frameworks are mainly considered: 
    Online Mirror Descent (OMD) \cite{md, omd, bregmanmd} 
    and Follow-The-Regularized-Leader (FTRL) \cite{ftrl1, ftrl2, ftrl3, ftrl4}.
    One of popular instantiations of both frameworks is the Online Portfolio Selection (OPS) algorithm, of which decision space is restricted to a probability simplex.
    The universal portfolio algorithm is the first that yields an optimal theoretical regret, $\mathcal{O}(K \log{T})$) despite its heavy computation ($\mathcal{O}(K^4 T^{14})$) \cite{up}, 
    the Online Gradient Descent \cite{ogd} and the Exponentiated Gradient (EG) \cite{eg} show slightly worse regrets (both are $\mathcal{O}(\sqrt{T})$), but can be executed in linear runtime in practice ($\mathcal{O}(K)$).
    Plus, the Online Newton Step (ONS) \cite{ons1, ons2} presents logarithmic regret with quadratic runtime in $K$.
    Since these OPS algorithms are proven to perform well when the decision is a probability vector, we adopt them for finding adaptive mixing coefficients to achieve performance fairness in FL.
    Apart from OPS, we refer interested readers to \textit{fair resource allocation}, which allocates limited resources to users in a network system~\cite{kelly1998rate,altman2008generalized,joe2013multiresource} especially in dynamic manner~\cite{si2022enabling,banerjee2023online,talebi2018learning}.
    To the best of our knowledge, we are the first to consider fair FL algorithms under the OCO framework.

\section{Backgrounds}
\label{sec:backgrounds}
\subsection{Mixing Coefficients in Federated Learning}
\label{subsec:mix_coefs}
    \begin{table*}[!h]
\centering
\caption{Summary of Unification Results of Existing Fair FL Methods into an OCO Framework (\ref{eq:omd_objective}), viz. \textit{Remark}~\ref{remark:unification}}
\label{tab:unification}
\resizebox{\textwidth}{!}{%
\begin{tabular}{l|l|llll}
\toprule
\textbf{Method} & 
\textbf{Original Objective} (w.r.t. $\boldsymbol{\theta}$) & 
\multicolumn{1}{l}{\begin{tabular}[l]{@{}l@{}}\textbf{Response}\\ $({r}_i^{(t)})$ \end{tabular}} &
\multicolumn{1}{l}{\begin{tabular}[l]{@{}l@{}}\textbf{Last Decision}\\ $(p_i^{(t)})$ \end{tabular}} &
\multicolumn{1}{l}{\begin{tabular}[l]{@{}l@{}}\textbf{Step Size}\\ $\left(\eta\right)$ \end{tabular}} &
\multicolumn{1}{l}{\begin{tabular}[l]{@{}l@{}}\textbf{New Decision}\\ $(p_i^{(t+1)})$ \end{tabular}} \\ \midrule
\begin{tabular}[l]{@{}l@{}}\texttt{FedAvg} \cite{fedavg}\end{tabular} & 
    $\min_{\boldsymbol{\theta}\in\mathbb{R}^d}
    \sum_{i=1}^K \frac{n_i}{n} F_i\left({\boldsymbol{\theta}}\right)$ & 
    $0$ &
    $n_i/n$ & 
    $1$ & 
    $\propto n_i$ \\
\begin{tabular}[l]{@{}l@{}}\texttt{q-FedAvg} \cite{qffl}\\ (\texttt{AFL} \cite{afl}) \end{tabular} & 
    \begin{tabular}[l]{@{}l@{}} $\min_{\boldsymbol{\theta}\in\mathbb{R}^d}  \sum_{i=1}^K \frac{n_i/n}{q+1} F_i^{q+1}\left(\boldsymbol{\theta}\right)$\\(\texttt{AFL} if $q\rightarrow \infty$) \end{tabular} & 
    $q\log {F}_i(\boldsymbol{\theta}^{(t)})$ & 
    $n_i/n$ & 
    $1$ & 
    $\propto n_i {F}_i^q\left(\boldsymbol{\theta}^{(t)}\right)$ \\
\begin{tabular}[l]{@{}l@{}}\texttt{TERM} \cite{term}\end{tabular} & 
    $\min_{\boldsymbol{\theta}\in\mathbb{R}^d}
    \frac{1}{\lambda} \log (\sum_{i=1}^K \frac{n_i}{n} \exp({\lambda F_i({\boldsymbol{\theta}})}))$ & 
    $F_i({\boldsymbol{\theta}}^{(t)})$ & 
    $n_i/n$ & 
    $\frac{1}{\lambda}$ & 
    $\propto n_i \exp\left(\lambda {F}_i\left(\boldsymbol{\theta}^{(t)}\right)\right)$ \\
\begin{tabular}[l]{@{}l@{}}\texttt{PropFair} \cite{propfair}\end{tabular} & 
    $\min_{\boldsymbol{\theta}\in\mathbb{R}^d} -\sum_{i=1}^K \frac{n_i}{n} \log(M-F_i(\boldsymbol{\theta}))$ & 
    $-\log (M - F_i(\boldsymbol{\theta}^{(t)}))$ & 
    $n_i/n$ & 
    $1$ & 
    $\propto \frac{n_i}{M - F_i\left(\boldsymbol{\theta}^{(t)}\right)}$ \\ \bottomrule
\end{tabular}%
}
\end{table*}

    The canonical objective of FL is given as follows:
    \begin{equation}
    \begin{gathered}
    \label{eq:fl_obj}
        \min_{\boldsymbol{\theta}\in\mathbb{R}^d}
        F\left(\boldsymbol{\theta}\right)
        =
        \sum\nolimits_{i=1}^K p_i F_i\left({\boldsymbol{\theta}}\right).
    \end{gathered}    
    \end{equation}
    For $K$ clients, the FL objective aims to minimize the composite objectives, where client $i$'s local objective is $F_i\left({\boldsymbol{\theta}}\right)$, 
    weighted by a corresponding \textit{mixing coefficient} $p_i \geq 0$ ($\sum_{i=1}^K p_i=1$),
    which is usually set to be a \textit{static} value proportional to the sample size $n_i$: 
    e.g., $p_i= \frac{n_i}{n}, n=\sum_{j=1}^K n_j$.
    Each local objective, $F_i\left({\boldsymbol{\theta}}\right)= \frac{1}{n_i} \sum_{k=1}^{n_i} \mathcal{L}\left(\xi_k;\boldsymbol{\theta}\right)$, is defined as the average of per-sample training loss $\mathcal{L}\left(\cdot;\boldsymbol{\theta}\right)$ calculated from the local dataset, 
    $\mathcal{D}_i=\{\xi_k\}_{k=1}^{n_i}$.
    Denote $\Vert\cdot\Vert_p$ as an $L_p$-norm
    and $\Delta_{K-1}$ as a probability simplex where $\Delta_{K-1}=\left\{\boldsymbol{q} \in \mathbb{R}^K: q_i \geq 0,\Vert \boldsymbol{q} \Vert_1 = 1\right\}$. 
    Note that the mixing coefficient is a member of $\Delta_{K-1}$.
    
    In vanilla FL, the role of the server to solve (\ref{eq:fl_obj}) is to naively add up local updates into a new global model by weighting each update with the \textit{static} mixing coefficient proportional to $n_i$.
    As the fixed scheme often violates the client-level fairness, 
    the server should use \textit{adaptive} mixing coefficients to pursue overall welfare across clients.
    This can be modeled as an optimization w.r.t. $\boldsymbol{p}\in\Delta_{K-1}$, apart from (\ref{eq:fl_obj}).

\subsection{Online Convex Optimization as a Unified Language}
\label{subsec:oco_lang}
    To mitigate the performance inequalities across clients, 
    adaptive mixing coefficients can be estimated in response to local signals 
    (e.g., local losses of a global model).
    Intriguingly, the adaptive aggregation strategies in existing fair FL methods \cite{fedavg, afl, qffl, term, propfair} can be readily \textit{unified into one framework}, borrowing the language of OCO.
    \begin{remark}
    \label{remark:unification}  
    Suppose we want to solve a minimization problem defined in (\ref{eq:omd_objective}).
    For all $t\in[T]$, it aims to minimize a \textit{decision loss} $\ell^{(t)}\left(\boldsymbol{p}\right) = -\left\langle \boldsymbol{p}, \boldsymbol{r}^{(t)} \right\rangle$ 
    (where $\left\langle \cdot, \cdot \right\rangle$ is an inner product)
    defined by a \textit{response} $\boldsymbol{r}^{(t)}\in\mathbb{R}^K$ 
    and a \textit{decision}  $\boldsymbol{p}\in\Delta_{K-1}$, 
    with a \textit{regularizer} $R\left(\boldsymbol{p}\right)$ having a constant \textit{step size} $\eta\in\mathbb{R}_{\geq0}$.
    \begin{equation}
    \begin{aligned}
    \label{eq:omd_objective}
        \boldsymbol{p}^{(t+1)}
        =   
        \argmin_{\boldsymbol{p} \in \Delta_{K-1}} \ell^{(t)}\left(\boldsymbol{p}\right)
        + \eta R\left(\boldsymbol{p}\right)
    \end{aligned}    
    \end{equation} 
    \end{remark}

    As long as the regularizer $R\left(\boldsymbol{p}\right)$ in the Remark~\ref{remark:unification} is fixed as the negative entropy, i.e., $R\left(\boldsymbol{p}\right)=\sum_{i=1}^K p_i \log p_i$, this subsumes aggregation strategies proposed in \texttt{FedAvg} \cite{fedavg}, \texttt{AFL} \cite{afl}, \texttt{q-FFL} \cite{qffl}, \texttt{TERM} \cite{term}, and \texttt{PropFair} \cite{propfair}.
    It has an update as follows.
    \begin{equation}
    \begin{aligned}
    \label{eq:eg_update}
        p^{(t+1)}_i \propto {p^{(t)}_i \exp\left( r^{(t)}_i / \eta\right)} 
    \end{aligned}    
    \end{equation}
    This is widely known as EG~\cite{eg}, a special realization of OMD~\cite{md, omd, bregmanmd}.
    We summarize how existing methods can be unified under this OCO framework in Table~\ref{tab:unification}.
    The detailed derivations of mixing coefficients from each method are provided in Appendix~\ref{app:unification}.
    
    To sum up, we can interpret the aggregation mechanism in FL is secretly a result of \textit{the server’s sequential decision making} behind the scene.
    Since the sequential learning scheme is \textit{well-behaved in a sample-deficient setting}, 
    adopting OCO is surely a suitable tactic for the server in that \textit{only a few bits are provided to update the mixing coefficients} in each FL round, e.g, the number of local responses collected from the clients is at most $K$.
    However, existing methods have not been devised with sequential decision making in mind.
    Therefore, one can easily find suboptimal designs inherent in existing methods from an OCO perspective.

\subsection{Sequential Probability Assignment}
\label{subsec:prob_assign}
    To address the client-level fairness, the server should make an adaptive mixing coefficient vector, 
    $\boldsymbol{p}^{(t)}\in\Delta_{K-1}$, for each round $t\in[T]$.
    In other words, the server needs to assign appropriate probabilities sequentially to local updates in every FL communication round.
    
    Notably, this fairly resembles OPS, which seeks to maximize an investor's cumulative profits on a set of $K$ assets during $T$ periods, 
    by assigning his/her wealth $\boldsymbol{p}\in\Delta_{K-1}$ to each asset every time.
    In the OPS, the investor observes a price of all assets, $\boldsymbol{r}^{(t)}\in\mathbb{R}^K$ for each time $t\in[T]$ 
    and accumulates corresponding wealth according to the portfolio $\boldsymbol{p}^{(t)}\in\Delta_{K-1}$. 
    After $T$ periods, achieved cumulative profits is represented as $\prod_{t=1}^T \left(1 + \left\langle\boldsymbol{p}^{(t)},\boldsymbol{r}^{(t)}\right\rangle\right)$, 
    or in the form of logarithmic growth ratio, 
    $\sum_{t=1}^T \log\left(1+\left\langle\boldsymbol{p}^{(t)},\boldsymbol{r}^{(t)}\right\rangle\right)$.
    In other words, one can view that OPS algorithms adopt negative logarithmic growth as a decision loss.
    
    \begin{definition} (Negative Logarithmic Growth as a Decision Loss) 
    For all $t\in[T]$, define a decision loss $\ell^{(t)}:\Delta_{K-1} \times \mathbb{R}^K \rightarrow \mathbb{R}$ as follows.
    \begin{equation}
    \begin{gathered}
    \label{eq:decision_loss}
        \ell^{(t)}(\boldsymbol{p}) = -\log(1+\langle\boldsymbol{p}, \boldsymbol{r}^{(t)}\rangle),
    \end{gathered}
    \end{equation}    
    where $\boldsymbol{p}$ is a decision vector in $\Delta_{K-1}$ and $\boldsymbol{r}^{(t)}$ is a response vector given at time $t$.
    \end{definition}
    
    Again, the OPS algorithm can serve as a metaphor for the central server's fairness-aware online decision making in FL. 
    For example, one can regard a response (i.e., local losses) of $K$ clients at a specific round $t$ the same as returns of assets on the day $t$. 
    Similarly, by considering cumulative losses (i.e., cumulative wealth) achieved until $t$, 
    the server can determine the next mixing coefficients (i.e., portfolio ratios) in $\Delta_{K-1}$. 
    In the same context, the negative logarithmic growth can also be adopted as the decision loss.
    Accordingly, we can adopt well-established OPS strategies for achieving client-level fairness in FL.
    
    It is also intuitive that the negative logarithmic growth is a soft approximation of maximum operator (i.e., log-sum-exp trick, or soft-max).
    Thus, the minimization of (\ref{eq:decision_loss}) w.r.t. $\boldsymbol{p}$ is equal to the maximization of the soft-max approximation of $\sum_{i=1}^Kp_i r_i^{(t)}\equiv\sum_{i=1}^Kp_i F_i(\boldsymbol{\theta}^{(t)})$ w.r.t. $\boldsymbol{p}$, when ignoring the constant and the response transformation is monotonically non-decreasing.
    In turn, along with the minimization objective in terms of parameter as in (\ref{eq:fl_obj}), we can see the overall objective of the federated system can be represented as \textit{the approximation of max-min optimization} problem.
    
    Including OPS, a de facto standard objective for OCO is to minimize the regret defined in (\ref{eq:regret}), 
    with regard to the best decision in hindsight, $\boldsymbol{p}^\star\triangleq\argmin_{\boldsymbol{p}\in\Delta_{K-1}}\sum_{t=1}^T\ell^{(t)}(\boldsymbol{p})$,
    given all decisions $\{ \boldsymbol{p}^{(1)}, ..., \boldsymbol{p}^{(T)} \}$. \cite{oco1, oco2, oco3}
    \begin{equation}
    \begin{gathered}
    \label{eq:regret}
        \text{Regret}^{(T)}(\boldsymbol{p}^{\star}) 
        = 
        \sum\nolimits_{t=1}^T 
        \ell^{(t)}(\boldsymbol{p}^{(t)}) 
        - 
        \sum\nolimits_{t=1}^T \ell^{(t)}\left(\boldsymbol{p}^{\star}\right)
    \end{gathered}
    \end{equation}    
    In finite time $T$, an online decision making strategy should guarantee that the regret grows sublinearly.
    Therefore, when OPS strategies are modified for the fair FL, 
    we should check if the strategy can guarantee vanishing regret upper bound in $T$.
    Besides, we should also consider \textit{the dependency on $K$} due to practical constraints of the federated system.

\section{Proposed Methods}
\label{sec:proposed}

\subsection{Improved Design for Better Decision Making}
\label{subsec:better_design}
    From the Remark~\ref{remark:unification} and Table~\ref{tab:unification}, one can easily notice suboptimal designs of existing methods in terms of OCO, as follows.
    \begin{enumerate}[a)] 
    \label{list:existing_problems}
        \item \label{item:problem1} Existing methods are \textit{stateless} in making a new decision, $p^{(t+1)}_i$. 
        The previous decision is ignored as a fixed value ($p^{(t)}_i=n_i/n$) in the subsequent decision making.
        This naive reliance on static coefficients still runs the risk of violating client-level fairness.  
        \item \label{item:problem2} The decision maker sticks to a \textit{fixed} and \textit{arbitrary} step size $\eta$, or a \textit{fixed} regularizer $R(\boldsymbol{p})$ across $t\in[T]$, which can significantly affect the performance of OCO algorithms and should be manually selected. 
        \item \label{item:problem3} The decision loss is neither \textit{Lipschitz continuous} nor \textit{strictly convex}, which is related to achieving a sublinear regret.
    \end{enumerate}

    As a remedy for handling \ref{item:problem1} and \ref{item:problem2}, 
    the OMD objective for the server (i.e., (\ref{eq:omd_objective})) can be replaced as follows. 
    \begin{equation}
    \begin{aligned}
    \label{eq:ftrl_objective}
        \boldsymbol{p}^{(t+1)} 
        = \argmin_{\boldsymbol{p} \in \Delta_{K-1}} \sum\nolimits_{\tau=1}^{t} \ell^{(\tau)}\left(\boldsymbol{p}\right) 
        + R^{(t+1)}\left(\boldsymbol{p}\right) \\
    \end{aligned}
    \end{equation}
    This is also known as FTRL objective \cite{ftrl1, ftrl2, ftrl3, ftrl4}, 
    which is inherently a \textit{stateful} sequential decision making algorithm 
    that adapts to histories of decision losses, $\sum_{\tau=1}^{t} \ell^{(\tau)}\left(\boldsymbol{p}\right)$,
    where $\ell^{(t)}:\Delta_{K-1} \times \mathbb{R}^K \rightarrow \mathbb{R}$, 
    with the \textit{time-varying} regularizer $R^{(t)}:\Delta_{K-1} \rightarrow \mathbb{R}$. 
    Note that the time-varying regularizer is sometimes represented as, $\eta^{(t+1)}R(\boldsymbol{p})$, 
    a fixed regularizer $R\left(\boldsymbol{p}\right)$ multiplied by a \textit{time-varying} step size, $\eta^{(t+1)}\in\mathbb{R}_{\geq0}$, which can later be automatically determined from the regret analysis
    (see e.g., Remark~\ref{remark:closed_form}).
    
    Additionally, when equipped with the negative logarithmic growth as a decision loss (i.e., eq.~\eqref{eq:decision_loss}), 
    the problem \ref{item:problem3} can be addressed due to its strict convexity and Lipscthiz continuity. 
    (See Lemma \ref{lemma:lipschitz})
    Note that when the loss function is convex, we can run the FTRL with a linearized loss 
    (i.e., $\tilde\ell^{(t)}\left(\boldsymbol{p}\right) 
    =\left\langle \boldsymbol{p}, \boldsymbol{g}^{(t)} \right\rangle$ 
    where $\boldsymbol{g}^{(t)}=\nabla\ell^{(t)}\left(\boldsymbol{p}^{(t)}\right)$). 
    This is useful in that a closed-form update can be obtained thanks to the properties of the Fenchel conjugate (see Remark \ref{remark:closed_form}).

\subsection{\texttt{AAggFF}: Adaptive Aggregation for Fair Federated Learning}
\label{subsec:aaggff}
    Based on the improved objective design derived from the FTRL, we now introduce our methods, \texttt{AAggFF},
    an acronym of \textbf{\underline{A}}daptive \textbf{\underline{Agg}}regation for \textbf{\underline{F}}air \textbf{\underline{F}}ederated Learning).
    Mirroring the practical requirements of FL, we further subdivide into two algorithms: 
    \texttt{AAggFF-S} for the cross-silo setting and \texttt{AAggFF-D} for the cross-device setting.

\subsubsection{\texttt{AAggFF-S}: Algorithm for the Cross-Silo Federated Learning}
\label{subsubsec:aaggff_s}
    In the cross-silo setting, it is typically assumed that \textit{all} $K$ clients participate in $T$ rounds,
    since there are a moderately small number of clients in the federated system.
    Therefore, the server's stateful decision making is beneficial for enhancing overall welfare across federation rounds.
    This is also favorable since existing OPS algorithms can be readily adopted.

\paragraph{Online Newton Step \cite{ons1, ons2}} 
    The ONS algorithm updates a new decision as follows ($\alpha$ and $\beta$ are constants to be determined).
    \begin{equation}
    \begin{gathered}
    \label{eq:ons}
        \boldsymbol{p}^{(t+1)}
        =
        \argmin_{\boldsymbol{p} \in \Delta_{K-1}} \sum\nolimits_{\tau=1}^{t} 
        \tilde\ell^{(\tau)}\left(\boldsymbol{p}\right)
        +
        \frac{\alpha}{2} \Vert \boldsymbol{p} \Vert_2^2 \\
        +
        \frac{\beta}{2} \sum\nolimits_{\tau=1}^{t} 
        (\langle \boldsymbol{g}^{(\tau)}, \boldsymbol{p} - \boldsymbol{p}^{(\tau)}\rangle)^2
    \end{gathered}
    \end{equation}
    The ONS can be reduced to the FTRL objective introduced in (\ref{eq:ftrl_objective}).
    It can be retrieved when we use a linearized loss, 
    $\tilde\ell^{(t)}\left(\boldsymbol{p}\right)
    =
    \left\langle \boldsymbol{p}, \boldsymbol{g}^{(t)} \right\rangle$,
    and the time-varying proximal regularizer, defined as 
    $R^{(t+1)}\left(\boldsymbol{p}\right)=\frac{\alpha}{2}\Vert\boldsymbol{p}\Vert^2_2 + 
    \frac{\beta}{2}\sum_{\tau=1}^{t} \left( \left\langle \boldsymbol{g}^{(\tau)}, \boldsymbol{p} - \boldsymbol{p}^{(\tau)} \right\rangle \right)^2$.
    
    We choose ONS in that its regret is optimal in $T$, which is also a dominating constant for the cross-silo FL setting:
    $\mathcal{O}(L_\infty K\log{T})$ regret upper bound, 
    where $L_\infty$ is the Lipschitz constant of decision loss w.r.t. $\Vert\cdot\Vert_\infty$.
    That is, $L_\infty$ should be \textit{finite} for a vanishing regret (see Theorem~\ref{thm:crosssilo}).

\paragraph{Necessity of Bounded Response}
\label{sec:resp_trans}
    Note that the Lipchitz continuity of the negative logarithmic growth as a decision loss is determined as follows.
    \begin{lemma}
    \label{lemma:lipschitz}
    For all $t\in[T]$, suppose each entry of a response vector $\boldsymbol{r}^{(t)}\in\mathbb{R}^K$ is bounded 
    as $r_i^{(t)}\in[C_1,C_2]$ for some constants $C_1$ and $C_2$ satisfying $0<C_1<C_2$.
    Then, the decision loss $\ell^{(t)}$ defined in (\ref{eq:decision_loss}) is $\frac{C_2}{1+C_1}$-Lipschitz continuous 
    in $\Delta_{K-1}$ w.r.t. $\Vert \cdot \Vert_\infty$.
    \end{lemma}
    
    From now on, all proofs are deferred to Appendix~\ref{app:proofs}. 
    According to the Lemma~\ref{lemma:lipschitz}, the Lipschitz constant of the decision loss, $L_\infty$,
    is dependent upon \textit{the range of a response vector's element}.
    While from the unification result in Table~\ref{tab:unification}, 
    one can easily notice that the response is constructed from local losses collected in round $t$, 
    $F_i(\boldsymbol{\theta}^{(t)})\in\mathbb{R}_{\geq0}, i\in[K]$.
    
    This is a scalar value calculated from a local training set of each client, 
    using the current model $\boldsymbol{\theta}^{(t)}$ \textit{before its local update}. 
    Since the local loss function is typically unbounded above (e.g., cross-entropy), 
    it should be transformed into bounded values to satisfy the Lipschitz continuity. 
    In existing fair FL methods, however, all responses are not bounded above, 
    thus we cannot guarantee the Lipschitz continuity. 
    
    To ensure a bounded response, we propose to use a transformation denoted as $\rho^{(t)}(\cdot)$, 
    inspired by the cumulative distribution function (CDF) as follows.
    \begin{definition} (CDF-driven Response Transformation)
    \label{def:resp_cdf}
        We define $r_i^{(t)} \equiv \rho^{(t)} \left( {F_i(\boldsymbol{\theta}^{(t)})} \right)$,
        each element of the response vector is defined from the corresponding entry of a local loss
        by an element-wise mapping $\rho^{(t)}:\mathbb{R}_{\geq0}\rightarrow[C_1,C_2]$,
        given a pre-defined \texttt{CDF} as: 
        \begin{equation}
        \begin{aligned}
        \label{eq:resp_vec}
            \rho^{(t)} \left( {F_i\left(\boldsymbol{\theta}^{(t)}\right)} \right)
            \triangleq C_1 + (C_2 - C_1)\texttt{CDF} \left( \frac{F_i\left(\boldsymbol{\theta}^{(t)}\right)}{\bar{\mathrm{F}}^{(t)}} \right),
        \end{aligned}
        \end{equation}
        where $\bar{\mathrm{F}}^{(t)} = \frac{1}{\lvert S^{(t)} \rvert}\sum_{i\in S^{(t)}} F_i\left(\boldsymbol{\theta}^{(t)}\right)$, 
        and $S^{(t)}$ is an index set of available clients in $t$.
    \end{definition}
    
    Note again that the larger mixing coefficient should be assigned for the larger local loss.
    In such a perspective, using the CDF for transforming a loss value is an acceptable approach 
    in that the CDF value is a good indicator for estimating \enquote{\textit{how large a specific local loss is}}, 
    relative to local losses from other clients.
    To instill the comparative nature, local losses are divided by the average of observed losses in time $t$ 
    before applying the transformation.
    As a result, all local losses are centered on 1 in expectation.
    See Appendix~\ref{app:cdfs} for detailed discussions.
    
    In summary, the whole procedure of \texttt{AAggFF-S} is illustrated as a pseudocode in Algorithm~\ref{alg:aaggff-s}.

\subsubsection{\texttt{AAggFF-D}: Algorithm for the Cross-Device Federated Learning}
\label{subsubsec:aaggff_d}
    Unlike the cross-silo setting, we cannot be na\"{i}vely edopt existing OCO algorithms for finding adaptive mixing coefficients in the cross-device setting.
    It is attributed to \textit{the large number of participating clients} in this special setting.
    Since the number of participating clients ($K$) is massive (e.g., Android users are over 3 billion \cite{android}), 
    the dependence on $K$ in terms of regret bound and algorithm runtime is as significant as a total communication round $T$. 

\paragraph{Linear Runtime OCO Algorithm}
    The ONS has regret proportional to $K$ and runs in $\mathcal{O}\left(K^2+K^3\right)$\footnote{A generalized projection required for the Online Newton Step can be solved in $\tilde{\mathcal{O}}\left(K^3\right)$ \cite{ons1, ons2}.} per round,
    which is \textit{nearly impossible} to be adopted for the cross-device FL setting due to large $K$, 
    even though the logarithmic regret is guaranteed in $T$.
    Instead, we can exploit the variant of EG adapted to FTRL \cite{oco3}, 
    which can be run in $\mathcal{O}\left(K\right)$ time per round. 
    \begin{equation}
    \begin{gathered}
    \label{eq:lin_ftrl}
        \boldsymbol{p}^{(t+1)}
        =
        \argmin_{\boldsymbol{p} \in \Delta_{K-1}} \sum\nolimits_{\tau=1}^{t} 
        \tilde\ell^{(\tau)}\left(\boldsymbol{p}\right)
        +
        \eta^{(t+1)}R(\boldsymbol{p}),
    \end{gathered}
    \end{equation}
    where $\eta^{(t)}$ is non-decreasing step size across $t\in[T]$, 
    and $R(\boldsymbol{p})=\sum\nolimits_{i=1}^K p_i \log p_i$ is a negative entropy regularizer.
    Still, the regret bound gets worse than that of ONS, as $\mathcal{O}\left(L_\infty \sqrt{T\log{K}}\right)$
    (see Theorem \ref{thm:crossdevice_full}).

\paragraph{Partially Observed Response} 
    The large number of clients coerces the federated system to introduce the \textit{client sampling scheme} in each round.
    Therefore, the decision maker (i.e., the central server) cannot always observe all entries of a response vector per round.
    This is problematic in terms of OCO, since OCO algorithms assume that they can acquire intact response vector for every round $t\in[T]$.
    Instead, when the client sampling is introduced, the learner can only observe entries of sampled client indices in the round $t$, denoted as $S^{(t)}$.
    
    To make a new decision using a \textit{partially observed} response vector, the effect of unobserved entries should be appropriately estimated.
    We solve this problem by adopting a doubly robust (DR) estimator \cite{doublyrobust2, doublyrobust3} for the expectation of the response vector.
    The rationale behind the adoption of the DR estimator is the fact that the unobserved entries are \textit{missing data}.
    
    For handling the missingness problem, the DR estimator combines inverse probability weighting (IPW \cite{exp3}) estimator and imputation mechanism,
    where the former is to adjust the weight of observed entries by the inverse of its observation probability (i.e., client sampling probability), 
    and the latter is to fill unobserved entries with appropriate values specific to a given task.
    
    Similar to the IPW estimator, the DR estimator is an unbiased estimator when the true observation probability is known. 
    Since we sample clients uniformly at random without replacement, \textit{the observation probability is known} 
    (i.e., $C\in(0,1)$) to the algorithm.
    
    \begin{lemma}
    \label{lemma:unbiased_resp}
    Denote $C=P\left(i \in S^{(t)}\right)$ as a client sampling probability in a cross-device FL setting for every round $t\in[T]$. 
    The DR estimator $\breve{\boldsymbol{r}}^{(t)}$, of which element is defined in (\ref{eq:dr_response}) is an unbiased estimator of given partially observed response vector $\boldsymbol{r}^{(t)}$.
    i.e., $\mathbb{E}\left[\breve{\boldsymbol{r}}^{(t)}\right] = \boldsymbol{r}^{(t)}$.
    \begin{equation}
    \begin{gathered}
    \label{eq:dr_response}
        \breve{r}^{(t)}_i = \left(1 - \frac{\mathbb{I}(i \in S^{(t)})}{C}\right)\mathrm{\bar{r}}^{(t)} + \frac{\mathbb{I}\left(i \in S^{(t)}\right)}{C}{r}^{(t)}_i,
    \end{gathered}
    \end{equation}
    where $\bar{\mathrm{r}}^{(t)}=\frac{1}{\left\vert S^{(t)}\right\vert} \sum_{i \in S^{(t)}} r_i^{(t)}$.
    \end{lemma}

    Still, it is required to guarantee that the gradient vector from the DR estimator is also an unbiased estimator of a true gradient of a decision loss.
    Unfortunately, the gradient of a decision loss is \textit{not linear} in the response vector due to its fractional form: 
    ${\boldsymbol{g}}^{(t)}
    =\nabla \ell^{(t)}\left(\boldsymbol{p}^{(t)}\right)
    =-\frac{\boldsymbol{r}^{(t)}}{1 +\left\langle \boldsymbol{p}^{(t)}, \boldsymbol{r}^{(t)} \right\rangle}$.
    
    Therefore, we instead use linearly approximated gradient \textit{w.r.t. a response vector} as follows.
    \begin{lemma}
    \label{lemma:linearized_grad}
        Denote the gradient of a decision loss in terms of a response vector as $\boldsymbol{g}\equiv\mathrm{\mathbf{h}}(\boldsymbol{r}) = \left[h_1(\boldsymbol{r}),...,h_K(\boldsymbol{r})\right]^\top = -\frac{\boldsymbol{r}}{1 +\left\langle \boldsymbol{p}, \boldsymbol{r} \right\rangle}$. 
        It can be linearized for the response vector into ${\tilde{\boldsymbol{g}}}\equiv\tilde{\mathrm{\mathbf{h}}}(\boldsymbol{r})$, 
        given a reference $\boldsymbol{r}_0$ as follows. 
        (Note that the superscript ${(t)}$ is omitted for a brevity of notation)
        \begin{equation}
        \begin{gathered}
        \label{eq:linearized_gradient}
            {\boldsymbol{g}}
            \approx
            {\tilde{\boldsymbol{g}}}
            \equiv
            \tilde{\mathrm{\mathbf{h}}}(\boldsymbol{r})
            =
            -\frac{\boldsymbol{r}}{1 +\left\langle
            \boldsymbol{p},
            \boldsymbol{r}_0
            \right\rangle}
            +
            \frac{\boldsymbol{r}_0 \boldsymbol{p}^\top (\boldsymbol{r} - \boldsymbol{r}_0)}{(1 +\left\langle
            \boldsymbol{p},
            \boldsymbol{r}_0
            \right\rangle)^2}
        \end{gathered}
        \end{equation}
        Further denote $\breve{\boldsymbol{g}}$ as a gradient estimate from (\ref{eq:linearized_gradient}) 
        using the DR estimator of a response vector according to (\ref{eq:dr_response}),
        at an arbitrary reference $\boldsymbol{r}_0$. 
        Then, $\breve{\boldsymbol{g}}$ is an unbiased estimator of the linearized gradient of a decision loss at $\boldsymbol{r}_0$, 
        which is close to the true gradient: $\mathbb{E}\left[\breve{\boldsymbol{g}}\right] = \tilde{\boldsymbol{g}}\approx\boldsymbol{g}$.
    \end{lemma}

    As suggested in (\ref{eq:dr_response}), we similarly set the reference as an average of observed responses at round $t$, i.e., $\boldsymbol{r}_0^{(t)} = \mathrm{\bar{r}}^{(t)}\boldsymbol{1}_K$.
    It is a valid choice in that dominating unobserved entries are imputed by the average of observed responses as in (\ref{eq:dr_response}).
    
    To sum up, we can update a new decision using this unbiased and linearly approximated gradient estimator even if only a partially observed response vector is provided (i.e., mixing coefficients of unsampled clients can also be updated).
    Note that the linearized gradient calculated from the DR estimator, $\breve{\boldsymbol{g}}$, has finite norm w.r.t. $\Vert \cdot \Vert_\infty$ (see Lemma~\ref{lemma:lipschitz_lin_grad}).

\paragraph{Closed-Form Update} 
    Especially for the cross-device setting, we can obtain a closed-form update of the objective (\ref{eq:lin_ftrl}), 
    which is due to the property of Fenchel conjugate.
    
    \begin{remark}
    \label{remark:closed_form} 
    The objective of \texttt{AAggFF-D} stated in (\ref{eq:lin_ftrl}) has a closed-form update formula as follows. \cite{oco3}
    \begin{equation}
    \begin{aligned}
    \label{eq:closed_form}
        {p}^{(t+1)}_i 
        \propto
        { \exp\left(-\frac{\sqrt{\log{K}}\sum_{\tau=1}^{t} \breve{g}^{(\tau)}_i}{\breve{L}_\infty\sqrt{t+1}}\right) }
    \end{aligned}
    \end{equation}
    \end{remark}
    It is equivalent to setting the time-varying step size as $\eta^{(t)}=\frac{ \breve{L}_\infty\sqrt{t} }{ \sqrt{\log{K}} }$. 
    Note that $\breve{g}^{(t)}_i$ is an entry of gradient from DR estimator defined in Lemma~\ref{lemma:linearized_grad} 
    and $\breve{L}_\infty$ is a corresponding Lipschitz constant satisfying $\Vert \breve{\boldsymbol{g}} \Vert_\infty \leq \breve{L}_\infty$
    stated in Lemma~\ref{lemma:lipschitz_lin_grad}. 
    See Appendix~\ref{app:deriv_closed} for the derivation.
    
    In summary, the whole procedure of \texttt{AAggFF-D} is illustrated in Algorithm~\ref{alg:aaggff-d}.

\begin{table*}[!ht]
\centering
\caption{Comparison Results of \texttt{AAggFF-S} in the Cross-Silo Setting}
\label{tab:result_silo}
\resizebox{\textwidth}{!}{%
\renewcommand{\arraystretch}{0.9}
\scriptsize
\begin{tabular}{lcccccccccccc}
\toprule
\textbf{Dataset} &
  \multicolumn{4}{c}{\textbf{Berka}} &
  \multicolumn{4}{c}{\textbf{MQP}} &
  \multicolumn{4}{c}{\textbf{ISIC}} \\
 &
  \multicolumn{4}{c}{\tiny(AUROC)} &
  \multicolumn{4}{c}{\tiny(AUROC)} &
  \multicolumn{4}{c}{\tiny(Acc. 5)} \\ \cmidrule(l){2-13}
\multirow{-2}{*}{\textbf{Method}} &
  \begin{tabular}[c]{@{}c@{}}Avg.\\ ($\uparrow$)\end{tabular} &
  \begin{tabular}[c]{@{}c@{}}Worst\\ ($\uparrow$)\end{tabular} &
  \begin{tabular}[c]{@{}c@{}}Best\\ ($\uparrow$)\end{tabular} &
  \begin{tabular}[c]{@{}c@{}}Gini\\ ($\downarrow$)\end{tabular} &
  \begin{tabular}[c]{@{}c@{}}Avg.\\ ($\uparrow$)\end{tabular} &
  \begin{tabular}[c]{@{}c@{}}Worst\\ ($\uparrow$)\end{tabular} &
  \begin{tabular}[c]{@{}c@{}}Best\\ ($\uparrow$)\end{tabular} &
  \begin{tabular}[c]{@{}c@{}}Gini\\ ($\downarrow$)\end{tabular} &
  \begin{tabular}[c]{@{}c@{}}Avg.\\ ($\uparrow$)\end{tabular} &
  \begin{tabular}[c]{@{}c@{}}Worst\\ ($\uparrow$)\end{tabular} &
  \begin{tabular}[c]{@{}c@{}}Best\\ ($\uparrow$)\end{tabular} &
  \begin{tabular}[c]{@{}c@{}}Gini\\ ($\downarrow$)\end{tabular} \\ \midrule
\begin{tabular}[c]{@{}l@{}}\texttt{FedAvg}\\ \tiny\cite{fedavg}\end{tabular} &
  \begin{tabular}[c]{@{}c@{}}80.09\\ \tiny \color[HTML]{9B9B9B} (2.45)\end{tabular} &
  \begin{tabular}[c]{@{}c@{}}48.06\\ \tiny \color[HTML]{9B9B9B} (25.15)\end{tabular} &
  \begin{tabular}[c]{@{}c@{}}\textbf{99.03}\\ \tiny \color[HTML]{9B9B9B} (1.37)\end{tabular} &
  \multicolumn{1}{c|}{\begin{tabular}[c]{@{}c@{}}10.87\\ \tiny \color[HTML]{9B9B9B} (4.11)\end{tabular}} &
  \begin{tabular}[c]{@{}c@{}}56.06\\ \tiny \color[HTML]{9B9B9B} (0.06)\end{tabular} &
  \begin{tabular}[c]{@{}c@{}}41.03\\ \tiny \color[HTML]{9B9B9B} (4.33)\end{tabular} &
  \begin{tabular}[c]{@{}c@{}}76.31\\ \tiny \color[HTML]{9B9B9B} (8.42)\end{tabular} &
  \multicolumn{1}{c|}{\begin{tabular}[c]{@{}c@{}}8.63\\ \tiny \color[HTML]{9B9B9B} (0.91)\end{tabular}} &
  \begin{tabular}[c]{@{}c@{}}87.42\\ \tiny \color[HTML]{9B9B9B} (2.11)\end{tabular} &
  \begin{tabular}[c]{@{}c@{}}69.92\\ \tiny \color[HTML]{9B9B9B} (6.78)\end{tabular} &
  \begin{tabular}[c]{@{}c@{}}92.57\\ \tiny \color[HTML]{9B9B9B} (2.56)\end{tabular} &
  \begin{tabular}[c]{@{}c@{}}4.84\\ \tiny \color[HTML]{9B9B9B} (1.17)\end{tabular} \\
\begin{tabular}[c]{@{}l@{}}\texttt{AFL}\\ \tiny\cite{afl}\end{tabular} &
  \begin{tabular}[c]{@{}c@{}}79.70\\ \tiny \color[HTML]{9B9B9B} (4.14)\end{tabular} &
  \begin{tabular}[c]{@{}c@{}}49.02\\ \tiny \color[HTML]{9B9B9B} (25.89)\end{tabular} &
  \begin{tabular}[c]{@{}c@{}}\underline{98.55}\\ \tiny \color[HTML]{9B9B9B} (2.05)\end{tabular} &
  \multicolumn{1}{c|}{\begin{tabular}[c]{@{}c@{}}10.58\\ \tiny \color[HTML]{9B9B9B} (5.03)\end{tabular}} &
  \begin{tabular}[c]{@{}c@{}}56.01\\ \tiny \color[HTML]{9B9B9B} (0.30)\end{tabular} &
  \begin{tabular}[c]{@{}c@{}}41.28\\ \tiny \color[HTML]{9B9B9B} (3.92)\end{tabular} &
  \begin{tabular}[c]{@{}c@{}}75.54\\ \tiny \color[HTML]{9B9B9B} (6.77)\end{tabular} &
  \multicolumn{1}{c|}{\begin{tabular}[c]{@{}c@{}}\underline{8.56}\\ \tiny \color[HTML]{9B9B9B} (1.24)\end{tabular}} &
  \begin{tabular}[c]{@{}c@{}}87.39\\ \tiny \color[HTML]{9B9B9B} (2.31)\end{tabular} &
  \begin{tabular}[c]{@{}c@{}}68.17\\ \tiny \color[HTML]{9B9B9B} (10.09)\end{tabular} &
  \begin{tabular}[c]{@{}c@{}}93.33\\ \tiny \color[HTML]{9B9B9B} (2.18)\end{tabular} &
  \begin{tabular}[c]{@{}c@{}}4.80\\ \tiny \color[HTML]{9B9B9B} (1.74)\end{tabular} \\
\begin{tabular}[c]{@{}l@{}}\texttt{q-FedAvg}\\ \tiny\cite{qffl}\end{tabular} &
  \begin{tabular}[c]{@{}c@{}}79.98\\ \tiny \color[HTML]{9B9B9B} (3.89)\end{tabular} &
  \begin{tabular}[c]{@{}c@{}}\underline{49.44}\\ \tiny \color[HTML]{9B9B9B} (26.15)\end{tabular} &
  \begin{tabular}[c]{@{}c@{}}98.07\\ \tiny \color[HTML]{9B9B9B} (2.73)\end{tabular} &
  \multicolumn{1}{c|}{\begin{tabular}[c]{@{}c@{}}10.62\\ \tiny \color[HTML]{9B9B9B} (5.22)\end{tabular}} &
  \begin{tabular}[c]{@{}c@{}}\textbf{56.89}\\ \tiny \color[HTML]{9B9B9B} (0.42)\end{tabular} &
  \begin{tabular}[c]{@{}c@{}}40.22\\ \tiny \color[HTML]{9B9B9B} (3.06)\end{tabular} &
  \begin{tabular}[c]{@{}c@{}}\textbf{79.38}\\ \tiny \color[HTML]{9B9B9B} (9.09)\end{tabular} &
  \multicolumn{1}{c|}{\begin{tabular}[c]{@{}c@{}}8.68\\ \tiny \color[HTML]{9B9B9B} (0.57)\end{tabular}} &
  \begin{tabular}[c]{@{}c@{}}41.59\\ \tiny \color[HTML]{9B9B9B} (16.22)\end{tabular} &
  \begin{tabular}[c]{@{}c@{}}20.38\\ \tiny \color[HTML]{9B9B9B} (23.24)\end{tabular} &
  \begin{tabular}[c]{@{}c@{}}58.08\\ \tiny \color[HTML]{9B9B9B} (28.52)\end{tabular} &
  \begin{tabular}[c]{@{}c@{}}22.25\\ \tiny \color[HTML]{9B9B9B} (10.02)\end{tabular} \\
\begin{tabular}[c]{@{}l@{}}\texttt{TERM}\\ \tiny\cite{term}\end{tabular} &
  \begin{tabular}[c]{@{}c@{}}\underline{80.11}\\ \tiny \color[HTML]{9B9B9B} (3.08)\end{tabular} &
  \begin{tabular}[c]{@{}c@{}}48.96\\ \tiny \color[HTML]{9B9B9B} (25.79)\end{tabular} &
  \begin{tabular}[c]{@{}c@{}}\textbf{99.03}\\ \tiny \color[HTML]{9B9B9B} (1.37)\end{tabular} &
  \multicolumn{1}{c|}{\begin{tabular}[c]{@{}c@{}}10.86\\ \tiny \color[HTML]{9B9B9B} (4.73)\end{tabular}} &
  \begin{tabular}[c]{@{}c@{}}56.47\\ \tiny \color[HTML]{9B9B9B} (0.19)\end{tabular} &
  \begin{tabular}[c]{@{}c@{}}40.73\\ \tiny \color[HTML]{9B9B9B} (4.36)\end{tabular} &
  \begin{tabular}[c]{@{}c@{}}76.80\\ \tiny \color[HTML]{9B9B9B} (8.30)\end{tabular} &
  \multicolumn{1}{c|}{\begin{tabular}[c]{@{}c@{}}8.67\\ \tiny \color[HTML]{9B9B9B} (1.43)\end{tabular}} &
  \begin{tabular}[c]{@{}c@{}}\underline{87.89}\\ \tiny \color[HTML]{9B9B9B} (1.69)\end{tabular} &
  \begin{tabular}[c]{@{}c@{}}\underline{77.32}\\ \tiny \color[HTML]{9B9B9B} (5.84)\end{tabular} &
  \begin{tabular}[c]{@{}c@{}}\underline{96.00}\\ \tiny \color[HTML]{9B9B9B} (3.27)\end{tabular} &
  \begin{tabular}[c]{@{}c@{}}\underline{3.77}\\ \tiny \color[HTML]{9B9B9B} (0.94)\end{tabular} \\
\begin{tabular}[c]{@{}l@{}}\texttt{FedMGDA}\\ \tiny\cite{fedmgda}\end{tabular} &
  \begin{tabular}[c]{@{}c@{}}79.24\\ \tiny \color[HTML]{9B9B9B} (2.96)\end{tabular} &
  \begin{tabular}[c]{@{}c@{}}46.38\\ \tiny \color[HTML]{9B9B9B} (24.11)\end{tabular} &
  \begin{tabular}[c]{@{}c@{}}\textbf{99.03}\\ \tiny \color[HTML]{9B9B9B} (1.37)\end{tabular} &
  \multicolumn{1}{c|}{\begin{tabular}[c]{@{}c@{}}11.64\\ \tiny \color[HTML]{9B9B9B} (4.84)\end{tabular}} &
  \begin{tabular}[c]{@{}c@{}}53.02\\ \tiny \color[HTML]{9B9B9B} (1.67)\end{tabular} &
  \begin{tabular}[c]{@{}c@{}}34.91\\ \tiny \color[HTML]{9B9B9B} (2.22)\end{tabular} &
  \begin{tabular}[c]{@{}c@{}}69.65\\ \tiny \color[HTML]{9B9B9B} (3.89)\end{tabular} &
  \multicolumn{1}{c|}{\begin{tabular}[c]{@{}c@{}}10.33\\ \tiny \color[HTML]{9B9B9B} (0.44)\end{tabular}} &
  \begin{tabular}[c]{@{}c@{}}42.36\\ \tiny \color[HTML]{9B9B9B} (14.94)\end{tabular} &
  \begin{tabular}[c]{@{}c@{}}21.44\\ \tiny \color[HTML]{9B9B9B} (21.30)\end{tabular} &
  \begin{tabular}[c]{@{}c@{}}59.21\\ \tiny \color[HTML]{9B9B9B} (28.52)\end{tabular} &
  \begin{tabular}[c]{@{}c@{}}22.25\\ \tiny \color[HTML]{9B9B9B} (10.02)\end{tabular} \\
\begin{tabular}[c]{@{}l@{}}\texttt{PropFair}\\ \tiny\cite{propfair}\end{tabular} &
  \begin{tabular}[c]{@{}c@{}}79.61\\ \tiny \color[HTML]{9B9B9B} (4.49)\end{tabular} &
  \begin{tabular}[c]{@{}c@{}}\underline{49.44}\\ \tiny \color[HTML]{9B9B9B} (26.15)\end{tabular} &
  \begin{tabular}[c]{@{}c@{}}98.07\\ \tiny \color[HTML]{9B9B9B} (2.73)\end{tabular} &
  \multicolumn{1}{c|}{\begin{tabular}[c]{@{}c@{}}\underline{10.47}\\ \tiny \color[HTML]{9B9B9B} (5.04)\end{tabular}} &
  \begin{tabular}[c]{@{}c@{}}56.60\\ \tiny \color[HTML]{9B9B9B} (0.39)\end{tabular} &
  \begin{tabular}[c]{@{}c@{}}\underline{41.71}\\ \tiny \color[HTML]{9B9B9B} (3.80)\end{tabular} &
  \begin{tabular}[c]{@{}c@{}}\underline{79.09}\\ \tiny \color[HTML]{9B9B9B} (7.40)\end{tabular} &
  \multicolumn{1}{c|}{\begin{tabular}[c]{@{}c@{}}8.74\\ \tiny \color[HTML]{9B9B9B} (0.87)\end{tabular}} &
  \begin{tabular}[c]{@{}c@{}}83.88\\ \tiny \color[HTML]{9B9B9B} (2.50)\end{tabular} &
  \begin{tabular}[c]{@{}c@{}}58.36\\ \tiny \color[HTML]{9B9B9B} (11.63)\end{tabular} &
  \begin{tabular}[c]{@{}c@{}}91.35\\ \tiny \color[HTML]{9B9B9B} (2.48)\end{tabular} &
  \begin{tabular}[c]{@{}c@{}}7.91\\ \tiny \color[HTML]{9B9B9B} (2.10)\end{tabular} \\
\rowcolor[HTML]{FFF5E6} 
\begin{tabular}[c]{@{}l@{}}\texttt{AAggFF-S}\\ (Proposed)\end{tabular} &
  \begin{tabular}[c]{@{}c@{}}\textbf{80.93}\\ \tiny \color[HTML]{9B9B9B} (2.96)\end{tabular} &
  \begin{tabular}[c]{@{}c@{}}\textbf{52.08}\\ \tiny \color[HTML]{9B9B9B} (23.59)\end{tabular} &
  \cellcolor[HTML]{FFF5E6}\begin{tabular}[c]{@{}c@{}}\textbf{99.03}\\ \tiny \color[HTML]{9B9B9B} (1.37)\end{tabular} &
  \multicolumn{1}{c|}{\cellcolor[HTML]{FFF5E6}\begin{tabular}[c]{@{}c@{}}\textbf{10.16}\\ \tiny \color[HTML]{9B9B9B} (3.80)\end{tabular}} &
  \begin{tabular}[c]{@{}c@{}}\underline{56.63}\\ \tiny \color[HTML]{9B9B9B} (0.54)\end{tabular} &
  \begin{tabular}[c]{@{}c@{}}\textbf{41.79}\\ \tiny \color[HTML]{9B9B9B} (4.43)\end{tabular} &
  \begin{tabular}[c]{@{}c@{}}75.56\\ \tiny \color[HTML]{9B9B9B} (6.53)\end{tabular} &
  \multicolumn{1}{c|}{\cellcolor[HTML]{FFF5E6}\begin{tabular}[c]{@{}c@{}}\textbf{8.38}\\ \tiny \color[HTML]{9B9B9B} (0.77)\end{tabular}} &
  \begin{tabular}[c]{@{}c@{}}\textbf{89.76}\\ \tiny \color[HTML]{9B9B9B} (1.03)\end{tabular} &
  \begin{tabular}[c]{@{}c@{}}\textbf{85.17}\\ \tiny \color[HTML]{9B9B9B} (3.87)\end{tabular} &
  \begin{tabular}[c]{@{}c@{}}\textbf{98.22}\\ \tiny \color[HTML]{9B9B9B} (1.66)\end{tabular} &
  \begin{tabular}[c]{@{}c@{}}\textbf{2.52}\\ \tiny \color[HTML]{9B9B9B} (0.38)\end{tabular} \\ \bottomrule
\end{tabular}%
}
\end{table*}
\begin{table*}[!ht]
\centering
\caption{Comparison Results of \texttt{AAggFF-D} in the Cross-Device Setting}
\label{tab:result_device}
\resizebox{\textwidth}{!}{%
\renewcommand{\arraystretch}{0.9}
\scriptsize
\begin{tabular}{!{}lcccccccccccc!{}}
\toprule
\textbf{Dataset} &
  \multicolumn{4}{c}{\textbf{CelebA}} &
  \multicolumn{4}{c}{\textbf{Reddit}} &
  \multicolumn{4}{c}{\textbf{SpeechCommands}} \\
 &
  \multicolumn{4}{c}{\scriptsize(Acc. 1)} &
  \multicolumn{4}{c}{\scriptsize(Acc. 1)} &
  \multicolumn{4}{c}{\scriptsize(Acc. 5)} \\ \cmidrule(l){2-13} 
\multirow{-2}{*}{\textbf{Method}} &
  \begin{tabular}[c]{@{}c@{}}Avg.\\ ($\uparrow$)\end{tabular} &
  \begin{tabular}[c]{@{}c@{}}Worst\\ 10\% ($\uparrow$)\end{tabular} &
  \begin{tabular}[c]{@{}c@{}}Best\\ 10\%($\uparrow$)\end{tabular} &
  \begin{tabular}[c]{@{}c@{}}Gini\\ ($\downarrow$)\end{tabular} &
  \begin{tabular}[c]{@{}c@{}}Avg.\\ ($\uparrow$)\end{tabular} &
  \begin{tabular}[c]{@{}c@{}}Worst\\ 10\%($\uparrow$)\end{tabular} &
  \begin{tabular}[c]{@{}c@{}}Best\\ 10\%($\uparrow$)\end{tabular} &
  \begin{tabular}[c]{@{}c@{}}Gini\\ ($\downarrow$)\end{tabular} &
  \begin{tabular}[c]{@{}c@{}}Avg.\\ ($\uparrow$)\end{tabular} &
  \begin{tabular}[c]{@{}c@{}}Worst\\ 10\%($\uparrow$)\end{tabular} &
  \begin{tabular}[c]{@{}c@{}}Best\\ 10\%($\uparrow$)\end{tabular} &
  \begin{tabular}[c]{@{}c@{}}Gini\\ ($\downarrow$)\end{tabular} \\ \midrule
\begin{tabular}[c]{@{}l@{}}\texttt{FedAvg}\\ \tiny\cite{fedavg}\end{tabular} &
  \begin{tabular}[c]{@{}c@{}}90.79\\ \tiny \color[HTML]{9B9B9B} (0.53)\end{tabular} &
  \begin{tabular}[c]{@{}c@{}}\underline{55.76} \\ \tiny \color[HTML]{9B9B9B} (0.84)\end{tabular} &
  \begin{tabular}[c]{@{}c@{}}\underline{100.00}\\ \tiny \color[HTML]{9B9B9B} (0.00)\end{tabular} &
  \multicolumn{1}{c|}{\begin{tabular}[c]{@{}c@{}}7.86\\ \tiny \color[HTML]{9B9B9B} (0.30)\end{tabular}} &
  \begin{tabular}[c]{@{}c@{}}10.76\\ \tiny \color[HTML]{9B9B9B} (1.45)\end{tabular} &
  \begin{tabular}[c]{@{}c@{}}2.50\\ \tiny \color[HTML]{9B9B9B} (0.21)\end{tabular} &
  \begin{tabular}[c]{@{}c@{}}20.86\\ \tiny \color[HTML]{9B9B9B} (3.64)\end{tabular} &
  \multicolumn{1}{c|}{\begin{tabular}[c]{@{}c@{}}25.66\\ \tiny \color[HTML]{9B9B9B} (0.49)\end{tabular}} &
  \begin{tabular}[c]{@{}c@{}}\underline{75.51}\\ \tiny \color[HTML]{9B9B9B} (1.08)\end{tabular} &
  \begin{tabular}[c]{@{}c@{}}7.93\\ \tiny \color[HTML]{9B9B9B} (2.87)\end{tabular} &
  \begin{tabular}[c]{@{}c@{}}\underline{100.00}\\ \tiny \color[HTML]{9B9B9B} (0.00)\end{tabular} &
  \begin{tabular}[c]{@{}c@{}}24.58\\ \tiny \color[HTML]{9B9B9B} (1.34)\end{tabular} \\
\begin{tabular}[c]{@{}l@{}}\texttt{q-FedAvg}\\ \tiny\cite{qffl}\end{tabular} &
  \begin{tabular}[c]{@{}c@{}}\underline{90.88}\\ \tiny \color[HTML]{9B9B9B} (0.19)\end{tabular} &
  \begin{tabular}[c]{@{}c@{}}55.73\\ \tiny \color[HTML]{9B9B9B} (0.85)\end{tabular} &
  \begin{tabular}[c]{@{}c@{}}\underline{100.00}\\ \tiny \color[HTML]{9B9B9B} (0.00)\end{tabular} &
  \multicolumn{1}{c|}{\begin{tabular}[c]{@{}c@{}}\underline{7.82}\\ \tiny \color[HTML]{9B9B9B} (0.21)\end{tabular}} &
  \begin{tabular}[c]{@{}c@{}}\underline{12.76}\\ \tiny \color[HTML]{9B9B9B} (0.32)\end{tabular} &
  \begin{tabular}[c]{@{}c@{}}\underline{3.38}\\ \tiny \color[HTML]{9B9B9B} (0.20)\end{tabular} &
  \begin{tabular}[c]{@{}c@{}}\underline{21.81}\\ \tiny \color[HTML]{9B9B9B} (0.19)\end{tabular} &
  \multicolumn{1}{c|}{\begin{tabular}[c]{@{}c@{}}\underline{23.34}\\ \tiny \color[HTML]{9B9B9B} (0.34)\end{tabular}} &
  \begin{tabular}[c]{@{}c@{}}73.34\\ \tiny \color[HTML]{9B9B9B} (0.47)\end{tabular} &
  \begin{tabular}[c]{@{}c@{}}\underline{11.19}\\ \tiny \color[HTML]{9B9B9B} (0.47)\end{tabular} &
  \begin{tabular}[c]{@{}c@{}}\underline{100.00}\\ \tiny \color[HTML]{9B9B9B} (0.00)\end{tabular} &
  \begin{tabular}[c]{@{}c@{}}\underline{23.16}\\ \tiny \color[HTML]{9B9B9B} (0.13)\end{tabular} \\
\begin{tabular}[c]{@{}l@{}}\texttt{TERM}\\ \tiny\cite{term}\end{tabular} &
  \begin{tabular}[c]{@{}c@{}}90.71\\ \tiny \color[HTML]{9B9B9B} (0.65)\end{tabular} &
  \begin{tabular}[c]{@{}c@{}}55.66\\ \tiny \color[HTML]{9B9B9B} (0.93)\end{tabular} &
  \begin{tabular}[c]{@{}c@{}}\underline{100.00}\\ \tiny \color[HTML]{9B9B9B} (0.00)\end{tabular} &
  \multicolumn{1}{c|}{\begin{tabular}[c]{@{}c@{}}7.90\\ \tiny \color[HTML]{9B9B9B} (0.38)\end{tabular}} &
  \begin{tabular}[c]{@{}c@{}}12.02\\ \tiny \color[HTML]{9B9B9B} (0.16)\end{tabular} &
  \begin{tabular}[c]{@{}c@{}}2.85\\ \tiny \color[HTML]{9B9B9B} (0.41)\end{tabular} &
  \begin{tabular}[c]{@{}c@{}}20.74\\ \tiny \color[HTML]{9B9B9B} (0.65)\end{tabular} &
  \multicolumn{1}{c|}{\begin{tabular}[c]{@{}c@{}}24.15\\ \tiny \color[HTML]{9B9B9B} (1.05)\end{tabular}} &
  \begin{tabular}[c]{@{}c@{}}70.90\\ \tiny \color[HTML]{9B9B9B} (2.96)\end{tabular} &
  \begin{tabular}[c]{@{}c@{}}5.98\\ \tiny \color[HTML]{9B9B9B} (1.10)\end{tabular} &
  \begin{tabular}[c]{@{}c@{}}\underline{100.00}\\ \tiny \color[HTML]{9B9B9B} (0.00)\end{tabular} &
  \begin{tabular}[c]{@{}c@{}}26.37\\ \tiny \color[HTML]{9B9B9B} (1.32)\end{tabular} \\
\begin{tabular}[c]{@{}l@{}}\texttt{FedMGDA}\\ \tiny\cite{fedmgda}\end{tabular} &
  \begin{tabular}[c]{@{}c@{}}88.33\\ \tiny \color[HTML]{9B9B9B} (0.63)\end{tabular} &
  \begin{tabular}[c]{@{}c@{}}48.60\\ \tiny \color[HTML]{9B9B9B} (25.85)\end{tabular} &
  \begin{tabular}[c]{@{}c@{}}\underline{100.00}\\ \tiny \color[HTML]{9B9B9B} (0.00)\end{tabular} &
  \multicolumn{1}{c|}{\begin{tabular}[c]{@{}c@{}}9.75\\ \tiny \color[HTML]{9B9B9B} (0.59)\end{tabular}} &
  \begin{tabular}[c]{@{}c@{}}10.58\\ \tiny \color[HTML]{9B9B9B} (0.18)\end{tabular} &
  \begin{tabular}[c]{@{}c@{}}2.35\\ \tiny \color[HTML]{9B9B9B} (0.20)\end{tabular} &
  \begin{tabular}[c]{@{}c@{}}19.09\\ \tiny \color[HTML]{9B9B9B} (0.62)\end{tabular} &
  \multicolumn{1}{c|}{\begin{tabular}[c]{@{}c@{}}25.20\\ \tiny \color[HTML]{9B9B9B} (0.22)\end{tabular}} &
  \begin{tabular}[c]{@{}c@{}}72.45\\ \tiny \color[HTML]{9B9B9B} (1.88)\end{tabular} &
  \begin{tabular}[c]{@{}c@{}}9.65\\ \tiny \color[HTML]{9B9B9B} (2.90)\end{tabular} &
  \begin{tabular}[c]{@{}c@{}}\underline{100.00}\\ \tiny \color[HTML]{9B9B9B} (0.00)\end{tabular} &
  \begin{tabular}[c]{@{}c@{}}23.68\\ \tiny \color[HTML]{9B9B9B} (1.27)\end{tabular} \\
\begin{tabular}[c]{@{}l@{}}\texttt{PropFair}\\ \tiny\cite{propfair}\end{tabular} &
  \begin{tabular}[c]{@{}c@{}}87.25\\ \tiny \color[HTML]{9B9B9B} (5.01)\end{tabular} &
  \begin{tabular}[c]{@{}c@{}}48.11\\ \tiny \color[HTML]{9B9B9B} (10.03)\end{tabular} &
  \begin{tabular}[c]{@{}c@{}}\underline{100.00}\\ \tiny \color[HTML]{9B9B9B} (0.00)\end{tabular} &
  \multicolumn{1}{c|}{\begin{tabular}[c]{@{}c@{}}10.39\\ \tiny \color[HTML]{9B9B9B} (3.43)\end{tabular}} &
  \begin{tabular}[c]{@{}c@{}}11.26\\ \tiny \color[HTML]{9B9B9B} (0.71)\end{tabular} &
  \begin{tabular}[c]{@{}c@{}}1.95\\ \tiny \color[HTML]{9B9B9B} (0.32)\end{tabular} &
  \begin{tabular}[c]{@{}c@{}}21.33\\ \tiny \color[HTML]{9B9B9B} (0.92)\end{tabular} &
  \multicolumn{1}{c|}{\begin{tabular}[c]{@{}c@{}}25.97\\ \tiny \color[HTML]{9B9B9B} (1.02)\end{tabular}} &
  \begin{tabular}[c]{@{}c@{}}73.64\\ \tiny \color[HTML]{9B9B9B} (3.31)\end{tabular} &
  \begin{tabular}[c]{@{}c@{}}7.30\\ \tiny \color[HTML]{9B9B9B} (1.02)\end{tabular} &
  \begin{tabular}[c]{@{}c@{}}\underline{100.00}\\ \tiny \color[HTML]{9B9B9B} (0.00)\end{tabular} &
  \begin{tabular}[c]{@{}c@{}}24.97\\ \tiny \color[HTML]{9B9B9B} (1.09)\end{tabular} \\
\rowcolor[HTML]{FFF5E6} 
\begin{tabular}[c]{@{}l@{}}\texttt{AAggFF-D}\\ (Proposed)\end{tabular} &
  \begin{tabular}[c]{@{}c@{}}\textbf{91.27}\\ \tiny \color[HTML]{9B9B9B} (0.07)\end{tabular} &
  \begin{tabular}[c]{@{}c@{}}\textbf{56.71}\\ \tiny \color[HTML]{9B9B9B} (0.08)\end{tabular} &
  \cellcolor[HTML]{FFF5E6}\begin{tabular}[c]{@{}c@{}}\underline{100.00}\\ \tiny \color[HTML]{9B9B9B} (0.00)\end{tabular} &
  \multicolumn{1}{c|}{\cellcolor[HTML]{FFF5E6}\begin{tabular}[c]{@{}c@{}}\textbf{7.54}\\ \tiny \color[HTML]{9B9B9B} (0.04)\end{tabular}} &
  \begin{tabular}[c]{@{}c@{}}\textbf{12.95}\\ \tiny \color[HTML]{9B9B9B} (0.39)\end{tabular} &
  \begin{tabular}[c]{@{}c@{}}\textbf{4.75}\\ \tiny \color[HTML]{9B9B9B} (0.76)\end{tabular} &
  \begin{tabular}[c]{@{}c@{}}\textbf{22.81}\\ \tiny \color[HTML]{9B9B9B} (1.36)\end{tabular} &
  \multicolumn{1}{c|}{\cellcolor[HTML]{FFF5E6}\begin{tabular}[c]{@{}c@{}}\textbf{22.59}\\ \tiny \color[HTML]{9B9B9B} (0.28)\end{tabular}} &
  \begin{tabular}[c]{@{}c@{}}\textbf{76.68}\\ \tiny \color[HTML]{9B9B9B} (0.80)\end{tabular} &
  \begin{tabular}[c]{@{}c@{}}\textbf{14.54}\\ \tiny \color[HTML]{9B9B9B} (2.58)\end{tabular} &
  \cellcolor[HTML]{FFF5E6}\begin{tabular}[c]{@{}c@{}}\underline{100.00}\\ \tiny \color[HTML]{9B9B9B} (0.00)\end{tabular} &
  \begin{tabular}[c]{@{}c@{}}\textbf{21.42}\\ \tiny \color[HTML]{9B9B9B} (0.81)\end{tabular} \\ \bottomrule
\end{tabular}%
}
\end{table*}

\section{Regret Analysis}
\label{sec:analyses}
    In this section, we provide theoretical guarantees of our methods, \texttt{AAggFF-S} and \texttt{AAggFF-D} in terms of sequential decision making.
    The common objective for OCO algorithms is to \textit{minimize the regret} across a sequence of decision losses in eq.~\eqref{eq:regret}.
    We provide sublinear regret upper bounds in terms of $T$ as follows.
    
    \begin{theorem}
    \label{thm:crosssilo} (Regret Upper Bound for \texttt{AAggFF-S} (i.e., ONS \cite{ons1, ons2}))
        With the notation in eq.~\eqref{eq:ons},
        suppose for every $\boldsymbol{p}\in\Delta_{K-1}$, 
        and for every $t\in[T]$, 
        let the decisions $\{ \boldsymbol{p}^{(t)} : t\in[T] \}$ be derived by \texttt{AAggFF-S} for $K$ clients during $T$ rounds in Algorithm~\ref{alg:aaggff-s}.
        Then, the regret defined in eq.~\eqref{eq:regret} is bounded above as follows, where $\alpha$ and $\beta$ are determined as $\alpha=4KL_\infty, \beta=\frac{1}{4L_\infty}$.
        \begin{align*}
        \begin{gathered}
            \normalfont\text{Regret}^{(T)}\left(\boldsymbol{p}^\star\right) \leq 2 L_\infty K \left( 1 + \log \left( 1 + \frac{T}{16 K} \right) \right).
        \end{gathered}
        \end{align*}
    \end{theorem}
    From the Theorem~\ref{thm:crosssilo}, we can enjoy $\mathcal{O}\left(L_\infty K \log{T}\right)$ regret upper bound,
    which is an acceptable result, considering a typical assumption in the cross-silo setting (i.e., $K<T$).
    
    For the cross-device setting, we first present the full synchronization setting,
    which requires no extra adjustment.
    \begin{theorem}
    \label{thm:crossdevice_full} (Regret Upper Bound for \texttt{AAggFF-D} with Full-Client Participation)
        With the notation in (\ref{eq:lin_ftrl}),
        suppose for every $\boldsymbol{p}\in\Delta_{K-1}$, 
        and for every $t\in[T]$, 
        let the decisions $\{ \boldsymbol{p}^{(t)} : t\in[T] \}$ be derived by \texttt{AAggFF-D} for $K$ clients with client sampling probability $C=1$ during $T$ rounds in Algorithm~\ref{alg:aaggff-d}.
        Then, the regret defined in (\ref{eq:regret}) is bounded above as follows.
        \begin{align*}
        \begin{gathered}
            \normalfont\text{Regret}^{(T)}\left(\boldsymbol{p}^\star\right) \leq 2L_\infty\sqrt{T\log{K}}.
        \end{gathered}
        \end{align*}
    \end{theorem}

    When equipped with a client sampling, the randomness from the sampling should be considered.
    Due to local losses of selected clients can only be observed, 
    \texttt{AAggFF-D} should be equipped with the unbiased estimator of a response vector 
    (from Lemma~\ref{lemma:unbiased_resp}) 
    and a corresponding linearly approximated gradient vector 
    (from Lemma~\ref{lemma:linearized_grad}).
    Since they are unbiased estimators, the expected regret is also the same.
    
    \begin{corollary}
    \label{cor:crossdevice_partial} (Regret Upper Bound for \texttt{AAggFF-D} with Partial-Client Participation)
    With the client sampling probability $C\in(0,1)$,
    the DR estimator of a partially observed response $\breve{\boldsymbol{r}}^{(t)}$
    and corresponding linearized gradient $\breve{\boldsymbol{g}}^{(t)}$ for all $t\in[T]$,
    the regret defined in (\ref{eq:regret}) is bounded above in expectation as follows.
    \begin{align*}
    \begin{gathered}
        \mathbb{E}\left[\normalfont\text{Regret}^{(T)}(\boldsymbol{p}^{\star})\right] 
        \leq
        \mathcal{O}\left(L_\infty\sqrt{T\log{K}}\right).
    \end{gathered}
    \end{align*}
    \end{corollary}

\section{Experimental Results}
\label{sec:experiments}
    We design experiments to evaluate empirical performances of our proposed framework \texttt{AAggFF}, 
    composed of sub-methods \texttt{AAggFF-S} and \texttt{AAggFF-D}.

\paragraph{Experimental Setup}
\label{subsec:exp_setup}
    \begin{table}[ht]
\centering
\caption{Description of Federated Benchmarks for Cross-Silo and Cross-Device Settings}
\label{tab:dataset_desc}
\resizebox{\columnwidth}{!}{%
\tiny
\begin{tabular}{!{}lcc|lcc@{}}
\toprule
\multicolumn{3}{l|}{\textbf{Cross-Silo}} & \multicolumn{3}{l}{\textbf{Cross-Device}} \\
\textbf{Dataset}     & $K$     & $T$     & \textbf{Dataset}    & $K$      & $T$      \\ \midrule
Berka                & 7       & 100      & CelebA              & 9,343    & 3,000    \\
MQP                  & 11      & 100     & Reddit              & 817      & 300      \\
ISIC                 & 6       & 50      & SpeechCommands      & 2,005    & 500    \\ \bottomrule
\end{tabular}%
}
\end{table}
    We conduct experiments on datasets mirroring \textit{realistic scenarios} in federated systems: 
    multiple modalities (vision, text, speech, and tabular form) and natural data partitioning. 
    We briefly summarize FL settings of each dataset in Table~\ref{tab:dataset_desc}.
    For the \textit{cross-silo} setting, 
    we used Berka tabular dataset \cite{berka},
    MQP clinical text dataset \cite{mqp}, 
    and ISIC oncological image dataset \cite{isic} (also a part of FLamby benchmark \cite{flamby}).
    For the \textit{cross-device} setting,
    we used CelebA vision dataset \cite{celeba}, 
    Reddit text dataset (both are parts of LEAF benchmark \cite{leaf})
    and SpeechCommands audio dataset \cite{speechcommands}. 
    
    Instead of manually partitioning data to simulate statistical heterogeneity, 
    we adopt natural client partitions inherent in each dataset.
    Each client dataset is split into an 80\% training set and a 20\% test set in a stratified manner where applicable.
    All experiments are run with 3 different random seeds after tuning hyperparameters.
    See Appendix~\ref{app:exp_details} for full descriptions of the experimental setup.

\paragraph{Improvement in the Client-Level Fairness}
\label{subsec:exp_improved_fairness}
    We compare our methods with existing fair FL methods including \texttt{FedAvg} \cite{fedavg}, \texttt{AFL} \cite{afl}, \texttt{q-FedAvg} \cite{qffl}, \texttt{TERM} \cite{term}, \texttt{FedMGDA} \cite{fedmgda}, and \texttt{PropFair} \cite{propfair}.
    Since \texttt{AFL} requires full synchronization of clients every round, it is only compared in the cross-silo setting.
    
    In the \textit{cross-silo} setting, we assume all $K$ clients participate in $T$ federation rounds (i.e., $C=1$), 
    and in the \textit{cross-device} setting, the client participation rate $C\in(0, 1)$ is set to ensure $5$ among $K$ clients are participating in each round.
    We evaluate each dataset using appropriate metrics for tasks as indicated under the dataset name in Table~\ref{tab:result_silo} and~\ref{tab:result_device}, where the average, the best (10\%), the worst (10\%), and Gini coefficient\footnote{The Gini coefficient is inflated by ($\times 10^2$) for readability.} of clients' performance distributions are reported with the standard deviation inside parentheses in gray color below the averaged metric. 
    
    From the results, we verify that \texttt{AAggFF} can lead to enhanced worst-case metric and Gini coefficient in both settings while retaining competitive average performance to other baselines.
    Remarkably, along with the improved worst-case performance, the small Gini coefficient indicates that performances of clients are close to each other,
    which is directly translated into the improved client-level fairness.

\paragraph{Connection to Accuracy Parity}
    \begin{table}[H]
\centering
\caption{Accuracy Parity Gap in the Cross-Silo Setting}
\label{tab:ag_silo}
\scriptsize
\resizebox{0.9\columnwidth}{!}{%
\begin{tabular}{!{}lccc!{}}
\toprule
\textbf{Dataset} &
  \textbf{Berka} &
  \textbf{MQP} &
  \textbf{ISIC} \\ \cmidrule(l){2-4} 
\textbf{Method} &
  \multicolumn{3}{c}{$\Delta\text{AG } (\downarrow)$} \\ \midrule
\begin{tabular}[c]{@{}l@{}}\texttt{FedAvg}\\ \cite{fedavg}\end{tabular} &
  \begin{tabular}[c]{@{}c@{}}50.84\\ \tiny \color[HTML]{9B9B9B} (23.98)\end{tabular} &
  \begin{tabular}[c]{@{}c@{}}35.30\\ \tiny \color[HTML]{9B9B9B} (5.39)\end{tabular} &
  \begin{tabular}[c]{@{}c@{}}22.64\\ \tiny \color[HTML]{9B9B9B} (4.50)\end{tabular} \\
\begin{tabular}[c]{@{}l@{}}\texttt{AFL}\\ \cite{afl}\end{tabular} &
  \begin{tabular}[c]{@{}c@{}}50.98\\ \tiny \color[HTML]{9B9B9B} (23.78)\end{tabular} &
  \begin{tabular}[c]{@{}c@{}}34.26\\ \tiny \color[HTML]{9B9B9B} (5.16)\end{tabular} &
  \begin{tabular}[c]{@{}c@{}}25.16\\ \tiny \color[HTML]{9B9B9B} (8.01)\end{tabular} \\
\begin{tabular}[c]{@{}l@{}}\texttt{q-FedAvg}\\ \cite{qffl}\end{tabular} &
  \begin{tabular}[c]{@{}c@{}}50.43\\ \tiny \color[HTML]{9B9B9B} (22.15)\end{tabular} &
  \begin{tabular}[c]{@{}c@{}}39.16\\ \tiny \color[HTML]{9B9B9B} (7.13)\end{tabular} &
  \begin{tabular}[c]{@{}c@{}}37.69\\ \tiny \color[HTML]{9B9B9B} (5.52)\end{tabular} \\
\begin{tabular}[c]{@{}l@{}}\texttt{TERM}\\ \cite{term}\end{tabular} &
  \begin{tabular}[c]{@{}c@{}}49.60\\ \tiny \color[HTML]{9B9B9B} (23.74)\end{tabular} &
  \begin{tabular}[c]{@{}c@{}}36.07\\ \tiny \color[HTML]{9B9B9B} (6.93)\end{tabular} &
  \begin{tabular}[c]{@{}c@{}}15.19\\ \tiny \color[HTML]{9B9B9B} (9.26)\end{tabular} \\
\begin{tabular}[c]{@{}l@{}}\texttt{FedMGDA}\\ \cite{fedmgda}\end{tabular} &
  \begin{tabular}[c]{@{}c@{}}44.46\\ \tiny \color[HTML]{9B9B9B} (17.49)\end{tabular} &
  \begin{tabular}[c]{@{}c@{}}34.74\\ \tiny \color[HTML]{9B9B9B} (1.74)\end{tabular} &
  \begin{tabular}[c]{@{}c@{}}37.69\\ \tiny \color[HTML]{9B9B9B} (5.52)\end{tabular} \\
\begin{tabular}[c]{@{}l@{}}\texttt{PropFair}\\ \cite{propfair}\end{tabular} &
  \begin{tabular}[c]{@{}c@{}}49.05\\ \tiny \color[HTML]{9B9B9B} (23.78)\end{tabular} &
  \begin{tabular}[c]{@{}c@{}}37.38\\ \tiny \color[HTML]{9B9B9B} (4.35)\end{tabular} &
  \begin{tabular}[c]{@{}c@{}}32.99\\ \tiny \color[HTML]{9B9B9B} (9.60)\end{tabular} \\
\rowcolor[HTML]{FFF5E6} 
\begin{tabular}[c]{@{}l@{}}\texttt{AAggFF-S}\\ (Proposed)\end{tabular} &
  \begin{tabular}[c]{@{}c@{}}\textbf{44.03}\\ \tiny \color[HTML]{9B9B9B} (17.55)\end{tabular} &
  \begin{tabular}[c]{@{}c@{}}\textbf{33.77}\\ \tiny \color[HTML]{9B9B9B} (3.31)\end{tabular} &
  \begin{tabular}[c]{@{}c@{}}\textbf{13.05}\\ \tiny \color[HTML]{9B9B9B} (2.23)\end{tabular} \\ \bottomrule
\end{tabular}%
}
\end{table}
    \begin{table}[H]
\centering
\caption{Accuracy Parity Gap in the Cross-Device Setting}
\label{tab:ag_device}
\resizebox{0.9\columnwidth}{!}{%
    \begin{tabular}{!{}lccc!{}}
    \toprule
    \textbf{Dataset} &
      \textbf{CelebA} &
      \textbf{Reddit} &
      \textbf{\begin{tabular}[c]{@{}c@{}}Speech\\ Commands\end{tabular}} \\ \cmidrule(l){2-4} 
    \textbf{Method} &
      \multicolumn{3}{c}{$\Delta\text{AG } (\downarrow)$} \\ \midrule
    \begin{tabular}[c]{@{}l@{}}\texttt{FedAvg}\\ \cite{fedavg}\end{tabular} &
      \begin{tabular}[c]{@{}c@{}}44.25\\ \small \color[HTML]{9B9B9B} (0.84)\end{tabular} &
      \begin{tabular}[c]{@{}c@{}}18.36\\ \small \color[HTML]{9B9B9B} (3.52)\end{tabular} &
      \begin{tabular}[c]{@{}c@{}}92.07\\ \small \color[HTML]{9B9B9B} (2.87)\end{tabular} \\
    \begin{tabular}[c]{@{}l@{}}\texttt{q-FedAvg}\\ \cite{qffl}\end{tabular} &
      \begin{tabular}[c]{@{}c@{}}44.27\\ \small \color[HTML]{9B9B9B} (0.85)\end{tabular} &
      \begin{tabular}[c]{@{}c@{}}18.43\\ \small \color[HTML]{9B9B9B} (0.09)\end{tabular} &
      \begin{tabular}[c]{@{}c@{}}88.81\\ \small \color[HTML]{9B9B9B} (0.47)\end{tabular} \\
    \begin{tabular}[c]{@{}l@{}}\texttt{TERM}\\ \cite{term}\end{tabular} &
      \begin{tabular}[c]{@{}c@{}}44.34\\ \small \color[HTML]{9B9B9B} (0.93)\end{tabular} &
      \begin{tabular}[c]{@{}c@{}}17.89\\ \small \color[HTML]{9B9B9B} (0.75)\end{tabular} &
      \begin{tabular}[c]{@{}c@{}}94.02\\ \small \color[HTML]{9B9B9B} (1.10)\end{tabular} \\
    \begin{tabular}[c]{@{}l@{}}\texttt{FedMGDA}\\ \cite{fedmgda}\end{tabular} &
      \begin{tabular}[c]{@{}c@{}}51.40\\ \small \color[HTML]{9B9B9B} (2.59)\end{tabular} &
      \begin{tabular}[c]{@{}c@{}}\textbf{16.74}\\ \small \color[HTML]{9B9B9B} (0.43)\end{tabular} &
      \begin{tabular}[c]{@{}c@{}}90.35\\ \small \color[HTML]{9B9B9B} (2.90)\end{tabular} \\
    \begin{tabular}[c]{@{}l@{}}\texttt{PropFair}\\ \cite{propfair}\end{tabular} &
      \begin{tabular}[c]{@{}c@{}}51.90\\ \small \color[HTML]{9B9B9B} (10.03)\end{tabular} &
      \begin{tabular}[c]{@{}c@{}}19.39\\ \small \color[HTML]{9B9B9B} (0.64)\end{tabular} &
      \begin{tabular}[c]{@{}c@{}}92.70\\ \small \color[HTML]{9B9B9B} (1.02)\end{tabular} \\
    \rowcolor[HTML]{FFF5E6} 
    \begin{tabular}[c]{@{}l@{}}\texttt{AAggFF-D}\\ (Proposed)\end{tabular} &
      \begin{tabular}[c]{@{}c@{}}\textbf{43.29}\\ \small \color[HTML]{9B9B9B} (0.08)\end{tabular} &
      \begin{tabular}[c]{@{}c@{}}18.07\\ \small \color[HTML]{9B9B9B} (0.70)\end{tabular} &
      \begin{tabular}[c]{@{}c@{}}\textbf{85.46}\\ \small \color[HTML]{9B9B9B} (2.58)\end{tabular} \\ \bottomrule
    \end{tabular}%
}
\end{table}
    As discussed in~\cite{qffl}, the client-level fairness can be loosely connected to existing fairness notion, the \textit{accuracy parity}~\cite{accparity}.
    It is guaranteed if the accuracies in protected groups are equal to each other.
    While the accuracy parity requires \textit{equal} performances among specific groups having protected attributes~\cite{accparity,accparity2}, 
    this is too restrictive to be directly applied to FL settings, since each client cannot always be exactly corresponded to the concept of `a group', and each client's local distribution may not be partitioned by protected attributes in the federated system.

    With a relaxation of the original concept, we adopt the notion of accuracy parity for measuring the degree of the client-level fairness in the federated system, i.e., we simply regard the group as each client.
    As a metric, we adopt the accuracy parity gap ($\Delta\text{AG}$) proposed by~\cite{apgap,apgap2}, which is simply defined as an absolute difference between the performance of the best and the worst performing groups (clients).  
    The results are in Table~\ref{tab:ag_silo} and Table~\ref{tab:ag_device}.
    It can be said that the smaller the $\Delta\text{AG}$, the more degree of the accuracy parity fairness (and therefore the client-level fairness) is achieved. 

    It should be noted that strictly achieving the accuracy parity can sometimes require sacrifice in the average performance. 
    This is aligned with the result of Reddit dataset in Table~\ref{tab:ag_device}, where \texttt{FedMGDA}~\cite{fedmgda} achieved the smallest $\Delta\text{AG}$, while its average performance is only 10.58 in Table~\ref{tab:result_device}. 
    This is far lower than our proposed method’s average performance, 12.95.
    Except this case, \texttt{AAggFF} consistently shows the smallest $\Delta\text{AG}$ than other baseline methods, 
    which is important in the perspective of striking a good balance between overall utility and the client-level fairness.

\paragraph{Plug-and-Play Boosting}
\label{subsec:pnp}
    \begin{table}[ht!]
\centering
\caption{Improved Performance of FL Algorithms after being Equipped with \texttt{AAggFF}}
\label{tab:pnp}
\resizebox{\columnwidth}{!}{%
\begin{tabular}{!{}lcccc!{}}
\toprule
\textbf{Dataset} &
  \multicolumn{2}{c}{\begin{tabular}[c]{@{}c@{}}\textbf{Heart}\\ (AUROC)\end{tabular}} &
  \multicolumn{2}{c}{\begin{tabular}[c]{@{}c@{}}\textbf{TinyImageNet}\\ (Acc. 5)\end{tabular}} \\ \cmidrule(l){2-5} 
\textbf{Method} &
  \begin{tabular}[c]{@{}c@{}}Avg.\\ ($\uparrow$)\end{tabular} &
  \begin{tabular}[c]{@{}c@{}}Worst\\ ($\uparrow$)\end{tabular} &
  \begin{tabular}[c]{@{}c@{}}Avg.\\ ($\uparrow$)\end{tabular} &
  \begin{tabular}[c]{@{}c@{}}Worst\\ 10\%($\uparrow$)\end{tabular} \\ \midrule
 &
  84.42 \color[HTML]{9B9B9B}(2.45) &
  \multicolumn{1}{c|}{65.22 \color[HTML]{9B9B9B}(9.78)} &
  85.93 \color[HTML]{9B9B9B}(0.77) &
  50.95 \color[HTML]{9B9B9B}(0.15) \\
\multirow{-2}{*}{\begin{tabular}[c]{@{}l@{}}\texttt{FedAvg}\\ \cite{fedavg}\end{tabular}} &
  \cellcolor[HTML]{FFF5E6}\textbf{85.04}\color[HTML]{9B9B9B}(2.86) &
  \multicolumn{1}{c|}{\cellcolor[HTML]{FFF5E6}\textbf{66.56} \color[HTML]{9B9B9B}(10.81)} &
  \cellcolor[HTML]{FFF5E6}\textbf{86.66} \color[HTML]{9B9B9B}(0.63) &
  \cellcolor[HTML]{FFF5E6}\textbf{51.50} \color[HTML]{9B9B9B}(2.32) \\ \cmidrule(l){2-5} 
 &
  84.48\color[HTML]{9B9B9B}(0.25) &
  \multicolumn{1}{c|}{65.44\color[HTML]{9B9B9B}(9.77)} &
  \textbf{86.49} \color[HTML]{9B9B9B}(0.72) &
  51.64 \color[HTML]{9B9B9B}(2.07) \\
\multirow{-2}{*}{\begin{tabular}[c]{@{}l@{}}\texttt{FedProx}\\ \cite{fedprox}\end{tabular}} &
  \cellcolor[HTML]{FFF5E6}\textbf{85.72}\color[HTML]{9B9B9B}(2.81) &
  \multicolumn{1}{c|}{\cellcolor[HTML]{FFF5E6}\textbf{66.67}\color[HTML]{9B9B9B}(10.71)} &
  \cellcolor[HTML]{FFF5E6}86.11 \color[HTML]{9B9B9B}(0.72) &
  \cellcolor[HTML]{FFF5E6}\textbf{52.29} \color[HTML]{9B9B9B}(2.16) \\ \cmidrule(l){2-5} 
 &
  84.34\color[HTML]{9B9B9B}(2.78) &
  \multicolumn{1}{c|}{65.44\color[HTML]{9B9B9B}(10.12)} &
  87.04 \color[HTML]{9B9B9B}(1.05) &
  53.54 \color[HTML]{9B9B9B}(2.63) \\
\multirow{-2}{*}{\begin{tabular}[c]{@{}l@{}}\texttt{FedAdam}\\ \cite{adaptivefl}\end{tabular}} &
  \cellcolor[HTML]{FFF5E6}\textbf{84.84}\color[HTML]{9B9B9B}(2.85) &
  \multicolumn{1}{c|}{\cellcolor[HTML]{FFF5E6}\textbf{67.00}\color[HTML]{9B9B9B}(10.61)} &
  \cellcolor[HTML]{FFF5E6}\textbf{87.89} \color[HTML]{9B9B9B}(0.90) &
  \cellcolor[HTML]{FFF5E6}\textbf{55.92} \color[HTML]{9B9B9B}(2.25) \\ \cmidrule(l){2-5} 
 &
  84.29\color[HTML]{9B9B9B}(2.62) &
  \multicolumn{1}{c|}{65.67\color[HTML]{9B9B9B}(10.68)} &
  86.70 \color[HTML]{9B9B9B}(1.40) &
  52.81 \color[HTML]{9B9B9B}(3.50) \\
\multirow{-2}{*}{\begin{tabular}[c]{@{}l@{}}\texttt{FedYogi}\\ \cite{adaptivefl}\end{tabular}} &
  \cellcolor[HTML]{FFF5E6}\textbf{84.86}\color[HTML]{9B9B9B}(3.01) &
  \multicolumn{1}{c|}{\cellcolor[HTML]{FFF5E6}\textbf{67.00}\color[HTML]{9B9B9B}(11.09)} &
  \cellcolor[HTML]{FFF5E6}\textbf{87.42} \color[HTML]{9B9B9B}(0.94) &
  \cellcolor[HTML]{FFF5E6}\textbf{54.76} \color[HTML]{9B9B9B}(3.11) \\ \cmidrule(l){2-5} 
 &
  84.61\color[HTML]{9B9B9B}(2.96) &
  \multicolumn{1}{c|}{65.67\color[HTML]{9B9B9B}(10.68)} &
  83.52 \color[HTML]{9B9B9B}(0.63) &
  45.09 \color[HTML]{9B9B9B}(1.79) \\
\multirow{-2}{*}{\begin{tabular}[c]{@{}l@{}}\texttt{FedAdagrad}\\ \cite{adaptivefl}\end{tabular}} &
  \cellcolor[HTML]{FFF5E6}\textbf{85.09}\color[HTML]{9B9B9B}(2.91) &
  \multicolumn{1}{c|}{\cellcolor[HTML]{FFF5E6}\textbf{66.67}\color[HTML]{9B9B9B}(10.37)} &
  \cellcolor[HTML]{FFF5E6}\textbf{84.62} \color[HTML]{9B9B9B}(0.51) &
  \cellcolor[HTML]{FFF5E6}\textbf{47.88} \color[HTML]{9B9B9B}(1.95) \\ \bottomrule
\end{tabular}%
}
\end{table}
    We additionally check if \texttt{AAggFF} can also boost other FL algorithms than \texttt{FedAvg}, 
    such as \texttt{FedAdam, FedAdagrad, FedYogi} \cite{adaptivefl} and \texttt{FedProx} \cite{fedprox}.
    Since the sequential decision making procedure required in \texttt{AAggFF} is about finding a good mixing coefficient, $\boldsymbol{p}$, this is orthogonal to the minimization of $\boldsymbol{\theta}$.
    Thus, our method can be easily integrated into existing methods with no special modification, in a plug-and-play manner.
    
    For the verification, we test with two more datasets, Heart \cite{heart} and TinyImageNet \cite{tinyimagenet}, each of which is suited for binary and multi-class classification (i.e., 200 classes in total).
    Since the Heart dataset is a part of FLamby benchmark \cite{flamby}, it has pre-defined $K=4$ clients.
    For the TinyImageNet dataset, we simulate statistical heterogeneity for $K=1,000$ clients using Dirichlet distribution with a concentration of $0.01$, following \cite{diri}.
    The results are in Table~\ref{tab:pnp}, where the upper cell represents the performance of a naive FL algorithm, 
    and the lower cell contains a performance of the FL algorithm with \texttt{AAggFF}.
    While the average performance remains comparable, the worst performance is consistently boosted in both cross-silo and cross-device settings.
    This underpins the efficacy and flexibility of \texttt{AAggFF}, which can strengthen the fairness perspective of existing FL algorithms.

\section{Limitations and Future Works}
\label{future_works}
    Our work suggests interesting future directions for better federated systems, which may also be a limitation of the current work.
    First, we can exploit side information (e.g., parameters of local updates) to not preserve all clients' mixing coefficients, 
    and filter out malicious signals for robustness. 
    For example, the former can be realized by adopting other decision making schemes such as contextual Bayesian optimization \cite{cbo},
    and the latter can be addressed by clustered FL \cite{cfl1, cfl2} for a group-wise estimation of mixing coefficients.
    Both directions are promising and may improve the practicality of federated systems. 
    Furthermore, the FTRL objective can be replaced by the Follow-The-Perturbed-Leader (FTPL) \cite{ftpl}, 
    of which random perturbation in decision making process can be directly linked to the differential privacy (DP \cite{dp}) guarantee \cite{ftpl2}, which is frequently considered for the cross-silo setting.
    Last but not least, further convergence analysis is required w.r.t. the parameter perspective along with mixing coefficients, e.g., using bi-level optimization formulation. 

\section{Conclusion}
\label{conclusion}
    For improving the degree of the client-level fairness in FL, 
    we first reveal the connection of existing fair FL methods with the OCO.
    To emphasize the sequential decision making perspective, we propose improved designs and further specialize them into two practical settings: cross-silo FL \& cross-device FL.
    Our framework not only efficiently enhances a low-performing group of clients compared to existing baselines,
    but also maintains an acceptable average performance with theoretically guaranteed behaviors.
    It should also be noted that \texttt{AAggFF} requires \textit{no extra communication} and \textit{no added local computation}, 
    which are significant constraints for serving FL-based services.
    With this scalability, our method can also improve the fairness of the performance distributions of existing FL algorithms without much modification to their original mechanism.
    By explicitly bringing the sequential decision making scheme to the front, 
    we expect our work to open up new designs to promote the practicality and scalability of FL. 

\section*{Acknowledgements}
    This research was supported by the National Research Foundation of Korea (NRF) Grant funded by the Korea Government (MSIT) (No. 2020R1C1C1011063), Basic Science Research Program through the NRF Grant funded by the Ministry of Education (No. NRF-2022R1I1A4069163) and the NRF (No. NRF-2020S1A3A2A02093277).
    Gi-Soo Kim was supported by the Institute of Information \& communications Technology Planning \& evaluation (IITP) grants funded by the Korea government (MSIT) (No. RS-2020-II201336, Artificial Intelligence Graduate School Program (UNIST); No. 2022-0-00469, Development of Core Technologies for Task-oriented Reinforcement Learning for Commercialization of Autonomous Drones).

\section*{Impact Statement}
\label{impact}
    Federated Learning (FL) poses a potential risk of skewed performance distributions toward partial groups of clients.
    We believe that our proposed method can mitigate such an unfair situation in the federated system, thereby enhancing the trustworthiness and social welfare.
    A potential concern is the situation where malicious clients deliberately inflate their local signals to cause global updates to be biased towards themselves.
    To prevent such a catastrophic situation, we aim to improve the current work by exploiting information other than local losses, as mentioned in the Limitations and Future Works section.
    The authors are not aware of any other critical ethical/social implications otherwise, but are open to discussing them if they exist.

\bibliography{aaggff}
\bibliographystyle{icml2024}

\newpage
\appendix
\onecolumn
\setcounter{table}{0}
\renewcommand{\thetable}{A\arabic{table}}
\setcounter{figure}{0}
\renewcommand{\thefigure}{A\arabic{figure}}
\section{Derivations \& Proofs}
\label{app:derivations_and_proofs}

\subsection{Derivation of Mixing Coefficients from Existing Methods}
\label{app:unification}
In this section, we provide details of the unification of existing methods in the OCO framework, introduced in Section~\ref{sec:backgrounds}.
We assume full-client participation for derivation, 
and we denote $n = \sum_{i=1}^K n_i$ as a total sample size for the brevity of notation.

Suppose any FL algorithms follow the update formula in (\ref{eq:generic_update}), 
where we define $\boldsymbol{p}^{(t+1)}$ as a \textit{mixing coefficient} vector discussed in Section~\ref{subsec:oco_lang}.
\begin{equation}
\label{eq:generic_update}
\begin{gathered}
    \boldsymbol{\theta}^{(t+1)}
    \leftarrow
    \boldsymbol{\theta}^{(t)}
    -
    \left(
    \sum_{i=1}^K
    {p}^{(t+1)}_i
    \left(
    \boldsymbol{\theta}^{(t)}
    -
    \boldsymbol{\theta}^{(t+1)}_i
    \right)
    \right),
\end{gathered}
\end{equation}
where $\boldsymbol{\theta}^{(t)}$ is a global model in a previous round $t$, 
$\boldsymbol{\theta}^{(t+1)}_i$ is a local update from $i$-th client starting from $\boldsymbol{\theta}^{(t)}$, 
and $\boldsymbol{\theta}^{(t+1)}$ is a new global model updated by averaging local updates with corresponding mixing coefficient $p^{(t+1)}_i$.

\textbf{FedAvg} \cite{fedavg} 
The update of a global model from \texttt{FedAvg} is defined as follows.
\begin{equation}
\begin{gathered}
    \boldsymbol{\theta}^{(t+1)}
    \leftarrow
    \boldsymbol{\theta}^{(t)}
    -
    \left(
    \sum_{i=1}^K
    \frac{n_i}{n}
    \left(
    \boldsymbol{\theta}^{(t)}
    -
    \boldsymbol{\theta}^{(t+1)}_i
    \right)
    \right),
\end{gathered}
\end{equation}
where $n_i$ is the sample size of client $i$.
Thus, we can regard ${p}^{(t+1)}_i \propto n_i$ in \texttt{FedAvg}.

\textbf{AFL \& q-FedAvg} \cite{afl, qffl} 
The objective of \texttt{AFL} is a minimax objective defined as follows.
\begin{equation}
\begin{gathered}
    \min_{\boldsymbol{\theta}\in\mathbb{R}^d} \max_{\boldsymbol{v}\in\Delta_{K-1}}
    \sum_{i=1}^K v_i F_i\left( \boldsymbol{\theta} \right),
\end{gathered}
\end{equation}
which is later subsumed by \texttt{q-FedAvg} as its special case for the algorithm-specific constant $q$, where $q\rightarrow0$.

The objective of \texttt{q-FedAvg} is therefore defined with a nonnegative constant $q$ as follows.
\begin{equation}
\begin{gathered}
\label{eq:q-ffl}
    \min_{\boldsymbol{\theta}\in\mathbb{R}^d}
    \sum_{i=1}^K\frac{1}{q+1} \frac{n_i}{n} F_i^{q+1} \left(\boldsymbol{\theta} \right)\\
    =
    \min_{\boldsymbol{\theta}\in\mathbb{R}^d}
    \sum_{i=1}^K \frac{n_i}{n} \tilde{F}_i\left(\boldsymbol{\theta}\right),
\end{gathered}
\end{equation}
which is reduced to \texttt{FedAvg} when $q=0$.

The update of a global model from (\ref{eq:q-ffl}) has been proposed in the form of a Newton style update 
by assuming $L$-Lipschitz continuous gradient of each local objective (i.e., \texttt{q-FedSGD}) \cite{qffl}.
\begin{equation}
\begin{gathered}
    \boldsymbol{\theta}^{(t+1)}
    =
    \boldsymbol{\theta}^{(t)}
    -
    \left(
    \sum_{j=1}^K \frac{n_j}{n} \nabla^2 \tilde{F}_j \left(\boldsymbol{\theta}^{(t)}\right)
    \right)^{-1}
    \sum_{i=1}^K \frac{n_i}{n} \nabla\tilde{F}_i \left(\boldsymbol{\theta}^{(t)} \right)\\
    \preceq
    \boldsymbol{\theta}^{(t)}
    -
    \left(
    \sum_{j=1}^K
    \frac{n_j}{n} L_{q,j}\boldsymbol{I}
    \right)^{-1}
    \sum_{i=1}^K
    \frac{n_i}{n} {F}_i^q\left(\boldsymbol{\theta}^{(t)}\right)
    \nabla{F}_i\left(\boldsymbol{\theta}^{(t)}\right),
\end{gathered}
\end{equation}
where $L_{q,i} = q F_i^{q-1}\left(\boldsymbol{\theta}^{(t)}\right) \Vert\nabla F_i\left(\boldsymbol{\theta}^{(t)}\right) \Vert^2 
+ L F_i^q\left(\boldsymbol{\theta}^{(t)}\right)$ is an upper bound of the local Lipschitz gradient of $\tilde{F}_i\left(\boldsymbol{\theta}^{(t)}\right)$ (see Lemma 3 of \cite{qffl}).
This can be extended to \texttt{q-FedAvg} by replacing $\nabla{F}_i\left(\boldsymbol{\theta}^{(t)}\right)$ into $L\left(\boldsymbol{\theta}^{(t)} - \boldsymbol{\theta}^{(t+1)}_i\right)$.
To sum up, the update formula of a global model from \texttt{q-FedAvg} (including \texttt{AFL} as a special case) is as follows.
\begin{equation}
\begin{gathered}
    \boldsymbol{\theta}^{(t+1)}
    \propto
    \boldsymbol{\theta}^{(t)}
    -
    \left(
    \sum_{i=1}^K
    { \frac{ \frac{n_i}{n} L {F}_i^q(\boldsymbol{\theta}^{(t)}) }
    { \sum_{j=1}^K \frac{n_j}{n} L_{q,j} } }
    \left(
    \boldsymbol{\theta}^{(t)}
    -
    \boldsymbol{\theta}^{(t+1)}_i\right)
    \right),
\end{gathered}
\end{equation}
which implies $p_i^{(t+1)} \propto { n_i {F}_i^q\left(\boldsymbol{\theta}^{(t)}\right) }$.

\textbf{TERM} \cite{term} 
The objective of \texttt{TERM} is dependent upon a hyperparameter, 
a \textit{tilting} constant $\lambda\in\mathbb{R}$.
\begin{equation}
\begin{gathered}
    \min_{\boldsymbol{\theta}\in\mathbb{R}^d}
    \frac{1}{\lambda} \log \left( \sum_{i=1}^K \frac{n_i}{n} \exp\left({\lambda {F}_i(\boldsymbol{\theta}^{(t)})}\right) \right)
\end{gathered}
\end{equation}

The corresponding update formula is given as follows.
\begin{equation}
\begin{gathered}
    \boldsymbol{\theta}^{(t+1)}
    =
    \boldsymbol{\theta}^{(t)}
    -
    \left(
    \sum_{i=1}^K
    \frac{
    \left(n_i / n\right) \exp\left(\lambda {F}_i\left(\boldsymbol{\theta}^{(t)}\right)\right)
    }
    {
    \sum_{j=1}^K 
    \left(n_j / n\right) \exp\left(\lambda {F}_j\left(\boldsymbol{\theta}^{(t)}\right)\right)
    }
    \left(
    \boldsymbol{\theta}^{(t)}
    -
    \boldsymbol{\theta}^{(t+1)}_i\right)
    \right)
\end{gathered}
\end{equation}

From the update formula, we can conclude that $p_i^{(t+1)} \propto n_i \exp\left(\lambda {F}_i\left(\boldsymbol{\theta}^{(t)}\right)\right)$.

\textbf{PropFair} \cite{propfair} 
The objective of \texttt{PropFair} is to maximize Nash social welfare 
by regarding a negative local loss as an achieved utility as follows.
\begin{equation}
\begin{gathered}
    \min_{\boldsymbol{\theta}\in\mathbb{R}^d} -\sum_{i=1}^K p_i\log\left(M - F_i\left(\boldsymbol{\theta}\right)\right),
\end{gathered}
\end{equation}
where $M\geq1$ is a problem-specific constant.

The corresponding update formula is given as follows.
\begin{equation}
\begin{gathered}
    \boldsymbol{\theta}^{(t+1)}
    \propto
    \boldsymbol{\theta}^{(t)}
    +
    \left(
    \sum_{i=1}^K
    \frac{n_i}{n}\nabla\log\left(M - F_i\left(\boldsymbol{\theta}^{(t)}\right)\right)
    \right)
    =
    \boldsymbol{\theta}^{(t)}
    -
    \left(
    \sum_{i=1}^K
    \frac{n_i}{n}\frac{
    \nabla F_i\left(\boldsymbol{\theta}^{(t)}\right)
    }{
    M-F_i\left(\boldsymbol{\theta}^{(t)}\right)
    }
    \right).
\end{gathered}
\end{equation}

Similar to \texttt{q-FedAvg}, by replacing the gradient $\nabla{F}_i\left(\boldsymbol{\theta}^{(t)}\right)$ 
into $\left(\boldsymbol{\theta}^{(t)} - \boldsymbol{\theta}^{(t+1)}_i\right)$, the update formula finally becomes:
\begin{equation}
\begin{gathered}
    \boldsymbol{\theta}^{(t+1)}
    \propto
    \boldsymbol{\theta}^{(t)}
    -
    \left(
    \sum_{i=1}^K
    \frac{ n_i/n }{ M-F_i\left(\boldsymbol{\theta}^{(t)}\right) }
    \left(\boldsymbol{\theta}^{(t)} - \boldsymbol{\theta}^{(t+1)}_i\right)
    \right),
\end{gathered}
\end{equation}
which implies $p_i^{(t+1)} \propto \frac{n_i}{M - F_i\left(\boldsymbol{\theta}^{(t)}\right)}$.

\subsection{Technical Lemmas}
\label{app:proofs}
In this section, we provide technical lemmas and proofs (including deferred ones in the main text) required for proving Theorem~\ref{thm:crosssilo}, Theorem~\ref{thm:crossdevice_full}, and Corollary~\ref{cor:crossdevice_partial}.

\subsubsection{Strict Convexity of Decision Loss}
\begin{lemma}
\label{lemma:convexity}
    For all $t\in[T]$, the decision loss $\ell^{(t)}$ defined in (\ref{eq:decision_loss}) satisfies following for $\gamma\in(0,1)$, 
    i.e., the decision loss is a strictly convex function of its first argument.
\begin{equation}
\begin{gathered}
    \ell^{(t)}\left(\gamma \boldsymbol{p} + (1-\gamma) \boldsymbol{q}\right)
    <
    \gamma \ell^{(t)}\left(\boldsymbol{p}\right) + (1-\gamma) \ell^{(t)}\left(\boldsymbol{q}\right),
    \forall \boldsymbol{p}, \boldsymbol{q} \in \Delta_{K-1}, \boldsymbol{p} \neq \boldsymbol{q}.
\end{gathered}
\end{equation}
\end{lemma}

\begin{proof}
From the left-hand side, we have
\begin{equation}
\begin{split}
    &\ell^{(t)}\left(\gamma \boldsymbol{p} + (1-\gamma) \boldsymbol{q}\right)\\
    =
    &-\log\left(1+ \langle \gamma \boldsymbol{p} + (1-\gamma) \boldsymbol{q}, \boldsymbol{r}^{(t)} \rangle\right) \\
    =
    &-\log\left(
    1+ \langle \gamma \boldsymbol{p}, \boldsymbol{r}^{(t)} \rangle + \langle (1-\gamma) \boldsymbol{q}, \boldsymbol{r}^{(t)} \rangle 
    \right) \\
    =
    &-\log \left( 
    \gamma (1 + \langle \boldsymbol{p}, \boldsymbol{r}^{(t)} \rangle) 
    + (1-\gamma)(1 + \langle  \boldsymbol{q}, \boldsymbol{r}^{(t)} \rangle) 
    \right).
\end{split}
\end{equation}
Since the negative of logarithm is strictly convex, the last term becomes
\begin{equation}
\begin{gathered}
    -\log \left( 
    \gamma (1 + \langle \boldsymbol{p}, \boldsymbol{r}^{(t)} \rangle) 
    + (1-\gamma)(1 + \langle  \boldsymbol{q}, \boldsymbol{r}^{(t)} \rangle) 
    \right)
    < \gamma \left( -\log (1 + \langle \boldsymbol{p}, \boldsymbol{r}^{(t)} \rangle) \right)
    + (1-\gamma) \left( -\log(1 + \langle  \boldsymbol{q}, \boldsymbol{r}^{(t)} \rangle) \right),
\end{gathered}
\end{equation}
which satisfies the definition of the strict convexity, thereby concludes the proof.
\end{proof}

\subsubsection{Lipschitz Continuity of Decision Loss (Lemma~\ref{lemma:lipschitz})}
From the definition of the Lipschitz continuity w.r.t. $\Vert\cdot\Vert$, 
we need to check if the decision loss $\ell^{(t)}$ satisfies following inequality for the constant $L_\infty$.
\begin{equation}
\begin{gathered}
    \left\vert \ell^{(t)} \left(\boldsymbol{p}\right) - \ell^{(t)} \left(\boldsymbol{q}\right)\right\vert 
    \leq L_\infty \left\Vert \boldsymbol{p} - \boldsymbol{q} \right\Vert_\infty.
\end{gathered}
\end{equation}

\begin{proof}
From Lemma~\ref{lemma:convexity}, we have the following inequality from the convexity of the decision loss.
\begin{equation}
\begin{split}
    &\left\vert \ell^{(t)} \left(\boldsymbol{p}\right) - \ell^{(t)} \left(\boldsymbol{q}\right) \right\vert
    \leq
    \left\vert \langle \nabla \ell^{(t)} \left(\boldsymbol{p}\right), \boldsymbol{p} - \boldsymbol{q} \rangle \right\vert\\
    &=
    \left\vert
    - \frac{ \langle \boldsymbol{p}-\boldsymbol{q}, \boldsymbol{r}^{(t)} \rangle }
    { 1 +\langle \boldsymbol{p}, \boldsymbol{r}^{(t)} \rangle }    
    \right\vert\\
    &=
    \frac{1}{
    1 +\langle
    \boldsymbol{p},
    \boldsymbol{r}^{(t)}
    \rangle}
    \left\vert
    \left\langle \boldsymbol{q}, \boldsymbol{r}^{(t)} \right\rangle
    - \left\langle \boldsymbol{p}, \boldsymbol{r}^{(t)} \right\rangle
    \right\vert
\end{split}
\end{equation}
Setting the denominator to be the minimum value, $\left\langle \boldsymbol{p}, \boldsymbol{r}^{(t)} \right\rangle$ is $C_1$,
we have the upper bound as follows.
\begin{equation}
\begin{split}
    &\frac{1}{
    1 +\left\langle
    \boldsymbol{p},
    \boldsymbol{r}^{(t)}
    \right\rangle}
    \left\vert
    \left\langle \boldsymbol{q}, \boldsymbol{r}^{(t)} \right\rangle
    - \left\langle \boldsymbol{p}, \boldsymbol{r}^{(t)} \right\rangle
    \right\vert\\
    &\leq
    \frac{1}{
    1 +\left\langle
    \boldsymbol{p},
    \boldsymbol{r}^{(t)}
    \right\rangle}
    \max\left({
    \left\langle \boldsymbol{q}, \boldsymbol{r}^{(t)} \right\rangle, 
    \left\langle \boldsymbol{p}, \boldsymbol{r}^{(t)} \right\rangle
    }\right)\\
    &\leq
    \frac{1}{
    1 + C_1}
    \max\left({ 
    \left\langle \boldsymbol{q}, \boldsymbol{r}^{(t)} \right\rangle,
    C_1
    }\right),
\end{split}
\end{equation}
where the first inequality is from the fact that both $\langle \boldsymbol{p}, \boldsymbol{r}^{(t)} \rangle$ and $\langle \boldsymbol{q}, \boldsymbol{r}^{(t)} \rangle$ are nonnegative,
and the second inequality is due to the minimized denominator achieving the upper bound. 
Since $\langle \boldsymbol{q}, \boldsymbol{r}^{(t)} \rangle$ can achieve its maximum as $C_2$,
we can further bound as follows.
\begin{equation}
\begin{gathered}
    \frac{1}{
    1 + C_1}
    \max({
    \langle \boldsymbol{q}, \boldsymbol{r}^{(t)} \rangle,
    C_1
    })
    \leq
    \frac{1}{
    1 + C_1}
    \max({
    C_2, 
    C_1
    })
    =
    \frac{C_2}{
    1 + C_1}
\end{gathered}
\end{equation}
Finally, using the fact that $\Vert \boldsymbol{p} - \boldsymbol{q} \Vert_\infty=\max_i \vert p_i - q_i \vert = 1$,
we can conclude the statement by setting $L_\infty = \frac{C_2}{1 + C_1}$.
\end{proof}

\subsubsection{Unbiasedness of Doubly Robust Estimator (Lemma~\ref{lemma:unbiased_resp})}
\begin{proof}
    Denote the client sampling probability $C\in[0,1]$ in time $t$ as $P(i\in S^{(t)})=C$.
    Taking expectation on the doubly robust estimator of partially observed response defined in (\ref{eq:dr_response}), we have
\begin{equation}
\begin{split}
    &\mathbb{E} \left[ \breve{r}^{(t)}_i \right]
    = 
    \mathbb{E} \left[ \bigg(1 - \frac{\mathbb{I}(i \in S^{(t)})}{C}\bigg)\mathrm{\bar{r}}^{(t)} \right] 
    + 
    \mathbb{E} \left[ \frac{\mathbb{I}(i \in S^{(t)})}{C}{r}^{(t)}_i \right] \\
    &=
    \bigg(1 - \frac{ \mathbb{E}[ \mathbb{I}(i \in S^{(t)}) ] }{C} \bigg)\mathrm{\bar{r}}^{(t)}  
    + 
    \frac{\mathbb{E} \left[ \mathbb{I}(i \in S^{(t)}) \right] }{ C }{r}^{(t)}_i\\
    &=
    \bigg(1 - \frac{ P(i\in S^{(t)}) }{C} \bigg)\mathrm{\bar{r}}^{(t)}  
    + 
    \frac{P(i\in S^{(t)}) }{ C }{r}^{(t)}_i\\
    &= {r}^{(t)}_i,
\end{split}
\end{equation}    
    where $\mathbb{I}(\cdot)$ is an indicator function.

    Note that the randomness of the doubly robust estimator comes from the random sampling of client indices $i\in S^{(t)}$ in round $t$,
    thus the expectation is with respect to $i\in S^{(t)}$.
    Thus, we can conclude that $\mathbb{E} \left[ \breve{\boldsymbol{r}}^{(t)} \right] = \boldsymbol{r}^{(t)}$.
    See also \cite{doublyrobust2, doublyrobust3}.
\end{proof}

\subsubsection{Unbiasedness of Linearly Approximated Gradient (Lemma~\ref{lemma:linearized_grad})}
\begin{proof}
    The gradient of a decision loss in terms of a response, 
    $\boldsymbol{g}\equiv\mathrm{\mathbf{h}}(\boldsymbol{r}) 
    = 
    [h_1(\boldsymbol{r}),...,h_K(\boldsymbol{r})]^\top 
    = -\frac{\boldsymbol{r}}{1 +\langle \boldsymbol{p}, \boldsymbol{r} \rangle}$ can be linearly approximated at reference 
    $\boldsymbol{r}_0$ as follows.
\begin{equation}
\label{app:eq_lin_grad}
\begin{gathered}
    \tilde{\mathrm{\mathbf{h}}}(\boldsymbol{r})
    =
    {\mathrm{\mathbf{h}}}(\boldsymbol{r}_0)
    +\mathrm{\mathbf{J}}_{\mathrm{\mathbf{h}}}(\boldsymbol{r}_0)(\boldsymbol{r}-\boldsymbol{r}_0)
\end{gathered}
\end{equation}

    The Jacobian $\mathrm{\mathbf{J}}_{\mathrm{\mathbf{h}}}(\boldsymbol{r})\in\mathbb{R}^{K\times K}$ is defined as follows.
\begin{equation}
\label{app:eq_lin_grad_jacobian}
\begin{split}
    \mathrm{\mathbf{J}}_{\mathrm{\mathbf{h}}}(\boldsymbol{r})
    &=
    \left[
    \frac{\partial\mathrm{\mathbf{h}}}{\partial r_1},
    ...,
    \frac{\partial\mathrm{\mathbf{h}}}{\partial r_K}
    \right] \\
    &=\left[
    \begin{array}{ccc}
    \frac{\partial h_1}{\partial r_1} & \cdots & \frac{\partial h_1}{\partial r_K} \\
    \vdots & \ddots & \vdots \\
    \frac{\partial h_K}{\partial r_1} & \cdots & \frac{\partial h_K}{\partial r_K}
    \end{array}
    \right]\\
    &=
    \left[
    \begin{array}{cccc}
    -\frac{1}{1 +\langle
    \boldsymbol{p},
    \boldsymbol{r}
    \rangle} + \frac{p_1r_1}{(1 +\langle
    \boldsymbol{p},
    \boldsymbol{r}
    \rangle)^2} & 
    \frac{p_2 r_1}{(1 +\langle
    \boldsymbol{p},
    \boldsymbol{r}
    \rangle)^2} & 
    \cdots &
    \frac{p_K r_1}{(1 +\langle
    \boldsymbol{p},
    \boldsymbol{r}
    \rangle)^2} \\
    \frac{p_1r_2}{(1 +\langle
    \boldsymbol{p},
    \boldsymbol{r}
    \rangle)^2} & 
    -\frac{1}{1 +\langle
    \boldsymbol{p},
    \boldsymbol{r}
    \rangle} + \frac{p_2r_2}{(1 +\langle
    \boldsymbol{p},
    \boldsymbol{r}
    \rangle)^2} & 
    \cdots &
    \frac{p_K r_2}{(1 +\langle
    \boldsymbol{p},
    \boldsymbol{r}
    \rangle)^2} \\
    \vdots & 
    \vdots & 
    \ddots & 
    \vdots \\
    \frac{p_1r_K}{(1 +\langle
    \boldsymbol{p},
    \boldsymbol{r}
    \rangle)^2} & 
    \frac{p_2r_K}{(1 +\langle
    \boldsymbol{p},
    \boldsymbol{r}
    \rangle)^2} & 
    \cdots &
    -\frac{1}{1 +\langle
    \boldsymbol{p},
    \boldsymbol{r}
    \rangle} + \frac{p_K r_K}{(1 +\langle
    \boldsymbol{p},
    \boldsymbol{r}
    \rangle)^2} 
    \end{array}
    \right] \\
    &=
    -\frac{1}{1 +\langle
    \boldsymbol{p},
    \boldsymbol{r}
    \rangle} \boldsymbol{I}_K
    +
    \frac{1}{(1 +\langle
    \boldsymbol{p},
    \boldsymbol{r}
    \rangle)^2} \boldsymbol{r}
    \boldsymbol{p}^\top
\end{split}
\end{equation}
    Plugging (\ref{app:eq_lin_grad_jacobian}) into (\ref{app:eq_lin_grad}) with respect to arbitrary reference $\boldsymbol{r}_0$,
    we have a linearized gradient of a decision loss as follows.
\begin{equation}
\begin{gathered}
\label{eq:lin_grad_formula}
    \tilde{\boldsymbol{g}}
    \triangleq
    \tilde{\mathrm{\mathbf{h}}}(\boldsymbol{r})
    =
    -\frac{\boldsymbol{r}_0}{1 +\langle
    \boldsymbol{p},
    \boldsymbol{r}_0
    \rangle}
    -\frac{\left(\boldsymbol{r} - \boldsymbol{r}_0\right)}
    {1 +\langle
    \boldsymbol{p},
    \boldsymbol{r}_0
    \rangle}
    +
    \frac{\boldsymbol{r}_0 \boldsymbol{p}^\top (\boldsymbol{r} - \boldsymbol{r}_0)}{(1 +\langle
    \boldsymbol{p},
    \boldsymbol{r}_0
    \rangle)^2}\\
    =
    -\frac{\boldsymbol{r}}{1 +\langle
    \boldsymbol{p},
    \boldsymbol{r}_0
    \rangle}
    +
    \frac{\boldsymbol{r}_0 \boldsymbol{p}^\top (\boldsymbol{r} - \boldsymbol{r}_0)}{(1 +\langle
    \boldsymbol{p},
    \boldsymbol{r}_0
    \rangle)^2}.
\end{gathered}
\end{equation}

    From the statement of Lemma~\ref{lemma:linearized_grad}, 
    plugging the doubly robust estimator of the partially observed response, $\breve{\boldsymbol{r}}$ from Lemma~\ref{lemma:unbiased_resp} into above, we have gradient estimate $\breve{\boldsymbol{g}}$ as follows.
\begin{equation}
\begin{gathered}
    \breve{\boldsymbol{g}}
    =
    \tilde{\mathrm{\mathbf{h}}}(\breve{\boldsymbol{r}})
    =
    -\frac{{\breve{\boldsymbol{r}}}}
    {1 +\langle
    \boldsymbol{p},
    \boldsymbol{r}_0
    \rangle}
    +
    \frac{\boldsymbol{r}_0 \boldsymbol{p}^\top (\breve{\boldsymbol{r}} - \boldsymbol{r}_0)}{(1 +\langle
    \boldsymbol{p},
    \boldsymbol{r}_0
    \rangle)^2}.
\end{gathered}
\end{equation}
    Taking an expectation, we have
\begin{equation}
\begin{gathered}
    \mathbb{E}\left[ \breve{\boldsymbol{g}} \right]
    =
    \mathbb{E}\left[ \tilde{\mathrm{\mathbf{h}}}(\breve{\boldsymbol{r}}) \right]
    =
    -\frac{ \mathbb{E}\left[ \breve{\boldsymbol{r}} \right] }
    {1 +\langle
    \boldsymbol{p},
    \boldsymbol{r}_0
    \rangle}
    +
    \frac{
    \boldsymbol{r}_0 \boldsymbol{p}^\top 
    ( \mathbb{E}\left[ \breve{\boldsymbol{r}} \right] - \boldsymbol{r}_0  ) 
    }
    { (1 +\langle \boldsymbol{p},  \boldsymbol{r}_0 \rangle)^2 }
    =
    -\frac{ \boldsymbol{r} }
    {1 +\langle
    \boldsymbol{p},
    \boldsymbol{r}_0
    \rangle}
    +
    \frac{
    \boldsymbol{r}_0 \boldsymbol{p}^\top 
    ( \boldsymbol{r} - \boldsymbol{r}_0 ) 
    }
    { (1 +\langle \boldsymbol{p},  \boldsymbol{r}_0 \rangle)^2 }\\
    = \tilde{\mathrm{\mathbf{h}}}({\boldsymbol{r}}) 
    = \tilde{\boldsymbol{g}}
    \approx 
    \boldsymbol{g}.
\end{gathered}
\end{equation}
\end{proof}

\subsubsection{Lipschitz Continuity of Linearly Approximated Gradient from Doubly Robust Estimator} 
\begin{lemma}
\label{lemma:lipschitz_lin_grad}
    Denote $\breve{\boldsymbol{g}}^{(t)}$ as the linearized gradient calculated from the doubly robust estimator of a response vector, $\boldsymbol{r}^{(t)}$, with reference $\boldsymbol{r}_0^{(t)} = \bar{\boldsymbol{r}} = \mathrm{\bar{r}}^{(t)}\boldsymbol{1}_K$
        where $\bar{\mathrm{r}}^{(t)}=\frac{1}{\vert S^{(t)} \vert} \sum_{i \in S^{(t)}} r_i^{(t)}$.
    When $S^{(t)}$ is a randomly selected client indices in round $t$ and $C=P\left(i \in S^{(t)}\right)$ is a client sampling probability, 
    then $\left\Vert \breve{\boldsymbol{g}}^{(t)} \right\Vert_\infty \leq \breve{L}_\infty 
    = \frac{C_2}{1+C_1} + \frac{2(C_2-C_1)}{C(1+C_1)}$ for $r_i^{(t)}\in[C_1, C_2], \forall i\in S^{(t)}$.
\end{lemma}

\begin{proof}
    Note that we intentionally omit superscript $^{(t)}$ from now on for the brevity of notation.
    The linearized gradient constructed from the doubly robust estimator of a response vector has a form as follows,
     according to (\ref{eq:lin_grad_formula}).
\begin{equation}
\begin{gathered}
    \breve{\boldsymbol{g}}
    =
    -\frac{ \breve{\boldsymbol{r}} }
    {1 +\langle
    \boldsymbol{p},
    \bar{\boldsymbol{r}}
    \rangle}
    +
    \frac{
    \bar{\boldsymbol{r}} \boldsymbol{p}^\top 
    ( \breve{\boldsymbol{r}} - \bar{\boldsymbol{r}} ) 
    }
    { (1 +\langle \boldsymbol{p},  \bar{\boldsymbol{r}} \rangle)^2 },
\end{gathered}
\end{equation}
    where we used $\bar{\boldsymbol{r}}=\bar{\mathrm{r}}\boldsymbol{1}_K$ as a reference $\boldsymbol{r}_0$,
    therefore $\left\Vert \bar{\boldsymbol{r}} \right\Vert_\infty=\bar{\mathrm{r}}\leq C_2$.
    
    Thus, we have
{\allowdisplaybreaks
\begin{align}
\label{app:lin_dr_grad}
    &\left\Vert \breve{\boldsymbol{g}} \right\Vert_\infty \nonumber\\
    &=
    \left\Vert
    -\frac{ \breve{\boldsymbol{r}} }
    {1 +\langle
    \boldsymbol{p},
    \bar{\boldsymbol{r}}
    \rangle}
    +
    \frac{
    \bar{\boldsymbol{r}} \boldsymbol{p}^\top 
    ( \breve{\boldsymbol{r}} - \bar{\boldsymbol{r}} ) 
    }
    { (1 +\langle \boldsymbol{p},  \bar{\boldsymbol{r}} \rangle)^2 }
    \right\Vert_\infty \nonumber\\
    &\leq
    \left\Vert
    -\frac{ \breve{\boldsymbol{r}} }
    {1 +\langle
    \boldsymbol{p},
    \bar{\boldsymbol{r}}
    \rangle}
    \right\Vert_\infty
    +
    \left\Vert
    \frac{
    \bar{\boldsymbol{r}} \boldsymbol{p}^\top 
    ( \breve{\boldsymbol{r}} - \bar{\boldsymbol{r}} ) 
    }
    { (1 +\langle \boldsymbol{p},  \bar{\boldsymbol{r}} \rangle)^2 }
    \right\Vert_\infty \nonumber\\
    &=
    \frac{ \left\Vert \breve{\boldsymbol{r}} \right\Vert_\infty }
    {1 +\langle
    \boldsymbol{p},
    \bar{\boldsymbol{r}}
    \rangle}
    +
    \frac{ 1 }
    { (1 +\langle \boldsymbol{p},  \bar{\boldsymbol{r}} \rangle)^2 }
    \left\Vert
    \bar{\boldsymbol{r}} \boldsymbol{p}^\top 
    ( \breve{\boldsymbol{r}} - \bar{\boldsymbol{r}} ) 
    \right\Vert_\infty \nonumber\\
    &\leq
    \frac{ \left\Vert \breve{\boldsymbol{r}} \right\Vert_\infty }
    {1 +\langle
    \boldsymbol{p},
    \bar{\boldsymbol{r}}
    \rangle}
    +
    \frac{ 1 }
    { (1 +\langle \boldsymbol{p},  \bar{\boldsymbol{r}} \rangle)^2 }
    \left\Vert
    \bar{\boldsymbol{r}} \boldsymbol{p}^\top
    \right\Vert_\infty
    \left\Vert
    \breve{\boldsymbol{r}} - \bar{\boldsymbol{r}} 
    \right\Vert_\infty \nonumber\\
    &=
    \frac{ \left\Vert \breve{\boldsymbol{r}} \right\Vert_\infty }
    {1 +\langle
    \boldsymbol{p},
    \bar{\boldsymbol{r}}
    \rangle}
    +
    \frac{ 1 }
    { (1 +\langle \boldsymbol{p},  \bar{\boldsymbol{r}} \rangle)^2 }
    \left\Vert
    \bar{\boldsymbol{r}}
    \right\Vert_\infty
    \left\Vert
    \boldsymbol{p}
    \right\Vert_1
    \left\Vert
    \breve{\boldsymbol{r}} - \bar{\boldsymbol{r}} 
    \right\Vert_\infty \nonumber\\
    &=
    \frac{ \left\Vert \breve{\boldsymbol{r}} \right\Vert_\infty }
    {1 +\langle
    \boldsymbol{p},
    \bar{\boldsymbol{r}}
    \rangle}
    +
    \frac{ 1 }
    { (1 +\langle \boldsymbol{p},  \bar{\boldsymbol{r}} \rangle)^2 }
    \left\Vert
    \bar{\boldsymbol{r}}
    \right\Vert_\infty
    \left\Vert
    \breve{\boldsymbol{r}} - \bar{\boldsymbol{r}} 
    \right\Vert_\infty \nonumber
\end{align}
}
    , where the first inequality is due to triangle inequality, 
    the second inequality is due to the property that
    $\left\Vert \boldsymbol{Ax} \right\Vert_\infty \leq  \left\Vert \boldsymbol{A} \right\Vert_\infty  \left\Vert \boldsymbol{x} \right\Vert_\infty$ 
    for a matrix $\boldsymbol{A}\in\mathbb{R}^{K \times K}$ and a vector $\boldsymbol{x} \in\mathbb{R}^{K}, 
    \boldsymbol{x}\neq\boldsymbol{0}_K$,
    the very next equality is due to $\vert \boldsymbol{x} \boldsymbol{y}^\top\vert_\infty
    =\max_i \Vert {x}_i \boldsymbol{y}^\top\Vert_1 
    =\max_i \vert {x}_i \vert \Vert \boldsymbol{y} \Vert_1
    =\Vert \boldsymbol{x} \Vert_\infty \Vert \boldsymbol{y} \Vert_1$,
    and the last equality is trivial since $\boldsymbol{p}\in\Delta_{K-1}$.

    Since 
    $\langle \boldsymbol{p}, \bar{\boldsymbol{r}} \rangle 
    = \sum_{i=1}^K \left( p_i \bar{\mathrm{r}} \right) = \bar{\mathrm{r}}$,
    this can be further bounded as follows.
\begin{equation}
\begin{split}
    &=
    \frac{ 1 }{1 + \bar{\mathrm{r}}} \left\Vert \breve{\boldsymbol{r}} \right\Vert_\infty
    +
    \frac{ \bar{\mathrm{r}} }{ \left( 1 + \bar{\mathrm{r}} \right)^2 }
    \left\Vert
    \breve{\boldsymbol{r}} - \bar{\boldsymbol{r}} 
    \right\Vert_\infty \\
    &=
    \frac{ 1 + \bar{\mathrm{r}} }{ \left( 1 + \bar{\mathrm{r}} \right)^2 }
    \left\Vert \breve{\boldsymbol{r}} \right\Vert_\infty
    +
    \frac{ \bar{\mathrm{r}} }{ \left( 1 + \bar{\mathrm{r}} \right)^2 }
    \left\Vert
    \breve{\boldsymbol{r}} - \bar{\boldsymbol{r}} 
    \right\Vert_\infty \\
    &\leq
    \frac{ 1 }{ 1 + \bar{\mathrm{r}}}
    \left(
    \left\Vert \breve{\boldsymbol{r}} \right\Vert_\infty
    +
    \left\Vert
    \breve{\boldsymbol{r}} - \bar{\boldsymbol{r}} 
    \right\Vert_\infty
    \right),
\end{split}
\end{equation}
    Since $\frac{ 1 }{1 + \bar{\mathrm{r}}}\leq\frac{1}{1+C_1}$,
    we can further upper bound as follows.
\begin{equation}
\begin{gathered}
\label{eq:lipschitz_two_terms}
    \frac{ 1 }{ 1 + \bar{\mathrm{r}}}
    \left(
    \left\Vert \breve{\boldsymbol{r}} \right\Vert_\infty
    +
    \left\Vert
    \breve{\boldsymbol{r}} - \bar{\boldsymbol{r}} 
    \right\Vert_\infty
    \right)
    \leq
    \frac{ 1 }{ 1 + C_1 }
    \left(
    \left\Vert \breve{\boldsymbol{r}} \right\Vert_\infty
    +
    \left\Vert
    \breve{\boldsymbol{r}} - \bar{\boldsymbol{r}} 
    \right\Vert_\infty
    \right)
\end{gathered}
\end{equation}

    To upper bound each term, let us look into $\breve{\boldsymbol{r}}$ first.
    By the definition in (\ref{eq:dr_response}), we have
\begin{equation}
\begin{split}
    &\breve{\boldsymbol{r}}
    =
    \begin{cases}
    \bar{\mathrm{r}} \boldsymbol{1}_K, & i\notin S^{(t)} \\
    \left( 1 - \frac{1}{C} \right) \bar{\mathrm{r}}\boldsymbol{1}_K + \frac{1}{C} \boldsymbol{r}, & i\in S^{(t)}
    \end{cases}.
\end{split}
\end{equation}
    For each case, $\left\Vert \breve{\boldsymbol{r}} \right\Vert_\infty$ becomes
\begin{equation}
\label{app:brev_inf_norm}
\begin{split}
    &\left\Vert \breve{\boldsymbol{r}} \right\Vert_\infty
    =
    \begin{cases}
    \bar{\mathrm{r}}, & i\notin S^{(t)} \\
    \sup_i \left\vert \frac{1}{C} (r_i - \bar{\mathrm{r}}) + \bar{\mathrm{r}} \right\vert, & i\in S^{(t)}.
    \end{cases}
\end{split}
\end{equation}
    For the first case, the average is smaller than its maximum, thus $\bar{\mathrm{r}} \leq C_2$.
    For the second case, it can be upper bounded as 
    $\sup_i \left\vert \frac{1}{C} (r_i - \bar{\mathrm{r}}) + \bar{\mathrm{r}} \right\vert
    \leq 
    \frac{1}{C} \sup_i \left\vert r_i - \bar{\mathrm{r}} \right\vert
    + C_2$ by the triangle inequality.
    
    From the trivial fact that the deviation from the average is always smaller than its range,
\begin{equation}
\begin{gathered}
    \frac{1}{C} \sup_i \left\vert r_i - \bar{\mathrm{r}} \right\vert
    \leq
    \frac{1}{C} (C_2 - C_1).
\end{gathered}
\end{equation}
    Combined, we have the following upper bounds.
\begin{equation}
\begin{split}
\label{eq:lipschitz_first}
    \left\Vert \breve{\boldsymbol{r}} \right\Vert_\infty
    \leq
    \begin{cases}
    C_2, & i\notin S^{(t)} \\
    \frac{C_2 - C_1}{C} + C_2, & i\in S^{(t)}
    \end{cases}
\end{split}
\end{equation}

    Similarly, for the second term inside in (\ref{eq:lipschitz_two_terms}), we have:
\begin{equation}
\label{app:breve_minus_bar_inf_norm}
\begin{split}
    &\left\Vert \breve{\boldsymbol{r}} - \bar{\boldsymbol{r}} \right\Vert_\infty
    =
    \begin{cases}
    0, & i\notin S^{(t)} \\
    \frac{1}{C} \sup_i \left\vert r_i - \bar{\mathrm{r}} \right\vert, & i\in S^{(t)}.
    \end{cases}
\end{split}
\end{equation}
    Corresponding upper bounds are:
\begin{equation}
\begin{gathered}
\label{eq:lipschitz_second}
    \left\Vert
    \breve{\boldsymbol{r}} - \bar{\boldsymbol{r}} 
    \right\Vert_\infty
    \leq
    \begin{cases}
    0, & i\notin S^{(t)} \\
    \frac{C_2 - C_1}{C}, & i\in S^{(t)}
    \end{cases}
\end{gathered}
\end{equation}
    Finally, adding (\ref{eq:lipschitz_first}) and (\ref{eq:lipschitz_second}) to have (\ref{eq:lipschitz_two_terms}), we have:
\begin{equation}
\begin{gathered}
    \left\Vert \breve{\boldsymbol{g}}^{(t)} \right\Vert_\infty
    \leq
    \begin{cases}
    \frac{C_2}{1+C_1}, & i\notin S^{(t)} \\
    \frac{C_2}{1+C_1} + \frac{2(C_2-C_1)}{C(1+C_1)}, & i\in S^{(t)}.
    \end{cases}
\end{gathered}    
\end{equation}
    Finally, it suffices to set $\breve{L}_\infty = \frac{C_2}{1+C_1} + \frac{2(C_2-C_1)}{C(1+C_1)}$ to conclude the proof.
\end{proof}

\subsubsection{Regret from a Linearized Loss}
\begin{corollary}
\label{corollary:lin_loss}
    From the convexity of a decision loss $\ell^{(t)}$ (Lemma~\ref{lemma:convexity}), 
    the regret defined in (\ref{eq:regret}) satisfies
\begin{equation}
\begin{gathered}
    \normalfont\text{Regret}^{(T)}\left(\boldsymbol{p}^{\star}\right) 
    = 
    \sum_{t=1}^T \ell^{(t)}\left(\boldsymbol{p}^{(t)}\right) 
    - 
    \sum_{t=1}^T \ell^{(t)}\left(\boldsymbol{p}^{\star}\right)
    \leq
    \sum_{t=1}^T \tilde\ell^{(t)}\left(\boldsymbol{p}^{(t)}\right) 
    - 
    \sum_{t=1}^T \tilde\ell^{(t)}\left(\boldsymbol{p}^{\star}\right),
\end{gathered}
\end{equation}
    where $\tilde\ell^{(t)}$ is a linearized loss defined as $\tilde\ell^{(t)}\left(\boldsymbol{p}\right) 
    =\left\langle \boldsymbol{p}, \boldsymbol{g}^{(t)} \right\rangle$
    and $\boldsymbol{g}^{(t)}=\nabla\ell^{(t)}\left(\boldsymbol{p}^{(t)}\right)$. 
\end{corollary}
\begin{proof}
    It is straightforward from the convexity of the decision loss.
\begin{equation}
\begin{gathered}
    \ell^{(t)}(\boldsymbol{p}^{(t)})
    -
    \ell^{(t)}(\boldsymbol{p}^{\star})
    \leq
    \left\langle
    \boldsymbol{g}^{(t)},
    \boldsymbol{p}^{(t)}
    -
    \boldsymbol{p}^{\star}
    \right\rangle.
\end{gathered}
\end{equation}
    Summing up both sides for $t\in[T]$, we proved the statement.
\end{proof}

\subsubsection{Equality for the Regret}
\begin{lemma} (Lemma 7.1 of \cite{oco2}; Lemma 5 of \cite{oco3})
\label{lemma:regret_bound_default_form}
    Let us define 
    $L^{(t)}\left(\boldsymbol{p}\right)\triangleq\sum_{\tau=1}^{t-1} \ell^{(\tau)} (\boldsymbol{p}) + R^{(t)}(\boldsymbol{p})$,
    where
    $\ell:\Delta_{K-1}\times\mathbb{R}^d\rightarrow\mathbb{R}$ is an arbitrary loss function and
    $R^{(t)}:\Delta_{K-1}\rightarrow\mathbb{R}$ is an arbitrary regularizer, non-decreasing across $t\in[T]$. 
    Assume further that $\boldsymbol{p}^{(t)} = \argmin_{\boldsymbol{p}\in\Delta_{K-1}} L^{(t)} \left( \boldsymbol{p} \right)$. 
    Then, for any $\boldsymbol{p}^\star\in\Delta_{K-1}$, we have
\begin{equation}
\begin{split}
    &\normalfont\text{Regret}^{(T)}(\boldsymbol{p}^{\star})
    =
    \sum_{t=1}^T 
    \ell^{(t)}\left(\boldsymbol{p}^{(t)}\right) 
    - 
    \sum_{t=1}^T \ell^{(t)}\left(\boldsymbol{p}^{\star}\right) \\
    &=
    R^{(T+1)}\left(\boldsymbol{p}^{\star}\right)
    -R^{(1)}\left(\boldsymbol{p}^{(1)}\right)
    +L^{(T+1)}\left(\boldsymbol{p}^{(T+1)}\right)
    -L^{(T+1)}\left({\boldsymbol{p}^{\star}}\right) \\
    &+
    \sum_{t=1}^T
    \left[
    L^{(t)}\left(\boldsymbol{p}^{(t)}\right)
    -
    L^{(t+1)}\left(\boldsymbol{p}^{(t+1)}\right)
    +
    \ell^{(t)}\left(\boldsymbol{p}^{(t)}\right) \right].
\end{split}
\end{equation}
\end{lemma}
\begin{proof}
    Since $\sum_{t=1}^T \ell^{(t)}\left(\boldsymbol{p}^{(t)}\right)$ appears in both sides, 
    we only need to show that
\begin{equation}
\begin{gathered}
    -\sum_{t=1}^T \ell^{(t)}(\boldsymbol{p}^\star)
    =
    R^{(T+1)}\left(\boldsymbol{p}^{\star}\right)
    -R^{(1)}\left(\boldsymbol{p}^{(1)}\right)
    +L^{(T+1)}\left(\boldsymbol{p}^{(T+1)}\right)
    -L^{(T+1)}\left({\boldsymbol{p}^{\star}}\right)
    +
    \sum_{t=1}^T
    \left[
    L^{(t)}\left(\boldsymbol{p}^{(t)}\right)
    -
    L^{(t+1)}\left(\boldsymbol{p}^{(t+1)}\right) \right]
\end{gathered}
\end{equation}
    First, note that
\begin{equation}
\begin{gathered}
    \ell^{(t)}\left(\boldsymbol{p}^{\star}\right) =
    \sum_{\tau=1}^t \ell^{(\tau)}\left(\boldsymbol{p}^{\star}\right)
    -
    \sum_{\tau=1}^{t-1} \ell^{(\tau)}\left(\boldsymbol{p}^{\star}\right)
    =
    \left( 
    L^{(t+1)}\left({\boldsymbol{p}^{\star}}\right)
    -
    R^{(t+1)}\left({\boldsymbol{p}^{\star}}\right)
    \right)
    -
    \left(
    L^{(t)}\left({\boldsymbol{p}^{\star}}\right)
    -
    R^{(t)}\left({\boldsymbol{p}^{\star}}\right)
    \right).  
\end{gathered}
\end{equation}
    Summing up the right-hand side of the above from $t=1,...,T$, we have 
\begin{equation}
\begin{gathered}
    \sum_{t=1}^T \ell^{(t)}\left(\boldsymbol{p}^{\star}\right)
    =
    \left( 
    L^{(T+1)}\left({\boldsymbol{p}^{\star}}\right)
    -
    R^{(T+1)}\left({\boldsymbol{p}^{\star}}\right)
    \right)
    -
    \left(
    L^{(1)}\left({\boldsymbol{p}^{\star}}\right)
    -
    R^{(1)}\left({\boldsymbol{p}^{\star}}\right)
    \right)
    =
    L^{(T+1)}\left({\boldsymbol{p}^{\star}}\right)
    -
    R^{(T+1)}\left({\boldsymbol{p}^{\star}}\right),
\end{gathered}
\end{equation}
    by telescoping summation.
    Thus,
\begin{equation}
\begin{gathered}
\label{app:first}
    -\sum_{t=1}^T \ell^{(t)}\left(\boldsymbol{p}^{\star}\right)
    =
    R^{(T+1)}(\boldsymbol{p}^{\star})
    -
    L^{(T+1)}(\boldsymbol{p}^{\star}).
\end{gathered}
\end{equation}
    Similarly, 
\begin{equation}
\begin{gathered}
    \sum_{t=1}^T 
    \left[
    L^{(t)}\left(\boldsymbol{p}^{(t)}\right)
    -
    L^{(t+1)}\left(\boldsymbol{p}^{(t+1)}\right)
    \right]
    =
    L^{(1)}\left(\boldsymbol{p}^{(1)}\right)
    -
    L^{(T+1)}\left(\boldsymbol{p}^{(T+1)}\right)
    =
    R^{(1)}\left(\boldsymbol{p}^{(1)}\right)
    -
    L^{(T+1)}\left(\boldsymbol{p}^{(T+1)}\right).
\end{gathered}
\end{equation}
    Rearranging, 
\begin{equation}
\begin{gathered}
\label{app:second}
    0
    =
    L^{(T+1)}\left(\boldsymbol{p}^{(T+1)}\right)
    -
    R^{(1)}\left(\boldsymbol{p}^{(1)}\right)
    +
    \sum_{t=1}^T \left[
    L^{(t)}\left(\boldsymbol{p}^{(t)}\right)
    -
    L^{(t+1)}\left(\boldsymbol{p}^{(t+1)}\right)
    \right]
\end{gathered}
\end{equation}
    Adding (\ref{app:first}) and (\ref{app:second}), we have
\begin{equation}
\begin{gathered}
    - 
    \sum_{t=1}^T \ell^{(t)}\left(\boldsymbol{p}^{\star}\right)
    =
    R^{(T+1)}\left(\boldsymbol{p}^{\star}\right)
    -R^{(1)}\left(\boldsymbol{p}^{(1)}\right)
    +L^{(T+1)}\left(\boldsymbol{p}^{(T+1)}\right)
    -L^{(T+1)}\left({\boldsymbol{p}^{\star}}\right)
    +
    \sum_{t=1}^T
    \left[
    L^{(t)}\left(\boldsymbol{p}^{(t)}\right)
    -
    L^{(t+1)}\left(\boldsymbol{p}^{(t+1)}\right)
    \right]
\end{gathered}
\end{equation}
    Finally, by adding $\sum_{t=1}^T \ell^{(t)}\left(\boldsymbol{p}^{(t)}\right)$ to both sides, we prove the statement.
    Note that the left hand side of the main statement does not depend on $R^{(t)}$, thus we can replace $R^{(T+1)}\left(\boldsymbol{p}^{\star}\right)$ 
    into $R^{(T)}\left(\boldsymbol{p}^{\star}\right)$. (Remark 7.3 of \cite{oco3})
\end{proof}

\subsubsection{Upper Bound to the Suboptimality Gap}
\begin{lemma} (Oracle Gap; Corollary 7.7 of \cite{oco3})
\label{lemma:cor_orabona}
    Let $f:\mathbb{R}^K \rightarrow \mathbb{R}$ be a $\mu$-strongly convex w.r.t. $\Vert\cdot\Vert$ over its domain.
    Let $\boldsymbol{x}^\star=\argmin_{\boldsymbol{x}}f(\boldsymbol{x})$.
    Then, for all $\boldsymbol{x}\in\operatorname{dom}\partial f(\boldsymbol{x})$, and $\boldsymbol{g}\in\partial f(\boldsymbol{x})$, we have:
\begin{equation}
\begin{gathered}
    f(\boldsymbol{x}) - f(\boldsymbol{x}^\star) \leq \frac{1}{2 \mu} \Vert \boldsymbol{g} \Vert_\star^2,
\end{gathered}
\end{equation}
    where $\Vert\cdot\Vert_\star$ is a dual norm of $\Vert\cdot\Vert$.
\end{lemma}

\subsubsection{Progress Bound}
\begin{lemma} 
\label{lemma:one_step_bound_proximal} (Progress Bound of FTRL with Proximal Regularizer)
    With a slight abuse of notation, 
    assume $L^{(t)}$ is closed, subdifferentiable and $\mu^{(t)}$-strongly convex w.r.t. 
    $\Vert\cdot\Vert$ in $\Delta_{K-1}$.
    First assume that $\boldsymbol{p}^{(t+1)} = \argmin_{\boldsymbol{p}\in\Delta_{K-1}}L^{(t+1)}\left(\boldsymbol{p}\right)$.
    Assume further that the regularizer satisfies
    $\boldsymbol{p}^{(t)} 
    = \argmin_{\boldsymbol{p}\in\Delta_{K-1}}
    \left(R^{(t+1)}\left(\boldsymbol{p}\right) - R^{(t)}\left(\boldsymbol{p}\right)\right)$,
    and $\boldsymbol{g}^{(t)}\in\partial L^{(t+1)}(\boldsymbol{p}^{(t)})$.
    Then, we have the following inequality:
\begin{equation}
\begin{gathered}
    L^{(t)}\left(\boldsymbol{p}^{(t)}\right) 
    -
    L^{(t+1)}\left(\boldsymbol{p}^{(t+1)}\right)
    +
    \ell^{(t)}\left(\boldsymbol{p}^{(t)}\right)
    \leq
    \frac{ \left\Vert \boldsymbol{g}^{(t)} \right\Vert_\star^2 }
    { 2\mu^{(t+1)} }
    +
    \left(
    R^{(t)}\left(\boldsymbol{p}^{(t)}\right)
    -
    R^{(t+1)}\left(\boldsymbol{p}^{(t)}\right)
    \right),
\end{gathered}
\end{equation}
    where $\Vert\cdot\Vert_\star$ is a dual norm of $\Vert\cdot\Vert$.
\end{lemma}

\begin{proof}
\begin{equation}
\begin{split}
    &L^{(t)}\left(\boldsymbol{p}^{(t)}\right)
    -
    L^{(t+1)}\left(\boldsymbol{p}^{(t+1)}\right)
    +
    \ell^{(t)}\left(\boldsymbol{p}^{(t)}\right)\\
    &=
    \left(
    L^{(t)}\left(\boldsymbol{p}^{(t)}\right)
    +
    \ell^{(t)}\left(\boldsymbol{p}^{(t)}\right)
    +
    R^{(t+1)}\left(\boldsymbol{p}^{(t)}\right)
    -
    R^{(t)}\left(\boldsymbol{p}^{(t)}\right)
    \right)
    -
    L^{(t+1)}\left(\boldsymbol{p}^{(t+1)}\right)
    -
    R^{(t+1)}\left(\boldsymbol{p}^{(t)}\right)
    +
    R^{(t)}\left(\boldsymbol{p}^{(t)}\right) \\
    &=
    L^{(t+1)}\left(\boldsymbol{p}^{(t)}\right)
    -
    L^{(t+1)}\left(\boldsymbol{p}^{(t+1)}\right)
    -
    R^{(t+1)}\left(\boldsymbol{p}^{(t)}\right)
    +
    R^{(t)}\left(\boldsymbol{p}^{(t)}\right) \\
    &\leq
    \frac{ \left\Vert \boldsymbol{g}^{(t)} \right\Vert_\star^2 }
    { 2\mu^{(t+1)} }
    -
    R^{(t+1)}\left(\boldsymbol{p}^{(t)}\right)
    +
    R^{(t)}\left(\boldsymbol{p}^{(t)}\right),
\end{split}
\end{equation}
    where the first inequality is due the assumption that $\boldsymbol{p}^{(t+1)} = \argmin_{\boldsymbol{p}\in\Delta_{K-1}}L^{(t+1)}\left(\boldsymbol{p}\right)$, 
    $\boldsymbol{g}^{(t)}\in\partial L^{(t+1)}(\boldsymbol{p}^{(t)})$, and the result from Lemma~\ref{lemma:cor_orabona}.
    See also Lemma 7.23 of \cite{oco3}.
\end{proof}

\begin{lemma} 
\label{lemma:one_step_bound} (Progress Bound of FTRL with Non-Decreasing Regularizer)
    With a slight abuse of notation, 
    assume $L^{(t)}$ to be closed and sub-differentiable in $\Delta_K$, 
    and $\left(L^{(t)} + \ell^{(t)}\right)$ to be closed, differentiable and $\nu^{(t)}$-strongly convex w.r.t. 
    $\Vert\cdot\Vert_1$ in $\Delta_{K-1}$. 
    Further define with an abuse of notation again that $\boldsymbol{g}^{(t)}=\nabla \ell^{(t)}(\boldsymbol{p}^{(t)})\in\partial\left(L^{(t)}+\ell^{(t)}\right)\left(\boldsymbol{p}^{(t)}\right)$,
    and define further that $\boldsymbol{p}^{(t)} = \argmin_{\boldsymbol{p}\in\Delta_{K-1}} L^{(t)}\left(\boldsymbol{p}\right)$. 
    Then, we have the following inequality:
\begin{equation}
\begin{gathered}
    L^{(t)}\left(\boldsymbol{p}^{(t)}\right) 
    -
    L^{(t+1)}\left(\boldsymbol{p}^{(t+1)}\right)
    +
    \ell^{(t)}\left(\boldsymbol{p}^{(t)}\right)
    \leq
    \frac{ \left\Vert \boldsymbol{g}^{(t)} \right\Vert_\infty^2 }
    { 2\nu^{(t)} }
    +
    \left(
    R^{(t)}\left(\boldsymbol{p}^{(t+1)}\right)
    -
    R^{(t+1)}\left(\boldsymbol{p}^{(t+1)}\right)
    \right).
\end{gathered}
\end{equation}
\end{lemma}
\begin{proof}
Let us first assume that $\boldsymbol{p}^{(t)}_* = \argmin_{\boldsymbol{p}\in\Delta_{K-1}}\left(L^{(t)}\left(\boldsymbol{p}\right) + \ell^{(t)}\left(\boldsymbol{p}\right)\right)$.
Observe that 
\begin{equation}
\begin{gathered}
    L^{(t+1)}\left(\boldsymbol{p}^{(t+1)}\right)
    =
    L^{(t)}\left(\boldsymbol{p}^{(t+1)}\right)
    +
    \ell^{(t)}\left(\boldsymbol{p}^{(t+1)}\right)
    -
    R^{(t)}\left(\boldsymbol{p}^{(t+1)}\right)
    +
    R^{(t+1)}\left(\boldsymbol{p}^{(t+1)}\right),
\end{gathered}
\end{equation}
we have
\begin{equation}
\begin{split}
    &L^{(t)}\left(\boldsymbol{p}^{(t)}\right)
    -
    L^{(t+1)}\left(\boldsymbol{p}^{(t+1)}\right)
    +
    \ell^{(t)}\left(\boldsymbol{p}^{(t)}\right)\\
    &=
    L^{(t)}\left(\boldsymbol{p}^{(t)}\right)
    -
    \left(
    L^{(t)}\left(\boldsymbol{p}^{(t+1)}\right)
    +
    \ell^{(t)}\left(\boldsymbol{p}^{(t+1)}\right)
    -
    R^{(t)}\left(\boldsymbol{p}^{(t+1)}\right)
    +
    R^{(t+1)}\left(\boldsymbol{p}^{(t+1)}\right)
    \right)
    +
    \ell^{(t)}\left(\boldsymbol{p}^{(t)}\right)\\
    &=
    \left(
    L^{(t)}\left(\boldsymbol{p}^{(t)}\right)
    +
    \ell^{(t)}\left(\boldsymbol{p}^{(t)}\right)
    \right)
    -
    \left(
    L^{(t)}\left(\boldsymbol{p}^{(t+1)}\right)
    +
    \ell^{(t)}\left(\boldsymbol{p}^{(t+1)}\right)
    \right)
    +
    \left(
    R^{(t)}\left(\boldsymbol{p}^{(t+1)}\right)
    -
    R^{(t+1)}\left(\boldsymbol{p}^{(t+1)}\right)
    \right)\\
    &\leq
    \left(
    L^{(t)}\left(\boldsymbol{p}^{(t)}\right)
    +
    \ell^{(t)}\left(\boldsymbol{p}^{(t)}\right)
    \right)
    -
    \left(
     L^{(t)}\left(\boldsymbol{p}^{(t)}_*\right)
    +
    \ell^{(t)}\left(\boldsymbol{p}^{(t)}_*\right)
    \right)
    +
    \left(
    R^{(t)}\left(\boldsymbol{p}^{(t+1)}\right)
    -
    R^{(t+1)}\left(\boldsymbol{p}^{(t+1)}\right)
    \right) \\
    &\leq
    \frac{ \left\Vert \boldsymbol{g}^{(t)} \right\Vert_\infty^2 }{ 2\nu^{(t)} }
    +
    \left(
    R^{(t)}\left(\boldsymbol{p}^{(t+1)}\right)
    -
    R^{(t+1)}\left(\boldsymbol{p}^{(t+1)}\right)
    \right),
\end{split}
\end{equation}
    where the first inequality is due the assumption that $\boldsymbol{p}^{(t)}_* = \argmin_{\boldsymbol{p}\in\Delta_{K-1}}\left(L^{(t)}+\ell^{(t)}\right)\left(\boldsymbol{p}\right)$,
    $\boldsymbol{g}^{(t)}\in\partial\left(L^{(t)}+\ell^{(t)}\right)$.
    Lastly, the second inequality is the direct result from Lemma~\ref{lemma:cor_orabona}.
    See also Lemma 7.8 of \cite{oco3}.
\end{proof}

\subsubsection{Exp-concavity}
\begin{definition}
    A function $f:X\rightarrow\mathbb{R}$ is $\gamma$-exp-concave if $\exp\left(-\gamma f\left(\boldsymbol{x}\right)\right)$ is concave for $\boldsymbol{x}\in X$.
\end{definition}

\begin{remark}
    The decision loss defined in (\ref{eq:decision_loss}) is $1$-exp-concave.
\end{remark}

\begin{lemma}
\label{lemma:exp_concave}
    For an $\gamma$-exp-concave function $f:X\rightarrow\mathbb{R}$, let the domain $X$ is bounded, 
    and choose $\delta\leq\frac{\gamma}{2}$ such that $\left\vert 
    \delta 
    \left\langle \nabla f(\boldsymbol{x}), \boldsymbol{y} - \boldsymbol{x}\right\rangle  
    \right\vert \leq \frac{1}{2}$ and for all $\boldsymbol{x}, \boldsymbol{y}\in X$, the following statement holds.
\begin{equation}
\begin{gathered}
    f(\boldsymbol{y}) \geq f(\boldsymbol{x}) 
    + \left\langle \nabla f(\boldsymbol{x}), \boldsymbol{y} - \boldsymbol{x}\right\rangle
    + \frac{\delta}{2} \left( \left\langle \nabla f(\boldsymbol{y}), \boldsymbol{y} - \boldsymbol{x} \right\rangle \right)^2
\end{gathered}
\end{equation}
\end{lemma}

\begin{proof}
    For all such that $\delta\leq\frac{\gamma}{2}$, a function $g(\boldsymbol{x})=\exp\left(-2\delta f(\boldsymbol{x})\right)$ is concave.
    From the concavity of $g$, we have:
\begin{equation}
\begin{gathered}
    g(\boldsymbol{x}) \leq g(\boldsymbol{y}) + \left\langle \nabla g(\boldsymbol{y}), \boldsymbol{x} - \boldsymbol{y}\right\rangle.
\end{gathered}
\end{equation}
    By taking logarithm on both sides, we have:
\begin{equation}
\begin{gathered}
    f(\boldsymbol{x}) \geq f(\boldsymbol{y}) - \frac{1}{2\delta} 
    \log \left(
    1 - 2\delta \left\langle \nabla f(\boldsymbol{y}), \boldsymbol{x} - \boldsymbol{y} \right\rangle
    \right).
\end{gathered}
\end{equation}
    From the assumption, we have 
    $\left\vert 
    \delta 
    \left\langle \nabla f(\boldsymbol{x}), \boldsymbol{y} - \boldsymbol{x}\right\rangle  
    \right\vert \leq \frac{1}{2}$,
    and using the fact that $\log(1+x)\leq x - \frac{x^2}{4}$ for $\vert x \vert \leq 1$,
    we can conclude the proof.
\end{proof}

\begin{remark} (Remark 7.27 from \cite{oco3})
For a positive definite matrix $\boldsymbol{B}\in\mathbb{R}^{K \times K}$, $\Vert \boldsymbol{p} \Vert_{\boldsymbol{B}}$ is a norm induced by $\boldsymbol{B}$, defined as 
$\Vert \boldsymbol{p} \Vert_{\boldsymbol{B}} \triangleq \sqrt{\boldsymbol{p}^\top \boldsymbol{B} \boldsymbol{p}}$
for $\boldsymbol{p}\in\mathbb{R}^K$.
A function $f(\boldsymbol{p})=\frac{1}{2}\boldsymbol{p}^\top \boldsymbol{B} \boldsymbol{p}$ is therefore $1$-strongly convex w.r.t. $\Vert \cdot \Vert_{\boldsymbol{B}}$. 
Note that the dual norm of $\Vert \cdot \Vert_{\boldsymbol{B}}$ is $\Vert \cdot \Vert_{\boldsymbol{B}^{-1}}$.
\end{remark}

\subsection{Regret Bound of \texttt{AAggFF-S}: Proof of Theorem~\ref{thm:crosssilo}}
\begin{remark}
    The regularizer of \texttt{AAggFF-S} (i.e., ONS) is proximal since it has a quadratic form.
\end{remark}
\begin{proof}
    Since Lemma~\ref{lemma:regret_bound_default_form} holds for arbitrary loss function,
    let us set ${L}^{(t)}\left(\boldsymbol{p}\right)\triangleq\sum_{\tau=1}^{t-1} \tilde\ell^{(\tau)} (\boldsymbol{p}) + \frac{\alpha}{2} \Vert \boldsymbol{p} \Vert_2^2 
    + \frac{\beta}{2} \sum_{\tau=1}^{t-1} 
    \left(\left\langle \boldsymbol{g}^{(\tau)}, \boldsymbol{p} - \boldsymbol{p}^{(\tau)} \right\rangle\right)^2$ 
    as in (\ref{eq:ons}) with a slight abuse of notation.
    Note that we set $R^{(t)}\left(\boldsymbol{p}\right) = \frac{\alpha}{2} \Vert \boldsymbol{p} \Vert_2^2 + 
    \frac{\beta}{2} \sum_{\tau=1}^{t-1} 
    \left(\left\langle \boldsymbol{g}^{(\tau)}, \boldsymbol{p} - \boldsymbol{p}^{(\tau)} \right\rangle\right)^2$, 
    which is often called a proximal regularizer.
    
    From the regret, we have:
\begin{equation}
\begin{split}
    &\normalfont\text{Regret}^{(T)}\left(\boldsymbol{p}^{\star}\right) \\
    &= 
    \sum_{t=1}^T \ell^{(t)}\left(\boldsymbol{p}^{(t)}\right) 
    - 
    \sum_{t=1}^T \ell^{(t)}\left(\boldsymbol{p}^{\star}\right) \\
    &\leq
    \sum_{t=1}^T \tilde\ell^{(t)}\left(\boldsymbol{p}^{(t)}\right) 
    - 
    \sum_{t=1}^T \tilde\ell^{(t)}\left(\boldsymbol{p}^{\star}\right) \\
    &=
    \underbrace{
    R^{(T+1)}\left(\boldsymbol{p}^{\star}\right)
    -R^{(1)}\left(\boldsymbol{p}^{(1)}\right)}_\text{(i)}
    +
    \underbrace{
    {L}^{(T+1)}\left(\boldsymbol{p}^{(T+1)}\right)
    -{L}^{(T+1)}\left({\boldsymbol{p}^{\star}}\right)}_\text{(ii)} \\
    &+
    \underbrace{
    \sum_{t=1}^T
    \left[
    {L}^{(t)}\left(\boldsymbol{p}^{(t)}\right)
    -
    {L}^{(t+1)}\left(\boldsymbol{p}^{(t+1)}\right)
    +
    \tilde\ell^{(t)}\left(\boldsymbol{p}^{(t)}\right) \right]}_\text{(iii)}
\end{split}
\end{equation}

    Let us first denote $\boldsymbol{A}^{(t)}\triangleq \boldsymbol{g}^{(t)} {\boldsymbol{g}^{(t)}}^\top$ 
    for $\boldsymbol{g}^{(t)}=\nabla \ell^{(t)}(\boldsymbol{p}^{(t)})$ as in the main text,
    and further denote that $\boldsymbol{B}^{(t)}\triangleq \alpha \boldsymbol{I}_K + \beta\sum_{\tau=1}^t \boldsymbol{A}^{(\tau)}$.
    Then, we can rewrite the regularizer of \texttt{AAggFF-D} as follows.
\begin{equation}
\begin{gathered}
\label{eq:rewriting_of_reg}
    R^{(t)}\left(\boldsymbol{p}\right)
    = \frac{\alpha}{2} \Vert \boldsymbol{p} \Vert_2^2 + 
    \frac{\beta}{2} \sum_{\tau=1}^{t-1} \left(\left\langle \boldsymbol{g}^{(\tau)}, \boldsymbol{p} - \boldsymbol{p}^{(\tau)}\right\rangle\right)^2
    = \frac{\alpha}{2} \Vert \boldsymbol{p} \Vert_2^2 + 
    \frac{\beta}{2} \sum_{\tau=1}^{t-1} 
    \left\Vert \boldsymbol{p}^{(\tau)} - \boldsymbol{p} \right\Vert^2_{\boldsymbol{A}^{(\tau)}}
\end{gathered}
\end{equation}
    That is, $R^{(t)}\left(\boldsymbol{p}\right)$, as well as $L^{(t)}\left(\boldsymbol{p}\right)$ is $\beta$-strongly convex function w.r.t. $\left\Vert\cdot\right\Vert_{\boldsymbol{B}^{(t-1)}}, t\in[T]$.

    For (i), since the regularizer $R^{(t)}\left(\boldsymbol{p}\right)$ is nonnegative for all $t\in[T]$,
    it can be upper bounded as $R^{(T+1)}\left(\boldsymbol{p}^{\star}\right)$.
    Using (\ref{eq:rewriting_of_reg}), we have:
\begin{equation}
\begin{gathered}
    R^{(T+1)}\left(\boldsymbol{p}^{\star}\right)
    =
    \frac{\alpha}{2} \Vert \boldsymbol{p} \Vert_2^2 + 
    \frac{\beta}{2} \sum_{t=1}^{T} 
    \left\Vert \boldsymbol{p}^{(t)} - \boldsymbol{p}^\star \right\Vert^2_{\boldsymbol{A}^{(t)}}.
\end{gathered}
\end{equation}

    For (ii), we use the assumption in Lemma~\ref{lemma:one_step_bound}, 
    where $\boldsymbol{p}^{(t)} = \argmin_{\boldsymbol{p}\in\Delta_{K-1}} L^{(t)}\left(\boldsymbol{p}\right)$.
    From the assumption, since $\boldsymbol{p}^{(T+1)}$ is the minimizer of $L^{(T+1)}$, (ii) becomes negative.
    Thus, we can exclude it in further upper bounding.
    
    For (iii), we can directly use the result of Lemma~\ref{lemma:one_step_bound_proximal}.
{\allowdisplaybreaks
\begin{align}
    &\sum_{t=1}^T
    \left[
    {L}^{(t)}\left(\boldsymbol{p}^{(t)}\right)
    -
    {L}^{(t+1)}\left(\boldsymbol{p}^{(t+1)}\right)
    +
    \tilde\ell^{(t)}\left(\boldsymbol{p}^{(t)}\right) \right] \\
    &\leq
    \frac{1}{2\beta} 
    \sum_{t=1}^T \left\Vert \boldsymbol{g}^{(t)} \right\Vert_{{\boldsymbol{B}^{(t)}}^{-1}}^2
    +
    \sum_{t=1}^T \left(
    R^{(t)}\left(\boldsymbol{p}^{(t)}\right)
    -
    R^{(t+1)}\left(\boldsymbol{p}^{(t)}\right)
    \right) \\
    &=
    \frac{1}{2\beta} 
    \sum_{t=1}^T \left\Vert \boldsymbol{g}^{(t)} \right\Vert_{{\boldsymbol{B}^{(t)}}^{-1}}^2.
\end{align}
}
    Combining all, we have regret upper bound as follows.
\begin{equation}
\begin{split}
    &\normalfont\text{Regret}^{(T)}\left(\boldsymbol{p}^{\star}\right)
    = 
    \sum_{t=1}^T \ell^{(t)}\left(\boldsymbol{p}^{(t)}\right) 
    - 
    \sum_{t=1}^T \ell^{(t)}\left(\boldsymbol{p}^{\star}\right) \\
    &\leq
    \sum_{t=1}^T \tilde\ell^{(t)}\left(\boldsymbol{p}^{(t)}\right) 
    - 
    \sum_{t=1}^T \tilde\ell^{(t)}\left(\boldsymbol{p}^{\star}\right) \\
    &\leq
    \frac{\alpha}{2} \Vert \boldsymbol{p} \Vert_2^2 + 
    \frac{\beta}{2} \sum_{t=1}^{T} 
    \left\Vert \boldsymbol{p}^{(t)} - \boldsymbol{p}^\star \right\Vert^2_{\boldsymbol{A}^{(t)}}
    +
    \frac{1}{2\beta} 
    \sum_{t=1}^T \left\Vert \boldsymbol{g}^{(t)} \right\Vert_{{\boldsymbol{B}^{(t)}}^{-1}}^2.
\end{split}
\end{equation}
    Lastly, from the result of Lemma~\ref{lemma:exp_concave}, we have 
\begin{equation}
\begin{split}
\label{ons_immediate_result}
    &\normalfont\text{Regret}^{(T)}\left(\boldsymbol{p}^{\star}\right)
    = 
    \sum_{t=1}^T \ell^{(t)}\left(\boldsymbol{p}^{(t)}\right) 
    - 
    \sum_{t=1}^T \ell^{(t)}\left(\boldsymbol{p}^{\star}\right) \\
    &\leq
    \sum_{t=1}^T \tilde\ell^{(t)}\left(\boldsymbol{p}^{(t)}\right) 
    - 
    \sum_{t=1}^T \tilde\ell^{(t)}\left(\boldsymbol{p}^{\star}\right) 
    -
    \frac{\beta}{2} \sum_{t=1}^{T} 
    \left\Vert \boldsymbol{p}^{(t)} - \boldsymbol{p}^\star \right\Vert^2_{\boldsymbol{A}^{(t)}} \\
    &\leq
    \frac{\alpha}{2} \Vert \boldsymbol{p} \Vert_2^2 
    +
    \frac{1}{2\beta} 
    \sum_{t=1}^T \left\Vert \boldsymbol{g}^{(t)} \right\Vert_{{\boldsymbol{B}^{(t)}}^{-1}}^2,
\end{split}
\end{equation}
    where we need $\left\vert\beta\left\langle \boldsymbol{g}^{(t)}, \boldsymbol{p} - \boldsymbol{p}^{(t)}\right\rangle\right\vert\leq\frac{1}{2}$ due to the assumption.
    
    To meet the assumption, we have
\begin{equation}
\begin{gathered}
    \left\vert\beta\left\langle \boldsymbol{g}^{(t)}, \boldsymbol{p} - \boldsymbol{p}^{(t)}\right\rangle\right\vert
    \leq
    \beta \Vert \boldsymbol{p} - \boldsymbol{p}^{(t)} \Vert_1 \Vert \boldsymbol{g}^{(t)} \Vert_\infty
    \leq
    2 \beta L_\infty, 
\end{gathered}
\end{equation}
    where the first inequality is due to H\"{o}lder's inequality and the second inequality is due to Lemma~\ref{lemma:lipschitz} and the fact that a diameter of the simplex is 2.
    Thus, we can set $\beta=\frac{1}{4 L_\infty}$ to satisfy the assumption.

    For the first term, 
    using the fact that $\Vert \cdot \Vert_2 \leq \Vert \cdot \Vert_1$, we have
\begin{equation}
\begin{gathered}
    \frac{\alpha}{2} \Vert \boldsymbol{p} \Vert_2^2 
    \leq
    \frac{\alpha}{2} \Vert \boldsymbol{p} \Vert_1^2 
    \leq
    \frac{\alpha}{2},
\end{gathered}
\end{equation}
    where the last equality is due to $\boldsymbol{p}\in\Delta_{K-1}$.
    
    For the second term, 
    denote $\lambda_1, ..., \lambda_K$ as the eigenvalues of $\boldsymbol{B}^{(T)} - \alpha \boldsymbol{I}_K$, 
    then we have:
\begin{equation}
\begin{gathered}
    \sum_{t=1}^T \left\Vert \boldsymbol{g}^{(t)} \right\Vert_{{\boldsymbol{B}^{(t)}}^{-1}}^2
    \leq
    \sum_{i=1}^K \log \left( 1 + \frac{\lambda_i}{\alpha} \right),
\end{gathered}
\end{equation}   
    which is the direct result of Lemma 11.11 and Theorem 11.7 of \cite{logdetineq}.
    This can be further bounded by AM-GM inequality as follows.
\begin{equation}
\begin{gathered}
    \sum_{i=1}^K \log \left( 1 + \frac{\lambda_i}{\alpha} \right)
    \leq
    K \log \left( 1 + \frac{1}{K\alpha}\sum_{i=1}^K{\lambda_i} \right)
\end{gathered}
\end{equation}    
    Since we have
\begin{equation}
\begin{gathered}
    \sum_{i=1}^K \lambda_i 
    = \operatorname{trace}\left( \boldsymbol{B}^{(T)} - \alpha \boldsymbol{I}_K \right)
    = \operatorname{trace}\left( \beta\sum_{t=1}^T \boldsymbol{g}^{(t)} {\boldsymbol{g}^{(t)}}^\top \right)
    =
    \beta\sum_{t=1}^T \Vert \boldsymbol{g}^{(t)} \Vert_2^2 \\
    \leq 
    \beta K T \Vert \boldsymbol{g}^{(t)} \Vert_\infty^2 
    \leq 
    \beta K T L_\infty^2
    = \frac{K T L_\infty}{4},
\end{gathered}
\end{equation}
    where $L_\infty$ is a Lipschitz constant w.r.t. $\Vert \cdot \Vert_\infty$ from Lemma~\ref{lemma:lipschitz},
    thus the inequality is due to $\Vert \boldsymbol{p} \Vert_2 \leq \sqrt{K} \Vert \boldsymbol{p} \Vert_\infty, \forall \boldsymbol{p}\in\Delta_{K-1}$.
    
    Followingly, we can upper-bound the second term as
\begin{equation}
\begin{gathered}
    \sum_{t=1}^T \left\Vert \boldsymbol{g}^{(t)} \right\Vert_{{\boldsymbol{B}^{(t)}}^{-1}}^2
    \leq
    K \log \left( 1 + \frac{TL_\infty}{4\alpha} \right).
\end{gathered}
\end{equation}        

    Putting them all together, we have:
\begin{equation}
\begin{split}
    &\normalfont\text{Regret}^{(T)}\left(\boldsymbol{p}^{\star}\right)
    = 
    \sum_{t=1}^T \ell^{(t)}\left(\boldsymbol{p}^{(t)}\right) 
    - 
    \sum_{t=1}^T \ell^{(t)}\left(\boldsymbol{p}^{\star}\right) \\
    &\leq
    \sum_{t=1}^T \tilde\ell^{(t)}\left(\boldsymbol{p}^{(t)}\right) 
    - 
    \sum_{t=1}^T \tilde\ell^{(t)}\left(\boldsymbol{p}^{\star}\right) 
    -
    \frac{1}{2} \sum_{t=1}^{T} 
    \left\Vert \boldsymbol{p}^{(t)} - \boldsymbol{p}^\star \right\Vert^2_{\boldsymbol{A}^{(t)}} \\
    &\leq
    \frac{\alpha}{2} \Vert \boldsymbol{p} \Vert_2^2 
    +
    \frac{1}{2\beta} 
    \sum_{t=1}^T \left\Vert \boldsymbol{g}^{(t)} \right\Vert_{{\boldsymbol{B}^{(t)}}^{-1}}^2 \\
    &
    \leq
    \frac{\alpha}{2} + \frac{K}{2\beta} \log \left( 1 + \frac{TL_\infty}{4\alpha} \right)\\
    &=
    \frac{\alpha}{2} + 2 L_\infty K \log \left( 1 + \frac{TL_\infty}{4\alpha} \right).
\end{split}
\end{equation}     
    If we further set $\alpha = 4 L_\infty K$,
    we finally have    
\begin{equation}
\begin{gathered}
    \normalfont\text{Regret}^{(T)}\left(\boldsymbol{p}^{\star}\right)
    \leq
    2 L_\infty K \left( 
    1 + \log \left( 1 + \frac{T}{16K} \right) 
    \right).
\end{gathered}
\end{equation}
\end{proof}

\subsection{Regret Bound of \texttt{AAggFF-D} with Full Client Participation: Proof of Theorem~\ref{thm:crossdevice_full}}
\label{sec:device_proof}
\begin{proof}
    Again, since Lemma~\ref{lemma:regret_bound_default_form} holds for arbitrary loss function,
    let us set ${L}^{(t)}\left(\boldsymbol{p}\right)\triangleq\sum_{\tau=1}^{t-1} \tilde\ell^{(\tau)} (\boldsymbol{p}) + \eta^{(t+1)}\sum_{i=1}^K p_i \log p_i$ with a slight abuse of notation.
    Note that we set $R^{(t)}\left(\boldsymbol{p}\right) = \eta^{(t)}\sum_{i=1}^K p_i \log p_i$ is a negative entropy regularizer with non-decreasing time-varying step size $\eta^{(t)}$,
    thus ${L}^{(t)}\left(\boldsymbol{p}\right)$ is $\eta^{(t)}$-strongly convex w.r.t. $\Vert\cdot\Vert_1$. (Proposition 5.1 from \cite{bregmanmd})
    Then, we have an upper bound of the regret of \texttt{AAggFF-D} (with full-client participation setting) as follows.
{\allowdisplaybreaks
\begin{align}
    &\normalfont\text{Regret}^{(T)}\left(\boldsymbol{p}^{\star}\right) \nonumber \\
    &= 
    \sum_{t=1}^T \ell^{(t)}\left(\boldsymbol{p}^{(t)}\right) 
    - 
    \sum_{t=1}^T \ell^{(t)}\left(\boldsymbol{p}^{\star}\right) \nonumber \\
    &\leq
    \sum_{t=1}^T \tilde\ell^{(t)}\left(\boldsymbol{p}^{(t)}\right) 
    - 
    \sum_{t=1}^T \tilde\ell^{(t)}\left(\boldsymbol{p}^{\star}\right) \nonumber \\
    &=
    \underbrace{
    R^{(T+1)}\left(\boldsymbol{p}^{\star}\right)
    -R^{(1)}\left(\boldsymbol{p}^{(1)}\right)}_\text{(i)}
    +
    \underbrace{
    {L}^{(T+1)}\left(\boldsymbol{p}^{(T+1)}\right)
    -{L}^{(T+1)}\left({\boldsymbol{p}^{\star}}\right)}_\text{(ii)} \\
    &+
    \underbrace{
    \sum_{t=1}^T
    \left[
    {L}^{(t)}\left(\boldsymbol{p}^{(t)}\right)
    -
    {L}^{(t+1)}\left(\boldsymbol{p}^{(t+1)}\right)
    +
    \tilde\ell^{(t)}\left(\boldsymbol{p}^{(t)}\right) \right]}_\text{(iii)} \nonumber,
\end{align}
}
    where the inequality is due to Corollary~\ref{corollary:lin_loss}.

    For (i), recall from Lemma~\ref{lemma:regret_bound_default_form} that the regret does not depend on the regularizer, 
    we can bound it after changing from $R^{(T+1)}\left(\boldsymbol{p}^{\star}\right)$ 
    to $R^{(T)}\left(\boldsymbol{p}^{\star}\right)$.
\begin{equation}
\begin{gathered}
    R^{(T)}\left(\boldsymbol{p}^{\star}\right)
    -
    R^{(1)}\left(\boldsymbol{p}\right)
    \leq
    \eta^{(T)}
    \sum_{i=1}^K
    {p_i^\star \log p_i^\star}
    +
    \eta^{(1)}\log{K}
    \leq
    \eta^{(T)}
    \sum_{i=1}^K
    {p_i^\star \log p_i^\star}
    +
    \eta^{(T)}\log{K}
    \leq
    \eta^{(T)}\log{K}.
\end{gathered}
\end{equation}

    For (ii), we use the assumption in Lemma~\ref{lemma:one_step_bound}, 
    where $\boldsymbol{p}^{(t)} = \argmin_{\boldsymbol{p}\in\Delta_{K-1}} L^{(t)}\left(\boldsymbol{p}\right)$.
    From the assumption, since $\boldsymbol{p}^{(T+1)}$ is the minimizer of $L^{(T+1)}$, (ii) becomes negative.
    Thus, we can exclude it from the upper bound.

    For (iii), we directly use the result of Lemma~\ref{lemma:one_step_bound} as follows.
\begin{equation}
\begin{gathered}
    \sum_{t=1}^T
    \left[
    {L}^{(t)}\left(\boldsymbol{p}^{(t)}\right)
    -
    {L}^{(t+1)}\left(\boldsymbol{p}^{(t+1)}\right)
    +
    \tilde\ell^{(t)}\left(\boldsymbol{p}^{(t)}\right) \right]
    \leq
    \sum_{t=1}^T
    \frac{ \left\Vert \boldsymbol{g}^{(t)} \right\Vert_\infty^2 }
    { 2\eta^{(t)} }
    \leq
    \sum_{t=1}^T
    \frac{ L_\infty^2 }
    { 2\eta^{(t)} },
\end{gathered}
\end{equation}
    where the additional terms are removed due to the non-decreasing property of regularizer thanks to the assumption of $\eta^{(t)}$, and the last inequality is due to Lemma~\ref{lemma:lipschitz}.
    Note that $\boldsymbol{g}^{(t)}=\nabla \ell^{(t)}(\boldsymbol{p}^{(t)})$.

    Combining all, we have regret upper bound as follows.
\begin{equation}
\begin{gathered}
    \normalfont\text{Regret}^{(T)}\left(\boldsymbol{p}^{\star}\right)
    = 
    \sum_{t=1}^T \ell^{(t)}\left(\boldsymbol{p}^{(t)}\right) 
    - 
    \sum_{t=1}^T \ell^{(t)}\left(\boldsymbol{p}^{\star}\right)
    \leq
    \sum_{t=1}^T \tilde\ell^{(t)}\left(\boldsymbol{p}^{(t)}\right) 
    - 
    \sum_{t=1}^T \tilde\ell^{(t)}\left(\boldsymbol{p}^{\star}\right)
    \leq
    \eta^{(T)}\log{K}
    +
    \sum_{t=1}^T
    \frac{ L_\infty^2 }
    { 2\eta^{(t)} }.
\end{gathered}
\end{equation}
    Finally, by setting $\eta^{(t)}=\frac{L_\infty\sqrt{t}}{\sqrt{\log{K}}}$, we have
\begin{equation}
\begin{gathered}
    \leq
    L_\infty\sqrt{T\log{K}}+\frac{L_\infty\sqrt{\log{K}}}{2}\sum_{t=1}^T\frac{1}{\sqrt{t}}
    \leq
    2L_\infty\sqrt{T\log{K}},
\end{gathered}
\end{equation}
    where the inequality is due to $\sum_{t=1}^{T} \frac{1}{\sqrt{t}} \leq 
    \int_{0}^{T} \frac{\mathrm{d}x}{\sqrt{x}} = 2\sqrt{T}$.
    See also equation (7.3) of \cite{oco3}.
\end{proof}

\subsection{Regret Bound of \texttt{AAggFF-D} with Partial Client Participation: Proof of Corollary~\ref{cor:crossdevice_partial}}
\begin{proof}
    Denote $\breve{\ell}^{(t)}$ as a linearized loss constructed from $\breve{\boldsymbol{r}}^{(t)}$ and $\breve{\boldsymbol{g}}^{(t)}$.
    i.e., 
\begin{equation}
\begin{gathered}
    \breve{\ell}^{(t)}(\boldsymbol{p}) = \left\langle \boldsymbol{p}, \breve{\boldsymbol{g}}^{(t)} \right\rangle
    = 
    \left\langle
    \boldsymbol{p},
    \frac{ \breve{\boldsymbol{r}}^{(t)} }
    {1 +\left\langle \boldsymbol{p}^{(t)}, \bar{\boldsymbol{r}}\boldsymbol{1}_K \right\rangle}
    +
    \frac{\bar{\boldsymbol{r}}\boldsymbol{1}_K {\boldsymbol{p}^{(t)}}^\top (\breve{\boldsymbol{r}}^{(t)} - \bar{\boldsymbol{r}}\boldsymbol{1}_K)}
    {(1 +\left\langle \boldsymbol{p}^{(t)}, \bar{\boldsymbol{r}}\boldsymbol{1}_K \right\rangle)^2}
    \right\rangle
\end{gathered}
\end{equation}
    The expected regret is
{\allowdisplaybreaks
\begin{align}
    &\mathbb{E}\left[ \normalfont\text{Regret}^{(T)}\left(\boldsymbol{p}^{\star}\right) \right]
    = 
    \mathbb{E}\left[ 
    \sum_{t=1}^T 
    \left( \ell^{(t)}\left(\boldsymbol{p}^{(t)}\right) 
    - 
    \ell^{(t)}\left(\boldsymbol{p}^{\star}\right) \right) 
    \right] \nonumber \\
    &\leq
    \mathbb{E}\left[ 
    \sum_{t=1}^T 
    \left( \breve\ell^{(t)}\left(\boldsymbol{p}^{(t)}\right) 
    - 
    \breve\ell^{(t)}\left(\boldsymbol{p}^{\star}\right) \right) 
    \right]
    =
    \mathbb{E}\left[ 
    \sum_{t=1}^T 
    \left\langle \breve{\boldsymbol{g}}^{(t)},
    \boldsymbol{p}^{(t)}
    - 
    \boldsymbol{p}^{\star} \right\rangle 
    \right] \nonumber \\
    &=
    \mathbb{E}\left[ 
    \sum_{t=1}^T 
    \mathbb{E}_{i\in S^{(t)}}\left[ 
    \left\langle \breve{\boldsymbol{g}}^{(t)},
    \boldsymbol{p}^{(t)}
    - 
    \boldsymbol{p}^{\star} \right\rangle 
    \right]
    \right] (\because \text{Law of Total Expectation}) \\
    &\approx
    \mathbb{E}\left[ 
    \sum_{t=1}^T  
    \left\langle {\boldsymbol{g}}^{(t)},
    \boldsymbol{p}^{(t)}
    - 
    \boldsymbol{p}^{\star} \right\rangle 
    \right] (\because \text{Lemma~\ref{lemma:linearized_grad}}) \nonumber \\
    &=
    \sum_{t=1}^T 
    \left( \tilde\ell^{(t)}\left(\boldsymbol{p}^{(t)}\right) 
    - 
    \tilde\ell^{(t)}\left(\boldsymbol{p}^{\star}\right) \right) \nonumber
    \leq
    \mathcal{O}\left(L_\infty K \sqrt{T\log{K}} \right)
\end{align}
}
\end{proof}
\vspace{2pt}

\begin{remark}
\label{remark:inflated_lip_const}
    Even though we can enjoy the same regret upper bound \textit{in expectation} from Corollary~\ref{cor:crossdevice_partial},
    it should be noted that the raw regret (i.e., regret without expectation) from $\breve{\boldsymbol{g}}^{(t)}$ may inflate the regret upper bound from
    $\mathcal{O}\left(L_\infty K \sqrt{T\log{K}} \right)$ to $\mathcal{O}\left(\breve{L}_\infty K \sqrt{T\log{K}} \right)$,
    where $\breve{L}_\infty$ is a Lipschitz constant from Lemma~\ref{lemma:lipschitz_lin_grad}, upper bounding $\left\Vert \breve{\boldsymbol{g}}^{(t)} \right\Vert_\infty \leq \breve{L}_\infty$.
    It is because $\breve{L}_\infty$ is dominated by ${1}/{C}\approx\mathcal{O}\left(K\right)$, 
    which can be a \textit{huge number} when if $C$ is a tiny constant.
    Although this inflation hinders proper update of \texttt{AAggFF-D} empirically,
    this can be easily eliminated in \texttt{AAggFF-D} through an appropriate choice of a range ($C_1$ and $C_2$) of the response vector,
    which ensures practicality of \texttt{AAggFF-D}. See Appendix~\ref{app:range} for a detail.
\end{remark}

\subsection{Derivation of Closed-From Update of \texttt{AAggFF-D}}
\label{app:deriv_closed}
The objective of \texttt{AAggFF-D} in (\ref{eq:lin_ftrl}) can be written in the following form.
\begin{equation}
\begin{split}
    &\min_{\boldsymbol{p}\in\Delta_{K-1}} \sum_{\tau=1}^{t} \tilde\ell^{(\tau)}(\boldsymbol{p}) + \eta^{(t+1)}\sum_{i=1}^K p_i \log p_i
    =
    \min_{\boldsymbol{p}\in\Delta_{K-1}} \left\langle \boldsymbol{p}, \sum_{\tau=1}^{t-1} \breve{\boldsymbol{g}}^{(\tau)} \right\rangle + \eta^{(t+1)}\sum_{i=1}^K p_i \log p_i \\
    &=
    \min_{\boldsymbol{p}\in\Delta_{K-1}} \left\langle \sum_{\tau=1}^{t} \breve{\boldsymbol{g}}^{(\tau)}, \boldsymbol{p} \right\rangle + R^{(t+1)}(\boldsymbol{p})
    =
    \max_{\boldsymbol{p}\in\Delta_{K-1}} \left\langle -\sum_{\tau=1}^{t} \breve{\boldsymbol{g}}^{(\tau)}, \boldsymbol{p} \right\rangle - R^{(t+1)}(\boldsymbol{p}).
\end{split}
\end{equation}
It exactly corresponds to the form of the Fenchel conjugate $R^{(t+1)}_*$, which is defined as follows.
\begin{equation}
\begin{gathered}
\label{eq:fenchel_conjugate}
    R^{(t+1)}_*(\boldsymbol{p})
    =
    \max_{\boldsymbol{p}\in\Delta_{K-1}} \left\langle -\sum_{\tau=1}^{t} \breve{\boldsymbol{g}}^{(\tau)}, \boldsymbol{p} \right\rangle - R^{(t+1)}(\boldsymbol{p}).
\end{gathered}
\end{equation}
Thus, we can enjoy the property of Fenchel conjugate, which is
\begin{equation}
\begin{gathered}
    \boldsymbol{p}^{(t+1)} = \nabla R^{(t+1)}_* \left(-\sum_{\tau=1}^{t} \breve{\boldsymbol{g}}^{(\tau)}\right)
\end{gathered}
\end{equation}
Since we can derive the log-sum-exp form by solving (\ref{eq:fenchel_conjugate}) as follows,
\begin{equation}
\begin{gathered}
    R^{(t+1)}_*(\boldsymbol{u}) = \log{\left( \sum_{i=1}^K \exp\left( u_i \right) \right)},
\end{gathered}
\end{equation}
we have the closed-form solution for the new decision update.
\begin{equation}
\begin{gathered}
    {p}^{(t+1)}_i = \frac
    {\exp \left( -\sum_{\tau=1}^{t} \breve{g}^{(\tau)}_i / \eta^{(t+1)} \right)}
    {\sum_{j=1}^K \exp \left( -\sum_{\tau=1}^{t} \breve{g}^{(\tau)}_j / \eta^{(t+1)} \right)}.
\end{gathered}
\end{equation}
Note that $\eta^{(t+1)}$ is already determined in Theorem~\ref{thm:crossdevice_full} and Corollary~\ref{cor:crossdevice_partial} as $\frac{\breve{L}_\infty \sqrt{t + 1}}{\sqrt{\log{K}}}$,
with the reflection of modified Lipschitz constant from $L_\infty$ to $\breve{L}_\infty$ (see Remark~\ref{remark:inflated_lip_const}).
See also \cite{eg} and Chapter 6.6 of \cite{oco3}.

\newpage
\section{Detailed Designs of \texttt{AAggFF}}
\label{app:design}

\subsection{Cumulative Distribution Function for Response Transformation}
\label{app:cdfs}
\paragraph{Choice of Distributions}
We used the CDF to transform unbounded responses from clients (i.e., local losses of clients) into bounded values.
Among diverse options, we used one of the following 6 CDFs in this work.
(Note that  $\operatorname{erf}(x)=\frac{2}{\sqrt{\pi}}\int_{0}^x e^{-y^2} \mathrm{d}y$ is the Gauss error function)

\begin{enumerate}
    \item Weibull \cite{weibull}: $\texttt{CDF}(x)=1 - e^{-(x/\alpha)^\beta}$; we set $\alpha=1$ (scale) and $\beta=2$ (shape).
    \item Frechet \cite{frechet}: $\texttt{CDF}(x)=e^{-(\frac{x}{\alpha})^{-\beta}}$; we set $\alpha=1$ (scale) and $\beta=1$ (shape).
    \item Gumbel \cite{gumbel}: $\texttt{CDF}(x)=e^{-e^{-(x - \alpha) / \beta}}$; we set $\alpha=1$ (scale) and $\beta=1$ (shape).
    \item Exponential: $\texttt{CDF}(x)=1 - e^{-\alpha x}$; we set $\alpha=1$ (scale).
    \item Logistic: $\texttt{CDF}(x)=\frac{1}{1+e^{(-(x-\alpha)/\beta)}}$; we set $\alpha=1$ (scale) and $\beta=1$ (shape).
    \item Normal \cite{gaussian}: $\texttt{CDF}(x)=\frac{1}{2}\left[1+\operatorname{erf}\left(\frac{x - \alpha}{\beta\sqrt{2}}\right)\right]$; we set $\alpha=1$ (scale) and $\beta=1$ (shape).
\end{enumerate}

Commonly, the scale parameter of all distributions is set to 1, since in (\ref{eq:resp_vec}) we centered inputs to 1 in expectation.
Although we fixed the parameters of each CDF, they can be statistically estimated in practice, such as using maximum spacing estimation \cite{mse}. 

For imposing larger mixing coefficients for larger losses, the transformation should (i) preserve the relative difference between responses, as well as (ii) not too sensitive for outliers.
While other heuristics (e.g., clipping values, subtracting from arbitrary large constant \cite{propfair}) for the transformation are also viable options for (i),
additional efforts are still required to address (ii).

On the contrary, CDFs can address both conditions with ease.
As CDFs are increasing functions, (i) can be easily satisfied.
For (ii), it can be intrinsically addressed by the nature of CDF itself.
Let us start with a simple example.

Suppose we have $K=3$ local losses: $F_1(\boldsymbol{\theta}) = 0.01, F_2(\boldsymbol{\theta}) = 0.10, F_3(\boldsymbol{\theta}) = 0.02$.
Since the average is $\bar{\mathrm{F}}=\frac{0.01+0.10+0.02}{3}\approx0.043$, we have inputs of CDF as follows: $F_1(\boldsymbol{\theta})/\bar{\mathrm{F}} = 0.23, F_2(\boldsymbol{\theta})/\bar{\mathrm{F}} = 2.31, F_3(\boldsymbol{\theta})/\bar{\mathrm{F}} = 0.46$.
These centered inputs are finally transformed into bounded values as in Table~\ref{tab:cdf_example}.

\begin{table}[h]
\footnotesize
\centering
\caption{Example: Effects of CDF Transformations}
\label{tab:cdf_example}
\begin{tabular}{@{}lc@{}}

\toprule                                                                                           & Transformed Responses \\ \midrule
\begin{tabular}[c]{@{}l@{}}Weibull CDF \\ $\texttt{CDF}(x)=1-e^{(-x^2)}$\end{tabular}              & 0.05 / \underline{1.00} / 0.19 \\
\begin{tabular}[c]{@{}l@{}}Frechet CDF \\ $\texttt{CDF}(x)=e^{(-1/x)}$\end{tabular}                & 0.01 / \underline{0.65} / 0.11 \\
\begin{tabular}[c]{@{}l@{}}Gumbel CDF \\ $\texttt{CDF}(x)=e^{\left(-e^{(-(x - 1))}\right)}$\end{tabular}      & 0.12 / \underline{0.76} / 0.18 \\
\begin{tabular}[c]{@{}l@{}}Exponential CDF \\ $\texttt{CDF}(x)=1 - e^{(-x)}$\end{tabular}          & 0.21 / \underline{0.90} / 0.37 \\
\begin{tabular}[c]{@{}l@{}}Logistic CDF \\ $\texttt{CDF}(x)=\frac{1}{1+e^{(-(x-1))}}$\end{tabular} & 0.32 / \underline{0.79} / 0.37 \\
\begin{tabular}[c]{@{}l@{}}Normal CDF \\ $\texttt{CDF}(x)=\frac{1}{2}\left[1+\operatorname{erf}\left(\frac{x - 1}{\sqrt{2}}\right)\right]$\end{tabular} &
  0.22 / \underline{0.90} / 0.29 \\ \bottomrule
\end{tabular}
\end{table}

While all losses become bounded values in $[0,1]$, the maximum local loss (i.e., $F_2(\boldsymbol{\theta}) = 0.10$) is transformed into different values by each CDF (see underlined figures in the `Transformed Responses' column of Table~\ref{tab:cdf_example}).
When using the Weibull CDF, the maximum local loss is translated into $1.00$, which means that there may be no value greater than 0.10 (i.e., $0.10$ is the largest one in $100\%$ probability) given current local losses.
Meanwhile, when using the Frechet CDF, the maximum local loss is translated into $0.65$, which means that there still is a 35\% chance that some other local losses are greater than $0.10$ when provided with other losses similar to 0.01 and 0.02.
This implies that each CDF \textit{treats a maximum value differently}.
When a transformation is easily inclined to the maximum value, thereby returning 1 (i.e., maximum of CDF), it may yield a degenerate decision, e.g., $\boldsymbol{p}\approx[0, 1, 0]^\top$. 

Fortunately, most of the listed CDFs are designed for \textit{modeling maximum values}.
For example, the three distributions, Gumbel, Frechet, and Weibull, are grouped as the Extreme Value Distribution (EVD) \cite{evd}.
As its name suggests, it models the behavior of extreme events, and it is well known that any density modeling a minimum or a maximum of independent and identically distributed (IID) samples follows the shape of one of these three distributions (by the Extreme Value Theorem \cite{evt}).
In other words, EVDs can reasonably measure \textit{how a certain value is close to a maximum}.
Thus, they can estimate whether a certain value is relatively large or small. 
Otherwise, the Exponential distribution is a special case of Weibull distribution, and the logistic distribution is also related to the Gumbel distribution.
Last but not least, although it is not a family of EVD, the Normal distribution is also considered due to the central limit theorem, since the local loss is the sum of errors from IID local samples.
We expect the CDF transformation can appropriately measure a relative magnitude of local losses,
and it should be helpful for decision making.

\paragraph{Effects of Response Transformation}
\label{subsec:exp_resp_transformation}
We also illustrated that the response should be bounded (i.e., Lipschitz continuous) in section \ref{sec:resp_trans}, to have non-vacuous regret upper bound.
To acquire bounded response, we compare the cumulative values of a global objective in (\ref{eq:fl_obj}), 
i.e., $\sum_{t=1}^T \sum_{i=1}^K p^{(t)}_i F_i (\boldsymbol{\theta}^{(t)})$ for the cross-silo setting, 
and $\sum_{t=1}^T \sum_{i\in S^{(t)}} p^{(t)}_i F_i (\boldsymbol{\theta}^{(t)})$ for the cross-device setting.

\begin{figure}[h]
    \centering
    \includegraphics[scale=0.6]{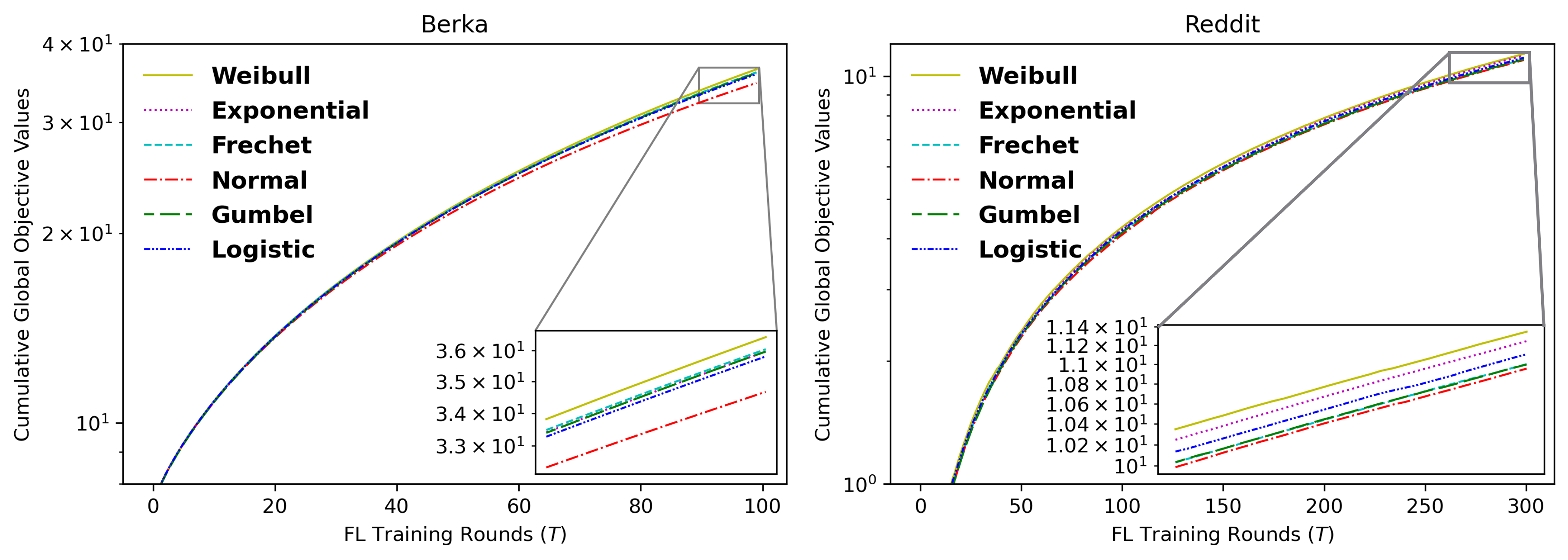}
    \caption{Cumulative values of a global objective according to different CDFs (smaller is better). 
    (Left) Berka dataset (cross-silo setting; $K=7, T=100$). 
    (Right) Reddit dataset (cross-device setting; $K=817, T=300, C=0.00612$)}
    \label{fig:enter-label}
\end{figure}

In the cross-silo experiment with the Berka dataset, the Normal CDF shows the smallest cumulative values, 
while in the cross-device experiment with the Reddit dataset, the Weibull CDF yields the smallest value.
For the Berka dataset, the Normal CDF yields an average performance (AUROC) of 79.37 with the worst performance of 43.75, 
but the Weibull CDF shows an average performance of 73.02 with the worst performance of 25.00.
The same tendency is also observed in the Reddit dataset. 
The Normal CDF presents an average performance (Acc. 1) of 14.05 and the worst performance of 4.26, 
while the Weibull CDF shows an average of 12.62 with the worst of 3.35.
From these observations, we can conclude that an appropriate choice of CDF is necessary for better sequential decision making, and suitable transformation helps minimize a global objective of FL. 
Note also that these behaviors are also directly related to the global convergence of the algorithm w.r.t. $\boldsymbol{\theta}$.

\subsection{Choice of a Response Range}
\label{app:range}

In regard to determining the range of a response vector, i.e., $[C_1, C_2]$, 
we can refer to the Lipschitz continuity in Lemma~\ref{lemma:lipschitz} and Lemma~\ref{lemma:lipschitz_lin_grad}.
For the cross-silo setting, we can set arbitrary constants so that $L_\infty=\mathcal{O}(1)$ according to Lemma~\ref{lemma:lipschitz}.
Thus, for all experiments of \texttt{AAggFF-S}, we set $C_1=0, C_2=\frac{1}{K}$.

For the cross-device setting, the Lipschitz constant is changed into $\breve{L}_\infty$, since it is influenced by the client sampling probability ($C$).
In detail, $C$ is located in the denominator of the Lipschitz constant, $\breve{L}_\infty$, which inflates the resulting gradient value as $C$ is a constant close to 0. 
(e.g., when 10 among 100,000 clients are participating in each round, ${1}/{C}=10^{4}$)
This is problematic and even causes an overflow problem empirically in updating a new decision.
Thus, we propose a simple remedy --- setting $C_1$ and $C_2$ to be \textit{a multiple of $C$}, 
so that the $C$ in the denominator is to be canceled out, according to Lemma~\ref{lemma:lipschitz_lin_grad}.
For instance, when $C_1=0, C_2=C$, resulting Lipschitz constant simply becomes $\Vert \breve{\boldsymbol{g}}^{(t)} \Vert_\infty \leq \breve{L}_\infty = \frac{C}{1+0} + \frac{2(C-0)}{C(1+0)}=C+2\approx2$,
which is a constant far smaller than $T$ and $K$, typically assumed in the practical cross-device FL setting.
Therefore, for all experiments of \texttt{AAggFF-D}, we set $C_1=0, C_2=C$.

\newpage
\section{Experimental Details}
\label{app:exp_details}
\begin{table}[ht]
\centering
\caption{Statistics of Federated Benchmarks}
\label{tab:dataset_stat}
\begin{tabular}{@{}ccccc@{}}
\toprule
\textbf{Dataset} & \textbf{Clients} & \textbf{Samples} & \textbf{Avg.} & \textbf{Std.} \\ \midrule
Berka          & 7     & 621     & 88.71   & 24.78   \\
MQP            & 11    & 3,048   & 277.09  & 63.25   \\
ISIC       & 6     & 21,310  & 3551.67 & 3976.16 \\
CelebA         & 9,343 & 200,288 & 21.44   & 7.63    \\
Reddit         & 817   & 97,961  & 119.90  & 229.85  \\
SpeechCommands & 2,005 & 84,700  & 42.24   & 36.69   \\ \bottomrule
\end{tabular}
\end{table}

\subsection{Datasets}
For the main experiments in Table~\ref{tab:result_silo} and Table~\ref{tab:result_device},
we used 3 datasets for the cross-silo setting (Berka \cite{berka}, MQP \cite{mqp}, and ISIC \cite{isic, flamby}), 
and 3 other datasets for the cross-device setting (CelebA, Reddit \cite{leaf}, and SpeechCommands \cite{speechcommands}).
In this section, we describe details of each dataset about its task, metrics, and client partitioning.
The statistics of all federated benchmarks are summarized in Table~\ref{tab:dataset_stat}.
Note that \textbf{Avg.} and \textbf{Std.} in the table refer to the average and a standard deviation of a sample size of each client in the federated system.

First, we present details of the federated benchmark for the cross-silo setting.
\begin{enumerate}[label=$\bullet$]
    \item \textbf{Berka} is a tabular dataset containing bank transaction records collected from Czech bank \cite{berka}.
    Berka accompanies the loan default prediction task (i.e., binary classification) of each bank's customers.
    It is fully anonymized and is originally composed of 8 relational tables: accounts, clients, disposition, loans, permanent orders, transactions, demographics, and credit cards.
    We merged all 8 tables into one dataset by joining the primary keys of each table, and finally have 15 input features.
    From the demographics table, we obtain information on the region: Prague, Central Bohemia, South Bohemia, West Bohemia, North Bohemia, East Bohemia, South Moravia, and North Moravia.
    We split each client according to the region and excluded all samples of North Bohemia since it has only one record of loan default, thus we finally have 7 clients (i.e., banks).
    Finally, we used the area under the receiver operating characteristic (ROC) curve for the evaluation metric.
    
    \item \textbf{MQP} is a clinical question pair dataset crawled from medical question answering dataset \cite{healthtap}, 
    and labeled by 11 doctors \cite{mqp}.
    All paired sentences are labeled as either similar or dissimilar, thereby it is suitable for the binary classification task.
    As a pre-processing, we merge two paired sentences into one sentence by adding special tokens: \texttt{[SEP]}, \texttt{[PAD]}, and \texttt{[UNK]}.
    We set the maximum token length to 200, thus merged sentences less than 200 are filled with \texttt{[PAD]} tokens, and otherwise are truncated. Then, merged sentences are tokenized using pre-trained DistilBERT tokenizer \cite{distilbert}.
    We regard each doctor as a separate client and thus have 11 clients.
    Finally, we used the area under the ROC curve for the evaluation metric.
    
    \item \textbf{ISIC} is a dermoscopic image dataset for a skin cancer classification, collected from 4 hospitals. \cite{isic, flamby}
    The task contains 8 distinct melanoma classes, thus designed for the multi-class classification task.
    Following \cite{flamby}, as one hospital has three different imaging apparatus, its samples are further divided into 3 clients, thus we finally have 6 clients in total.   
    Finally, we used top-5 accuracy for the evaluation metric.
\end{enumerate}

Next, we illustrate details of the federated benchmark for the cross-device setting.
\begin{enumerate}[label=$\bullet$]
    \item \textbf{CelebA} is a vision dataset containing the facial images of celebrities \cite{celeba}.
    It is curated for federated setting in LEAF benchmark \cite{leaf}, and is targeted for the binary classification task (presence of smile).
    We follow the processing of \cite{leaf}, thereby each client corresponds to each celebrity, having 9,343 total clients in the federated system.
    Finally, we used top-1 accuracy for the evaluation metric following \cite{leaf}.
    
    \item \textbf{Reddit} is a text dataset containing the comments of community users of Reddit in December 2017, and a part of LEAF benchmark \cite{leaf}. 
    Following \cite{leaf}, we build a dictionary of vocabularies of size 10,000 from tokenized sentences and set the maximum sequence length to 10.
    The main task is tailored for language modeling, i.e., next token prediction, given word embeddings in each sentence of clients.
    Each client corresponds to one of the community users, thus 817 clients are presented in total.
    Finally, we used the top-1 accuracy for the evaluation metric following \cite{leaf}.
    
    \item \textbf{SpeechCommands} is designed for a short-length speech recognition task that includes one second 35 short-length words,
    such as \enquote{Up}, \enquote{Down}, \enquote{Left}, and \enquote{Right} \cite{speechcommands}. 
    It is accordingly a multi-class classification task for 35 different classes.
    While it is collected from 2,618 speakers, we rule out all samples of speakers having too few samples when splitting the dataset into training and test sets.
    (e.g., exclude speakers whose total sample counts are less than 3)
    As a result, we have 2,005 clients, and each client has 16,000-length time-domain waveform samples.
    Finally, we used top-5 accuracy for the evaluation metric.
\end{enumerate}

\subsection{Models}
For each dataset, we used task-specific model architectures which are already used in previous works, or widely used in reality, to simulate the practical FL scenario as much as possible.
For the experiment of the cross-silo setting, we used the following models.
\begin{enumerate}[label=$\bullet$]
    \item \textit{Logistic Regression} is used for the Berka dataset. 
    We used a simple logistic regression model with a bias term, and the output (i.e., logit vector) is transformed into predicted class probabilities by the softmax function.

    \item \textit{DistilBERT \cite{distilbert}} is used for the MQP dataset.
    We used a pre-trained DistilBERT model, from BookCorpus and English Wikipedia \cite{distilbert}. 
    We also used the corresponding DistilBERT tokenizer for the pre-processing of raw clinical sentences.
    For a fine-tuning of the pre-trained DistilBERT model, we attach a classifier head next to the last layer of the DistilBERT's encoder, which outputs an embedding of $768$ dimension.
    The classifier is in detail processing the embedding as follows: ($768$-ReLU-Dropout-$2$),
    where each figure is an output dimension of a fully connected layer with a bias term, 
    ReLU is a rectified linear unit activation layer, 
    and Dropout \cite{dropout} is a dropout layer having probability of $0.1$.
    In the experiment, we trained all layers including pre-trained weights.

    \item \textit{EfficientNet \cite{efficientnet}} is used for the ISIC dataset.
    We also used the pre-trained EfficientNet-B0 model from ImageNet benchmark dataset \cite{imagenet}.
    For fine-tuning, we attach a classifier head after the convolution layers of EfficientNet.
    The classifier is composed of the following components: (AdaptiveAvgPool($7, 7$)-Dropout-$8$),
    where AdaptiveAvgPool($cdot, cdot$) is a 2D adaptive average pooling layer outputs a feature map of size $7 \times 7$ (which are flattened thereafter), 
    Dropout is a dropout layer with a probability of 0.1, 
    and the last linear layer outputs an 8-dimensional vector,
    which is the total number of classes.
\end{enumerate}

Next, for the cross-device setting, we used the following models.
\begin{enumerate}[label=$\bullet$]
    \item \textit{ConvNet} model used in LEAF benchmark \cite{leaf} is used for the CelebA dataset.
    It is composed of four convolution layers, of which components are:
    2D convolution layer without bias term with 32 filters of size $3\times 3$ (stride=1, padding=1),
    group normalization layer (the number of groups is decreased from 32 by a factor of 2: 32, 16, 8, 4),
    2D max pooling layer with $2 \times 2$ filters,
    and a ReLU nonlinear activation layer.
    Plus, a classifier comes after the consecutive convolution layers, which are composed of:
    (AdaptiveAvgPool($5, 5$)-$1$),
    which are a 2D adaptive average pooling layer that outputs a feature map of size $5 \times 5$ (which are flattened thereafter), 
    and a linear layer with a bias term outputs a scalar value since it is a binary classification task.

    \item \textit{StackedLSTM} model used in LEAF benchmark \cite{leaf} is used for the Reddit dataset.
    It is composed of an embedding layer of which the number of embeddings is 200, and outputs an embedding vector of 256 dimensions.
    It is processed by consecutive 2 LSTM \cite{lstm} layers with the hidden size of 256, and enters the last linear layer with a bias term, which outputs a logit vector of 10,000 dimensions, which corresponds to the vocabulary size.

    \item \textit{M5 \cite{m5}} model is used for the SpeechCommands dataset.
    It is composed of four 1D convolution layers followed by a 1D batch normalization layer, ReLU nonlinear activation, and a 1D max pooling layer with a filter of size 4.
    All convolution layers EXCEPT the input layer have a filter of size $3$, and the numbers of filters are 64, 128, and 256 (all with stride=1 and padding=1).
    The input convolution layer has 64 filters with a filter of size $80$, and stride of $4$.
    Lastly, one linear layer outputs a logit vector of 35 dimensions. 
\end{enumerate}

\subsection{Hyperparameters}
Before the main experiment, we first tuned the hyperparameter of all baseline fair FL algorithms from a separate random seed.
The hyperparameter of each fair algorithm is listed as follows.
\begin{enumerate}[label=$\bullet$]
    \item \texttt{AFL} \cite{afl} --- a step size of a mixing coefficient $\in$\{0.01, 0.1, 1.0\}
    \item \texttt{q-FedAvg} \cite{qffl} --- a magnitude of fairness, $\in$\{0.1, 1.0, 5.0\}
    \item \texttt{TERM} \cite{term} --- a tilting constant, $\lambda$, $\in$\{0.1, 1.0, 10.0\}
    \item \texttt{FedMGDA} \cite{fedmgda} --- a deviation from static mixing coefficient $\in$\{0.1, 0.5, 1.0\}
    \item \texttt{PropFair} \cite{propfair} --- a baseline constant $\in$\{2, 3, 5\}
\end{enumerate}
Each candidate value is selected according to the original paper, and we fix the number of local epochs, $E=1$  (following the set up in \cite{qffl}), along with the number of local batch size $B=20$ in all experiments.
For each dataset, a weight decay (L2 penalty) factor ($\psi$), a local learning rate ($\zeta$), and variables related to a learning rate scheduling (i.e., learning rate decay factor ($\phi$), and a decay step ($s$)) are tuned first with \texttt{FedAvg} \cite{fedavg} as follows.
\begin{enumerate}[label=$\bullet$]
    \item \textbf{Berka}: $\psi=10^{-3}, \zeta=10^0, \phi=0.99, s=10$
    \item \textbf{MQP}: $\psi=10^{-2}, \zeta=10^{-\frac{5}{2}}, \phi=0.99, s=15$
    \item \textbf{ISIC}: $\psi=10^{-2}, \zeta=10^{-4}, \phi=0.95, s=5$
    \item \textbf{CelebA}: $\psi=10^{-4}, \zeta=10^{-1}, \phi=0.96, s=300$
    \item \textbf{Reddit}: $\psi=10^{-6}, \zeta=10^{\frac{7}{8}}, \phi=0.95, s=20$
    \item \textbf{SpeechCommands}: $\psi=0, \zeta=10^{-1}, \phi=0.999, s=10$
\end{enumerate}
This is intended under the assumption that all fair FL algorithms should at least be effective in the same setting of the FL algorithm with the static aggregation scheme (i.e., \texttt{FedAvg}).
Note that client-side optimization in all experiments is done by the Stochastic Gradient Descent (SGD) optimizer.

\subsection{Implementation Details}
All code is implemented in PyTorch \cite{pytorch}, simulating a central parameter server that orchestrates a whole FL procedure and operates \texttt{AAggFF}. 
We further simulate $K$ participating clients having their own local samples, and a communication scheme with the central server. 
All experiments are conducted on a server with 2 Intel\textsuperscript \textregistered 
Xeon\textsuperscript \textregistered 
Gold 6226R CPUs (@ 2.90GHz) and 2 NVIDIA\textsuperscript \textregistered  Tesla\textsuperscript \textregistered V100-PCIE-32GB GPUs.

\newpage
\section{Pseudocode for \texttt{AAggFF}}
\subsection{Pseudocode for \texttt{ClientUpdate}}
\begin{algorithm}[h]
   \caption{\texttt{ClientUpdate}}
   \label{alg:clientupate}
\begin{algorithmic}
   \STATE {\bfseries Input:} number of local epochs $E$, local batch size $B$, local learning rate $\zeta$, global model $\boldsymbol{\theta}$
   \STATE {\bfseries Procedure:}
   \STATE Evaluate the received global model on training set according to eq.~\eqref{eq:local_loss} to yield $F_i\left(\boldsymbol{\theta}\right)$.
   \STATE Set local model $\boldsymbol{\theta}^{(0)}\leftarrow\boldsymbol{\theta}$
   \FOR{$e=0$ {\bfseries to} $E-1$}
    \STATE $\mathcal{B}_e$ $\leftarrow$ Split the client training dataset into batches of size $B$.
    \FOR{mini-batch $\Xi$ in $\mathcal{B}_e$}
     \STATE Update the model $\boldsymbol{\theta}^{(e)}\leftarrow\boldsymbol{\theta}^{(e)} - \frac{\zeta}{B} \sum\limits_{j=1}^B \nabla_{\boldsymbol{\theta}} \mathcal{L}\left(\Xi;\boldsymbol{\theta}^{(e)}\right)$.
    \ENDFOR
    \STATE Set $\boldsymbol{\theta}^{(e+1)}\leftarrow\boldsymbol{\theta}^{(e)}$
   \ENDFOR
   \STATE{\bfseries Return:} $F_i\left(\boldsymbol{\theta}\right)$, $\boldsymbol{\theta} - \boldsymbol{\theta}^{(E)}$.
\end{algorithmic}
\end{algorithm}
For generating a response vector, each client is requested to evaluate the received global model on its \textit{training samples}, 
$\{\xi_k\}_{k=1}^{n_i}$, before the local update.
As a result of the evaluation, the local loss of client $i$ at round $t$, i.e., $F_i\left(\boldsymbol{\theta}^{(t)}\right)$, is calculated as follows.
\begin{equation}
\begin{aligned}
\label{eq:local_loss}
    F_i\left(\boldsymbol{\theta}^{(t)}\right) = \frac{1}{n_i} \sum_{k=1}^{n_i} \mathcal{L}\left(\xi_k;\boldsymbol{\theta}^{(t)}\right),
\end{aligned}
\end{equation}
where $n_i$ is the total sample size of client $i$, $l$ is a loss function specific to the task, and $h_{\boldsymbol{\theta}}$ is a hypothesis realized by the parameter $\boldsymbol{\theta}$.

\subsection{Pseudocode for \texttt{AAggFF-S}}
\begin{algorithm}[h]
   \caption{\texttt{AAggFF-S}}
   \label{alg:aaggff-s}
\begin{algorithmic}
   \STATE {\bfseries Input:} number of clients $K$, total rounds $T$, transformation $\rho$, minimum and maximum of a response $[C_1, C_2]$.
   \STATE {\bfseries Initialize:} mixing coefficients $\boldsymbol{p}^{(1)}=\frac{1}{K}\boldsymbol{1}_K$, global model $\boldsymbol{\theta}^{(1)}\in\mathbb{R}^d$
   \STATE {\bfseries Procedure:} \FOR{$t=0$ {\bfseries to} $T-1$}
    \FOR{each client $i=1,...,K$ {\bfseries in parallel}} 
        \STATE $F_i\left(\boldsymbol{\theta}^{(t)}\right), \boldsymbol{\theta}^{(t)}-\boldsymbol{\theta}^{(t+1)}_i \leftarrow \texttt{ClientUpdate}\left(\boldsymbol{\theta}^{(t)}\right)$
    \ENDFOR
    \STATE Return $\boldsymbol{r}^{(t)}$ according to eq.~\eqref{eq:resp_vec} and $C_1, C_2$.
    \STATE Suffer decision loss $\ell^{(t)}(\boldsymbol{p}^{(t)})$ according to eq.~\eqref{eq:decision_loss}.
    \STATE Return a gradient $\boldsymbol{g}^{(t)}=\nabla\ell^{(t)}\left(\boldsymbol{p}^{(t)}\right)$.
    \STATE Return a mixing coefficient $\boldsymbol{p}^{(t+1)}$ according to eq.~\eqref{eq:ons}.
    \STATE Update a global model $\boldsymbol{\theta}^{(t+1)} = \boldsymbol{\theta}^{(t)}-\sum\limits_{i=1}^K p^{(t+1)}_i \left(\boldsymbol{\theta}^{(t)}-\boldsymbol{\theta}^{(t+1)}_i\right)$.
   \ENDFOR
   \STATE{\bfseries Return:} $\boldsymbol{\theta}^{(T)}$
\end{algorithmic}
\end{algorithm}

\subsection{Pseudocode for \texttt{AAggFF-D}}
\begin{algorithm}[ht]
   \caption{\texttt{AAggFF-D}}
   \label{alg:aaggff-d}
\begin{algorithmic}
   \STATE {\bfseries Input:} number of clients $K$, client sampling ratio $C\in(0,1)$, total rounds $T$, transformation $\rho$, range of a response $[C_1, C_2]$.
   \STATE {\bfseries Initialize:} mixing coefficients $\boldsymbol{p}^{(0)}=\frac{1}{K}\boldsymbol{1}_K$, global model $\boldsymbol{\theta}^{(0)}\in\mathbb{R}^d$
   \STATE {\bfseries Procedure:} \FOR{$t=0$ {\bfseries to} $T-1$}
    \STATE $S^{(t)} \leftarrow \text{Wait until } \min(1, \lfloor C \cdot K \rfloor) \text{ clients are active in a network}$.
    \FOR{each client $i \in S^{(t)}$ {\bfseries in parallel}} 
        \STATE $F_i\left(\boldsymbol{\theta}^{(t)}\right), \boldsymbol{\theta}^{(t)}-\boldsymbol{\theta}^{(t+1)}_i \leftarrow \texttt{ClientUpdate}\left(\boldsymbol{\theta}^{(t)}\right)$
    \ENDFOR
    \STATE Return $\breve{\boldsymbol{r}}^{(t)}$ according to eq.~\eqref{eq:resp_vec}, eq.~\eqref{eq:dr_response}, and  $C_1, C_2$.
    \STATE Suffer decision loss $\ell^{(t)}(\boldsymbol{p}^{(t)})$ according to eq.~\eqref{eq:decision_loss}.
    \STATE Return a gradient estimate $\breve{\boldsymbol{g}}^{(t)}$ according to eq.~\eqref{eq:linearized_gradient}.
    \STATE Return mixing coefficients $\boldsymbol{p}^{(t+1)}$ according to eq.~\eqref{eq:closed_form}).
    \STATE Acquire selected coefficients $\tilde{\boldsymbol{p}}^{(t+1)}$ according to eq.~\eqref{eq:normalized_mix_coefs}. 
    \STATE Update a global model $\boldsymbol{\theta}^{(t+1)} = \boldsymbol{\theta}^{(t)}-\sum\limits_{i \in S^{(t)}} \tilde{p}^{(t+1)}_i \left(\boldsymbol{\theta}^{(t)}-\boldsymbol{\theta}^{(t+1)}_i\right)$.
   \ENDFOR
   \STATE{\bfseries Return:} $\boldsymbol{\theta}^{(T)}$
\end{algorithmic}
\end{algorithm}

After updating whole entries of a decision variable, the server only exploits mixing coefficients of which indices correspond to selected clients.
Denoting as $\tilde{\boldsymbol{p}}^{(t+1)}\in\Delta_{\vert S^{(t)} \vert - 1}$, 
each selected entry is normalized as follows.
\begin{equation}
\begin{aligned}
\label{eq:normalized_mix_coefs}
    \tilde{p}_i^{(t+1)} = \frac{p_i^{(t+1)}}{ \sum_{j \in S^{(t)}} p_j^{(t+1)} }, i \in S^{(t)}
\end{aligned}
\end{equation}
This ensures that selected coefficients also satisfy the condition for being a probability vector (i.e., sum up to 1).

\end{document}